\documentclass[10pt]{article}

\usepackage{fullpage}
\usepackage{mathpazo}
\usepackage[scaled=.90]{helvet}
\usepackage{inconsolata}
\usepackage[T1]{fontenc}
\usepackage[utf8]{inputenc}
\usepackage[dvipsnames]{xcolor}
\usepackage[utf8]{inputenc}   %
\usepackage[T1]{fontenc}      %

\usepackage{titletoc}

\usepackage{amsmath}
\usepackage{amssymb}
\usepackage{amsthm}
\usepackage{mathtools}
\usepackage{mathrsfs}

\usepackage{graphicx}
\usepackage{subcaption}
\usepackage{float}            %

\usepackage{nicefrac}
\usepackage{microtype}        %
\usepackage{xcolor}
\usepackage{mdframed}

\usepackage{enumitem}
\usepackage[colorinlistoftodos]{todonotes}
\usepackage{algorithm}
\usepackage{algorithmic}

\usepackage{booktabs}
\usepackage{adjustbox}
\usepackage{multirow}
\usepackage{natbib}
\usepackage{thmtools,thm-restate}

\definecolor{mydarkblue}{rgb}{0,0.08,0.45}
\usepackage[colorlinks=true,
	linkcolor=mydarkblue,
	citecolor=mydarkblue,
	filecolor=mydarkblue,
	urlcolor=mydarkblue,
	pagebackref=true]{hyperref}
\usepackage[capitalize, noabbrev]{cleveref}
\renewcommand*\backref[1]{\ifx#1\relax \else (cited on #1) \fi}

\declaretheorem[name=Theorem]{theorem}
\declaretheorem[name=Proposition]{proposition}
\declaretheorem[name=Assumption]{assumption}
\declaretheorem[name=Corollary]{corollary}

\declaretheorem[name=Lemma]{lemma}
\declaretheorem[name=Problem]{problem}

\crefname{theorem}{Theorem}{Theorems}
\crefname{proposition}{Proposition}{Propositions}
\crefname{assumption}{Assumption}{Assumptions}
\crefname{corollary}{Corollary}{Corollaries}
\crefname{remark}{Remark}{Remarks}
\crefname{definition}{Definition}{Definitions}
\crefname{lemma}{Lemma}{Lemmas}
\crefname{problem}{Problem}{Problems}

\crefname{ALC@unique}{line}{lines}
\Crefname{ALC@unique}{Line}{Lines}

\newcommand{\red}[1]{\textcolor{red}{#1}}
\newcommand{\green}[1]{\textcolor{olive}{#1}}

\def\1{\bm{1}}

\newcommand{\pistarl}{\pi^*_{\lambda, \tau}}
\newcommand{\pistarlmu}{\pi^*_{\lambda, \mu, \tau}}
\newcommand{\pistarbl}{\pi^*_{\bar{\lambda}, \tau}}
\newcommand{\pistarblT}{\pi^*_{\bar{\lambda}_T, \tau}}
\newcommand{\pistarlp}{\pi^*_{\lambda', \tau}}
\newcommand{\pistargtau}{\pi^*_{\lambda,\tau}}

\newcommand{\gammat}{R_t}
\newcommand{\gammatt}{R_{t+1}}
\newcommand{\pistart}{\pi^*_t}
\newcommand{\pistartt}{\pi^*_{t+1}}
\newcommand{\lag}{\cL_{\tau}}
\newcommand{\dg}{\nabla \cD(\lambda)}
\newcommand{\dgt}{\nabla \cD(\lambda_t)}
\newcommand{\bdg}{\widetilde{\nabla} \cD(\lambda)}
\newcommand{\bdgt}{\widetilde{\nabla} \cD(\lambda_t)}
\newcommand{\bFgt}{\widetilde{\nabla} \F(\lambda_t)}

\newcommand{\lt}{\lambda_t}
\newcommand{\ltt}{\lambda_{t+1}}

\newcommand{\betat}{\beta_t}

\newcommand{\blambda}{\bar{\lambda}}

\newcommand{\projL}{\text{Proj}_{\Lambda}}

\newcommand{\opt}{\mathrm{opt}}
\newcommand{\bias}{\mathrm{bias}}

\newcommand\numberthis{\addtocounter{equation}{1}\tag{\theequation}}
\newcommand{\saom}{\mu^{\pi}}

\newcommand{\KL}{\mathrm{KL}}

\def\pistar{\pi^*}

\DeclarePairedDelimiter\parens{\lparen}{\rparen}
\DeclarePairedDelimiter\abs{\lvert}{\rvert}
\DeclarePairedDelimiter\norm{\lVert}{\rVert}
\DeclarePairedDelimiter\braces{\lbrace}{\rbrace} %
\DeclarePairedDelimiter\bracks{\lbrack}{\rbrack} %
\DeclarePairedDelimiter\angles{\langle}{\rangle}

\def\dpd#1{{\angles*{ #1 }}}
\def\eps{{\epsilon}}

\def\supnorm#1{\norm*{ #1 }_\infty}
\def\normsq#1{{\norm*{ #1 }^2_2}}

\DeclareMathAlphabet{\mathsfit}{\encodingdefault}{\sfdefault}{m}{sl}
\SetMathAlphabet{\mathsfit}{bold}{\encodingdefault}{\sfdefault}{bx}{n}

\def\gC{{\mathcal{C}}}

\def\gE{{\mathcal{E}}}

\def\gH{{\mathcal{H}}}

\def\gK{{\mathcal{K}}}

\def\gO{{\mathcal{O}}}

\def\gS{{\mathcal{S}}}
\def\gT{{\mathcal{T}}}

\newcommand{\E}{\mathbb{E}}

\newcommand{\R}{\mathbb{R}}

\DeclareMathOperator*{\argmax}{arg\,max}
\DeclareMathOperator*{\argmin}{arg\,min}

\def\cA{{\mathcal{A}}}

\def\cS{{\mathcal{S}}}

\def\cT{{\mathcal{T}}}

\def\cD{{\mathcal{D}_{\tau}}}
\def\cP{{\mathcal{P}}}

\newcommand{\x}{x}

\newcommand{\cC}{\mathcal{C}}
\newcommand{\cL}{\mathcal{L}}

\newcommand{\lstar}{\lambda^*_\tau}
\newcommand{\lstarmu}{\lambda^*_{\tau, \mu}}

\newcommand{\F}{\mathcal{D}_{\tau, \mu}}

\begin{document}
\title{Last-Iterate Convergence for Constrained MDPs: \\ From Dual Descent to Primal-Dual}
\author{
    \begin{tabular}{cc}
        Michael Lu & Max Qiushi Lin \\
        Simon Fraser University & Simon Fraser University \\
        \texttt{michael\_lu\_3@sfu.ca} & \texttt{maxqslin@gmail.com} \\[1em]
        Mo Chen & Sharan Vaswani \\
        Canada CIFAR AI Chair, Amii, 
        & Simon Fraser University \\ Simon Fraser University & \texttt{vaswani.sharan@gmail.com} \\
        \texttt{mochen@cs.sfu.ca} & 
    \end{tabular}
}

\date{}
\maketitle

\begin{abstract}
We study policy optimization for infinite-horizon, discounted constrained Markov decision processes. While primal-dual ($\texttt{PD}$) methods are widely used in practice, their theoretical guarantees typically hold for an impractical mixture policy. In contrast, dual descent (\texttt{DD}) methods output a single policy  (``last iterate'') with guarantees on its performance. However, such methods require iteratively solving entropy-regularized MDPs in an inner-loop, limiting their practical effectiveness. Moreover, existing analyses rely on strong coverage assumptions that are difficult to satisfy in practice. Consequently, we first provide a new analysis of $\texttt{DD}$ methods that does not require coverage assumptions. We then employ entropy-regularized natural policy gradient and dual regularization to eliminate the inner loop, resulting in an efficient \texttt{PD} algorithm. We develop a novel tracking analysis which shows that, compared to existing methods, the proposed algorithm achieves improved last-iterate convergence guarantees. Furthermore, our theoretical framework allows us to analyze a truncated bilevel method that interpolates between $\texttt{DD}$ and $\texttt{PD}$ methods. Finally, we generalize our analysis to handle linear function approximation. Compared to prior last-iterate methods, we prove stronger guarantees under less restrictive assumptions.
\end{abstract}

\section{Introduction}
Reinforcement learning (RL) has found success across a wide range of applications, including video games~\citep{mnih2015human}, robotics~\citep{kober2013reinforcement}, and the alignment of large language models~\citep{uc2023survey}.
Standard RL algorithms are typically designed to maximize a single objective.  However, in many real-world applications, operational or safety constraints must also be considered.
For example, in wireless communication systems, a network controller seeks to maximize overall throughput while ensuring that the average connectivity exceeds a target rate~\citep{julian2002qos,buratti2009overview}.
Such problems can be modeled as a constrained Markov decision process (CMDP)~\citep{altman2021constrained}.
In this setting, the goal of an agent is to maximize the cumulative reward (e.g. wireless throughput) while ensuring that the cumulative constraint reward (e.g. connectivity rate) is above a prescribed threshold.
However, since an optimal policy of a CMDP is, in general, stochastic~\citep{altman2021constrained}, designing principled algorithms that efficiently output a stochastic policy remains a central challenge.

A common approach to solving CMDPs are primal-dual (\texttt{PD}) methods~\citep{jain2022towards,liu2025sample,ding2020natural,bura2022dope}. 
These methods construct the Lagrangian of the CMDP and then iteratively optimize the policy and the associated dual variable (Lagrange multipliers). 
Their ease of implementation makes them appealing to practitioners. 
However the final policy (normally used in practice) exhibits an undesirable oscillating behaviour (see~\cref{fig:bandit}).
From a theoretical perspective, convergence guarantees of $\texttt{PD}$ methods are typically limited to those for impractical \emph{mixture policies}.
Constructing such a mixture policy requires storing all intermediate policies, a memory intensive approach. 
Furthermore, deploying a mixture policy requires randomly selecting one of these policies.
As a result, the constraint guarantees hold only in expectation over this random selection, rather than for the individual policy that is ultimately deployed.
\begin{figure}[!h]
    \centering
    \includegraphics[scale=1.2]{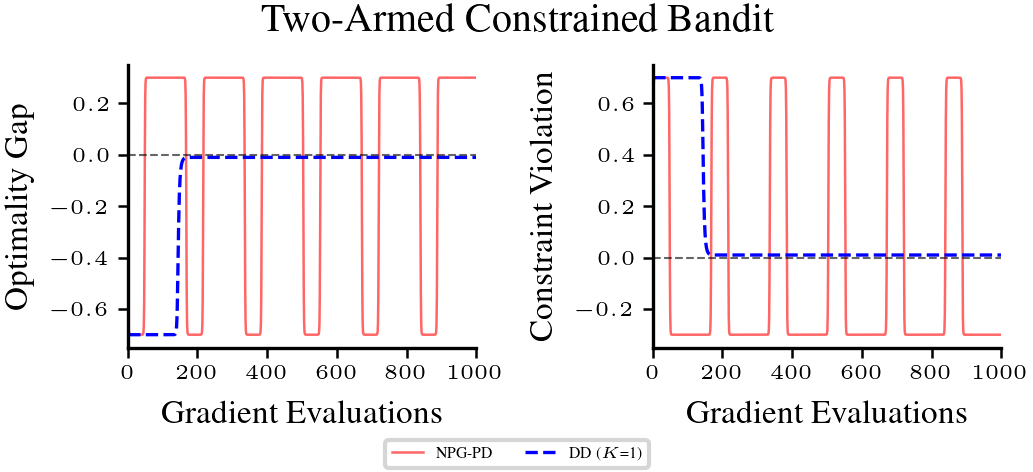}
    \caption{Comparing a \texttt{PD} approach (\texttt{NPG-PD})~\citep{ding2020natural} to the proposed \texttt{DD} method (\texttt{DD(K=1)}) on a 2-armed constrained bandit problem. While the iterates of \texttt{NPG-PD} oscillate, \texttt{DD(K=1)} converges to zero optimality gap and constraint violation. See \cref{appendix:experiments_bandit} for details.}
    \label{fig:bandit}
\end{figure}

\begin{table*}[h!]
\centering
\small
\begin{adjustbox}{width=1.0\textwidth,center}
\begin{tabular}{lcccc}
\toprule
Method &
Number of Constraints &
Rate &
Coverage Assumption &
Other Assumptions \& Caveats \\
\midrule
\multicolumn{5}{l}{\textit{Tabular Policy Parameterization}} \\
\midrule
\citet[Corollary 5.4]{ying2022dual}  & Multiple & $\tilde{\gO}(\nicefrac{1}{\epsilon^2})$ & $\green{\checkmark}$ & Bilevel method \\
\citet[Corollary 6.2]{ying2022dual} & Single & $\gO(\ln(\nicefrac{1}{\epsilon}))$ & $\green{\checkmark}$ & Bilevel method \\
\citet[Corollary 3]{ding2023last} & Single & $\tilde{\gO}(\nicefrac{1}{\epsilon^6})$ & $\green{\checkmark}$ & Requires  projections onto a restricted simplex \\
\citet[Corollary 7]{ding2023last}  & Single & $\gO( \ln(\nicefrac{1}{\epsilon}))$ & $\green{\checkmark}$ & Rate depends on potentially unbounded constants \\
\citet[Corollary 4.5]{gladin2023algorithm} & Multiple & $\gO(\ln(\nicefrac{1}{\epsilon}))$ & $\green{\checkmark}$ & Bilevel method, rate holds only when $\eps$ is very small \\
\citet[Corollary 1]{lu2026augmented}  & Multiple & $\gO(\nicefrac{1}{\epsilon^6})$ & $\green{\checkmark}$ & Bilevel method \\
\citet[Theorem 3.2]{montenegro2024last} & Multiple & $\gO(\nicefrac{1}{\epsilon^2})$ & $\green{\checkmark}$ & Rate depends on potentially unbounded constants \\
\citet[Theorem 3.7]{moskovitz2023reload}   & Multiple & -- & $\red{\times}$ & Rate unspecified \\
\citet[Theorem 4.1]{muller2024truly}  & Multiple & $\gO(\nicefrac{\text{poly}(A, H, n)}{\epsilon^{10}})$ & $\red{\times}$ & Limited to the finite horizon setting \\
\textbf{\cref{theorem:reg_smoothness_l2_rate}} (\textbf{ours}) & Multiple & $\tilde{\gO}(\nicefrac{1}{\epsilon^2})$ & $\red{\times}$ & Bilevel method \\
\textbf{\cref{theorem:l2_last_iterate_K_1}} (\textbf{ours}) & Multiple & $\tilde{\gO}(\nicefrac{1}{\epsilon^5})$ & $\red{\times}$ & -- \\
\textbf{\cref{theorem:l2_last_iterate_K_geq_1}} (\textbf{ours}) & Multiple & $\tilde{\gO}(\nicefrac{1}{\epsilon^{3.5}})$ & $\red{\times}$ & -- \\
\midrule
\multicolumn{5}{l}{\textit{General Policy Parameterization}} \\
\midrule
\citet[Theorem 3.2]{montenegro2024last} & Multiple & $\gO(\nicefrac{1}{\epsilon^6})$ & $\green{\checkmark}$ & \texttt{CFA}, \texttt{FND}, $\epsilon_{\text{bias}}$ depends on an \texttt{FND} constant \\
\citet[Corollary 2]{lu2026augmented} & Multiple & $\gO(\nicefrac{1}{\epsilon^6})$ & $\green{\checkmark}$ & Log-linear policy, $\epsilon_{\text{bias}} \leq \gO(\epsilon^{16})$ \\
\citet[Corollary 5]{ding2023last} & Single & $\gO(\nicefrac{1}{\epsilon^6})$ & $\red{\times}$ & Log-linear policy, \texttt{CFA}, $\epsilon_{\text{bias}} \leq \gO(\epsilon^8)$ \\
\citet[Theorem 1]{mondal2026lastiterate} & Single & $\gO(\nicefrac{1}{\epsilon^4})$ & $\red{\times}$ & \texttt{CFA}, \texttt{FND}, $\epsilon_{\text{bias}} \leq \gO(\eps^6)$ \\
\textbf{\cref{corollary:smooth_l2_last_iterate_K_geq_1_fa}} (\textbf{ours}) & Multiple & $\tilde{\gO}(\nicefrac{1}{\epsilon^{3.5}})$ & $\red{\times}$ & Log-linear policy, $\epsilon_{\text{bias}} \leq \gO(\epsilon^{4})$ \\
\bottomrule
\end{tabular}
\end{adjustbox}
\caption{Comparison of prior last-iterate guarantees. 
To enable a fair comparison, the convergence rate is characterized in terms of the number of policy gradient evaluations.
The coverage assumption requires that initial state distribution has full support, i.e.\ $\min_s \rho(s) > 0$. 
A bilevel methods refers to two-time scale method in which the inner sub-problem must be solved to an sufficient accuracy.
With function approximation, $\epsilon_{\text{bias}}$ denotes the irreducible approximation error of the restricted policy class (i.e.~\cref{assumption:q_realizability}), \texttt{CFA} denotes the requirement of compatible function approximation framework~\citep{agarwal2021theory} and \texttt{FND} denotes the analysis requires the Fisher non-degenerate assumption~\citep{zhang2021convergence}. Note the \texttt{FND} assumption does not hold for log-linear policies~\cite[Proposition C.4]{tingting2025safe}.}\label{table:last}
\end{table*}

To address these issues, recent works~\citep{ding2023last,mondal2026lastiterate,muller2024truly,moskovitz2023reload,ying2022dual,montenegro2024last,lu2026augmented} have developed methods with convergence guarantees for the final policy (``last-iterate''). 
However to establish such theoretical guarantees, these approaches lead to impractical algorithms (see~\cref{table:last} for a detailed comparison).
Among these approaches, we consider dual descent (\texttt{DD}) methods~\citep{ying2022dual,gladin2023algorithm} as our starting point.
By incorporating entropy-regularization into the policy optimization problem, \texttt{DD} methods obtains guarantees for a final stochastic policy. 
However, existing \texttt{DD} methods suffer from three key limitations:
(i) their guarantees relies on a impractical coverage assumption (an assumption that implies that every policy can visit every state);
(ii) they are cumbersome to deploy, as they are bilevel methods for which every update to the dual variable requires accurately solving an entropy-regularized MDP sub-problem in an inner-loop; and
(iii) their guarantees are restricted to the tabular setting~\citep{ying2022dual,gladin2023algorithm}.
We address these limitations, improving the practicality and generality of \texttt{DD} methods, and make the following contributions. 

\paragraph{Contribution 1.}
In~\cref{section:inexact_dd}, we provide a new analysis of \texttt{DD} methods and establishes last-iterate guarantees for the final policy without any impractical coverage assumptions.
Following~\citet{ying2022dual} and~\citet{gladin2023algorithm}, we first consider the tabular setting and analyze the \texttt{DD} method when using entropy-regularized natural policy gradient (\texttt{E-NPG})~\citep{cen2022fast} as the inner sub-problem  solver. Our new analysis achieves an $\tilde{\gO}(\nicefrac{1}{\eps^4})$ last-iterate rate. We then develop a new algorithm that relies on dual-regularization that attains an $\tilde{\gO}(\nicefrac{1}{\eps^2})$ rate matching~\citet{ying2022dual}. 
However, our algorithm does not employ Nesterov acceleration; it only relies on gradient descent on the regularized dual and is consequently simpler to implement.

\paragraph{Contribution 2.}
We then address the second limitation of \texttt{DD} methods. Their bilevel structure requires using multiple \texttt{E-NPG} iterations to solve the inner sub-problem to an sufficient accuracy. In~\cref{section:dd_to_pd}, we consider a more practical algorithm by using a single \texttt{E-NPG} update which results in a simple primal-dual (\texttt{PD}) method. Through a novel tracking analysis that controls the bias (as a result of solving the inner sub-problem inexactly), we prove that this method attains an $\tilde{\gO}(\nicefrac{1}{\eps^5})$ last-iterate rate.In contrast to prior \texttt{PD} methods without coverage assumptions (refer to~\cref{table:last}), our approach achieves a substantially faster last-iterate rate.
While related \texttt{PD} methods have been proposed~\citep{muller2024truly,mondal2026lastiterate}, their analysis is either limited to the finite-horizon setting or supports only a single constraint. 

\paragraph{Contribution 3.}
In~\cref{section:truncate}, we further refine our tracking analysis. For a good initial policy (obtained by solving the first inner sub-problem to sufficiently high accuracy), we prove the inexact bias is bounded uniformly and decreases exponentially with the number of \texttt{E-NPG} iterations. We use this property to show that for a fixed number of $K$ inner-iterations, the resulting algorithm attains an $\widetilde{\mathcal{O}}(\nicefrac{1}{\epsilon^{2+3/2^K}})$ last-iterate rate. In the case when $K=1$, the resulting \texttt{PD} method attains a $\tilde{\gO}(\nicefrac{1}{\epsilon^{3.5}})$ convergence. This method achieves fastest last-iterate rate for \texttt{PD} methods without any coverage assumptions. Furthermore as $K \to \infty$, the method converges (linearly in terms of $K$) to $\tilde{\gO}(\nicefrac{1}{\eps^2})$, matching the bilevel $\texttt{DD}$ approach. Thus, this analysis interpolates between the \texttt{PD} (corresponding to $K=1$) and \texttt{DD} (corresponding as $K \to \infty$) regimes.

\paragraph{Contribution 4.}
Finally, in~\cref{appendix:handling_linear_fa}, we address the final limitation of \texttt{DD} methods --- their restriction to the tabular setting. In particular, we extend our analysis to handle large state-action spaces and consider log-linear policies~\citep{yuan2022general,agarwal2021theory}. For this setting, we solve the inner sub-problem using \texttt{Q-E-NPG}~\citep{agarwal2021theory}, an extension of \texttt{E-NPG} when employing linear function approximation. Notably, this extension only requires reanalyzing how the inner sub-problem is solved while the remainder of the analysis remains unchanged. We use the policy features to construct a coreset to control the resulting approximation error of the soft state-action values functions. For sufficiently expressive features, this yields an $\tilde{\gO}(\nicefrac{1}{\eps^{3.5}})$ last-iterate rate, achieving the best known rate when handling function approximation.

\section{Problem Formulation \& Background}
\looseness=-1
An infinite-horizon discounted constrained Markov decision process (CMDP)~\citep{altman2021constrained} is defined by the tuple $(\cS, \cA, \cP, \rho, \gamma, r, \{c_i\}_{i=1}^m, \{b_i\}_{i=1}^m)$. Here, $\cS$ is the set of states, $\cA$ is the set of actions, $\cP: \cS \times \cA \to \Delta_{\cS}$ is the transition probability function, $\rho \in \Delta_\cS$ is the initial state distribution, $\gamma \in [0, 1)$ is the discount factor and $r: \cS \times \cA \to [0, 1]$ is the reward function. The set of $m$ constraint reward functions is denoted by $\{c_i\}_{i=1}^m$ where $c_i: \cS \times \cA \to [0, 1]$ and $\{b_i\}_{i=1}^m$ is the set of corresponding non-negative thresholds.
For a policy $\pi: \cS \to \Delta_{\cA}$, the \textit{state-action value function} $Q^\pi_\diamond: \cS \times \cA \to \R$ is defined as $Q^\pi_\diamond(s, a) \coloneq \E[\sum_{h=0}^\infty \gamma^h \diamond(s_h, a_h) | s_0 = s , a_0 = a]$ such that for $h \geq 0$, $s_{h+1} \sim \cP(\cdot | s_h, a_h)$, $a_{h+1} \sim \pi(\cdot | s_h)$ where $\diamond : \cS \times \cA \to \R$ is a bounded (pseudo) reward function.
The \textit{value function} $V^\pi_\diamond: \cS \to \R$ is defined as $V^\pi_\diamond(s) \coloneq \E_{a \sim {\pi(\cdot | s)}}[Q^\pi_\diamond(s, a)]$, and we use $V^{\pi}_{\diamond}(\rho) := \E_{s_0 \sim \rho}[V^{\pi}_{\diamond}(s_0)]$ to denote the expected value function. Using this notation, the \textit{reward state-action value function} and \textit{reward value function} are denoted as $Q^\pi_r$ and $V^{\pi}_r$. 
Likewise, for each constraint $i \in [m]$, we denote the \textit{constraint reward state-action value function} as $Q^\pi_{c_i}$ and the \textit{constraint reward value function} as $V^{\pi}_{c_i}$.

For a given CMDP, the objective is to find a policy $\pi$ that maximizes the expected reward value function while satisfying the set of expected constraint reward value functions in expectation. Formally, this corresponds to solving the following  problem:
\begin{problem}\label{problem:cmdp}
$\max_{\pi}  V^{\pi}_{r}(\rho) \text{ subject to } V^{\pi}_{c_i}(\rho) \geq {b_i} \quad \forall i \in [m]$.
\end{problem}

Let $\pistar$ denote the optimal policy of~\cref{problem:cmdp} and note in general, unlike unconstrained MDPs, $\pistar$ is stochastic~\citep{altman2021constrained}. A typical approach to solve~\cref{problem:cmdp} is to 
construct the Lagrangian and consider the following saddle-point problem:
\begin{equation}\label{eq:lag}
    \max_{\pi} \min_{\lambda \in \R_+^m} L(\pi, \lambda) \coloneqq V^{\pi}_{r}(\rho) + \sum_{i=1}^m \lambda_i (V^{\pi}_{c_i}(\rho) - b_i)
\end{equation}
where $\lambda \in \R^m_+$ is the dual variable.
For example, primal-dual (\texttt{PD}) methods~\citep{ding2020natural} solve~\cref{problem:cmdp} by iteratively updating the policy (e.g., via natural policy gradient) and the dual variables (e.g., via gradient descent). 
For such a \texttt{PD} algorithm, prior work~\citep{jain2022towards,ding2020natural,liu2025sample,paternain2019constrained} establishes after $T$ iterations, for a target error $\eps > 0$,  the mixture policy $\bar{\pi}_T \coloneqq \frac{1}{T} \sum_{t=0}^{T-1} \pi_t$ is an $\eps$-optimal policy, i.e.,
\begin{equation*}
V^{\bar{\pi}_T}_{r}(\rho) \geq V^{\pistar}_{r}(\rho) - \eps,  \text{ ; } V^{\bar{\pi}_T}_{c_i}(\rho) \geq b_i - \eps\quad \forall i \in [m] 
\end{equation*}

Despite their strong theoretical guarantees, mixture policies suffer from two practical drawbacks. First, they are memory intensive since they require storing all intermediate policies $\{\pi_t \}_{t=0}^{T-1}$. Second, deploying such a mixture policy requires repeatedly sampling a policy uniformly from this set. As a result, constraint satisfaction is only guaranteed in expectation over this random selection. For safety-critical applications, this behaviour is undesirable as any single policy may significantly violate the constraints. This has motivated the development of principled algorithms that converge to a single deployable stochastic policy.

Prior approaches are developed by introducing additional regularization, either on the policy~\citep{lu2026augmented,ying2022dual,gladin2023algorithm}, dual variables~\citep{montenegro2024last} or to both~\citep{mondal2026lastiterate,ding2023last,muller2024truly}. 
These algorithmic designs provide guarantees for the reward sub-optimality and constraint satisfaction of the final policy (``last-iterate'') produced by the algorithm.
However to establish such guarantees, most prior analysis requires strong coverage assumptions on the initial state space that are difficult to satisfy in practice. Moreover, extending these guarantees to general policy parameterization often requires additional, overly restrictive assumptions on the policy class (refer to~\cref{table:last} for details).
To address these issues, we use the dual descent (\texttt{DD}) method as the starting point.

\section{Inexact Dual Descent}\label{section:inexact_dd}
In~\cref{section:dual-problem}, we first introducing the entropy-regularized CMDP and formulate its dual problem. In~\cref{section:algorithm} we develop an inexact \texttt{DD} algorithm for solving this problem and analyze a generic convergence guarantee without any coverage assumption in~\cref{section:abstract-analysis}. We then specialize the analysis to the tabular setting, where the number of states and actions are finite, and in~\cref{section:tabular} establish last-iterate guarantees for the resulting algorithms. Finally, in~\cref{section:ell_two} we add dual-regularization and attain faster rates.

\subsection{Constructing a Differentiable Dual}\label{section:dual-problem}
An common approach for constrained optimization problems is to consider the dual function of the Lagrangian in~\cref{eq:lag}, $D(\lambda) \coloneqq \max_{\pi} L(\pi, \lambda)$ and solve the convex problem $\min_{\lambda \in \R^m_+} D(\lambda)$ using gradient-based methods~\citep{boyd2004convex}.  However, a key challenge is that for a fixed $\lambda \in \R^m_+$, the maximizer $\argmax_{\pi} L(\pi, \lambda)$ may not be unique, making the dual function non-differentiable. To address this issue, following prior work~\citep{ying2022dual,li2024faster,gladin2023algorithm}, we consider solving an entropy-regularized primal. This enables the dual of the regularized primal to have a unique solution and enables the gradient to be well-defined.

In particular, let $\gH(\pi(\cdot | s)) \coloneqq - \dpd{\pi(\cdot | s), \ln( \pi(\cdot | s))}$ 
denote the Shannon entropy of a policy $\pi$ for any state $s \in \cS$. For the entropy parameter $\tau > 0$, the \emph{soft (entropy-regularized) value function} $V^{\pi}_{\diamond, \tau}: \cS \times \cA \to \R$ is defined as  $V^{\pi}_{\diamond, \tau}(s) \coloneqq V^{\pi}_{\diamond}(s) + \tau \sum_{h=0}^\infty \gamma^h[\gH(\pi(\cdot | s_h) \mid s_0 = s]$ and \emph{soft state-action-value function} $Q^{\pi}_{\diamond, \tau}: \cS \times \cA \to \R$ is defined as  $Q^{\pi}_{\diamond, \tau} \coloneqq \E_{s' \sim \cP(\cdot | s, a)}[\diamond(s, a) + \gamma V^{\pi}_{\diamond, \tau}(s')]$~\citep{geist2019theory,liu2024elementary}.

Using this notation, we define the soft reward value function as $V^{\pi}_{r, \tau}(\rho) \coloneqq \E_{s \sim \rho}[V^{\pi}_{r, \tau}(s)]$, and focus on the following soft Lagrangian formulation:
\begin{problem}\label{problem:soft_max_min_lag}
$\max_{\pi} \min_{\lambda \in \R^m_+} \cL_\tau(\pi, \lambda)$  \text{ where }  $\lag(\pi, \lambda) \coloneqq V^{\pi}_{r, \tau}(\rho) + \sum_{i=1}^m \lambda_i \, (V^{\pi}_{c_i}(\rho) - b_i)$.
\end{problem}
\noindent Hence, for any $\lambda \in \R^m_+$, entropy regularization ensures that the maximizer of the Lagrangian:
\begin{equation}\label{eq:dual_entropy}
\cD(\lambda) := \max_{\pi} \lag(\pi, \lambda) \quad \text{ ; } \quad \lstar := \argmin_{\lambda \in \R^m_+} \cD(\lambda) \, 
\end{equation} 
is unique~\citep{nachum2017bridging}.  Consequently, by Danskin's theorem~\citep{floudas1995nonlinear}, its gradient is well-defined and characterized as follows.
\begin{restatable}{proposition}{propositiondualgradient}\label{proposition:dual_gradient}
For any $\lambda \in \R^m_{+}$ and entropy regularization parameter $\tau > 0$, the dual function $\cD(\lambda)$ in~\cref{eq:dual_entropy} is differentiable such that $\forall i \in [m]$, 
\begin{equation*}
[\nabla \cD(\lambda)]_i \coloneqq V^{\pistarl}_{c_i}(\rho) - b_i \quad  \text{;} \quad \pistar_{\lambda,\tau} \coloneqq \argmax_{\pi} \lag(\pi, \lambda)
\end{equation*} 
\end{restatable}
\noindent For a fixed dual variable $\lambda$, computing this gradient exactly requires solving an an entropy-regularized MDP as the inner sub-problem. We therefore consider an approximate solution to this maximization problem and will use it to construct \emph{inexact} gradients.

\subsection{Algorithm Design}\label{section:algorithm}
Building on the \texttt{DD} approach of~\citet{ying2022dual}, we now present the following generic inexact
\texttt{DD} algorithm, which employs a black-box entropy-regularized MDP solver to construct the inexact dual gradients. 
\begin{algorithm}[ht]
\begin{algorithmic}[1]
    \STATE \textbf{Input}: $\pi_{-1}$ (initial policy), $\lambda_0 \in \R^m_+$ (initial dual variable), $\tau$ (entropy regularization),  $\beta$ (dual step-size), $\mu$ (dual regularization), $T$ (number of outer-iterations), $K$ (number of inner-iterations), $J$ (number of recovery-iterations)
    \FOR{$t = 0, 1, \dots, T-1$ }
        \STATE $\pi_{t} = \texttt{PolicyOptimization}(\pi_{t-1}, R(\lambda_t), \tau, K)$ \\ where $R(\lambda_t) = r + \sum_{i=1}^m c_i \, \lambda_t[i]$ \label{line:policyoptimization}.
        \STATE Construct inexact dual gradient estimate $[\bdgt]_i \coloneqq V^{\pi_{t}}_{c_i}(\rho) - b_i$, $\forall i \in [m]$. \label{line:consruct_dual}
        \STATE $\lambda_{t+1} = \text{Proj}_{\Lambda}\left[(1 - \beta \mu) \lambda_t - \beta \bdgt\right]$ \label{line:update}
        \ENDFOR
    \STATE $\lambda_T = \frac{1}{T} \sum_{t=1}^{T} \lambda_t$ if $\mu = 0$ else $\lambda_T$
    \STATE Recover deployable policy from dual variable: $\pi_{T} = \texttt{PolicyOptimization}(\pi_{-1}, R(\lambda_T), \tau, J)$ \label{line:recover}
    \STATE \textbf{Return}: Policy $\pi_T$ and (average) dual variable $\lambda_T$
\caption{Inexact \texttt{DD} for CMDPs}
\label{algorithm:dd}
\end{algorithmic}
\end{algorithm}

\paragraph{Algorithmic details.} 
\cref{algorithm:dd} is a bilevel method that uses an approximate solution to the inner sub-problem to construct an inexact gradient for the outer dual update.
By the linearity of the reward $r$ and the weighted constrained reward $\{\lambda_i c_i\}_{i=1}^m$, computing $\pistarl \coloneqq \argmax_{\pi} \lag(\pi, \lambda)$ is equivalent to solving the following inner sub-problem induced by $\lambda$,
\begin{equation}
\textstyle \argmax_{\pi} V^{\pi}_{R(\lambda), \tau}(\rho) \quad \text{ ; } \quad R(\lambda) \coloneqq r + \sum_{i=1}^m \lambda_i \, c_i.
\end{equation}
Thus, constructing the dual gradient of~\cref{eq:dual_entropy} reduces to solving an \emph{unconstrained} entropy-regularized MDP with pseudo-reward $R(\lambda)$. We can thereby obtain inexact dual gradients by using any entropy-regularized MDP solver as a black-box. 
To this end, in~\Cref{line:policyoptimization}, we introduce an abstract method, \\ \texttt{PolicyOptimization}$(\pi_{t-1}, R(\lambda_t), \tau, K)$, which takes in as input an initial policy $\pi_{t-1}$, pseudo-reward $R(\lambda_t)$, an entropy regularizer parameter $\tau$, and iterations $K$
and outputs a policy $\pi_t$.
Then in~\Cref{line:consruct_dual}, using the approximate solution to the entropy-regularized MDP $\pi_t$, we construct the corresponding inexact dual gradient $\bdgt \in \R^m$ such that $[\bdgt]_{i} = V^{\pi_t}_{c_i}(\rho) - b_i$ for all $i \in [m]$. 

We first consider the case when $\mu = 0$ and will handle the regularized case in~\cref{section:ell_two}.
Using the inexact dual gradient estimate, we update the dual variable using projected dual descent (\Cref{line:update}). 
In particular, given a gradient estimator $\bdgt$, we consider the following update $\lambda_{t+1} = \text{Proj}_{\Lambda}\left[\lambda_t - \beta \bdgt\right]$
where $\beta > 0$ is the dual step-size and $\text{Proj}_{\Lambda}[\lambda]$ denotes the Euclidean projection onto a convex set $\Lambda$, i.e., $\text{Proj}_{\Lambda}[\lambda] \coloneqq \argmin_{\lambda' \in \Lambda} \normsq{\lambda - \lambda'}$.
The projection set $\Lambda$ is chosen to ensure the dual iterates remain bounded while containing the optimal dual solution.
To obtain an upper bound on the magnitude $\lstar$, following prior work~\citep{ying2022dual,gladin2023algorithm,ding2023last} we consider the following assumption. 
\begin{assumption}[Slater Condition]\label{assumption:slater_condition}
$\exists$ policy $\bar{\pi}$ and $\zeta \in \R^m_+$ such that $\forall i \in [m]$, $\zeta_i \coloneqq V^{\bar{\pi}}_{c_i}(\rho) - b_i > 0$.
\end{assumption}
\noindent This assumption ensures the existence of a strictly feasible policy.
The constant $\zeta$ is referred to as \emph{Slater constant}~\citep{ding2020natural,bai2022achieving,liu2025convergence,vaswani2022near}. 
This quantity measures the width of the region of feasible policies and is a measure of the problem's difficulty.
Furthermore by~\cref{lemma:strong_duality,lemma:optimal_dual_bounded}, these lemmas show that $\lstar$ is bounded component-wise by 
$\left[\lstar\right]_i \leq \frac{2 + 2 \, \tau \, \ln(A)}{(1-\gamma)\zeta_i}$.
Hence, in order for the set to contain $\lstar$, we may restrict the dual iterates to the compact set $\Lambda \coloneqq \left\{ \lambda \in \R^m_+ \mid 0 \leq \lambda_i \leq \frac{2 + 2 \, \tau \, \ln(A)}{(1 - \gamma) \zeta_i} \right\}$.
Finally, in~\cref{line:recover}, we obtain a deployable policy by solving an entropy-regularized MDP induced by a fixed dual.

\subsection{Theoretical Analysis}
\label{section:abstract-analysis}
\Cref{algorithm:dd} aims to minimize the dual problem.  
We first show how the sub-optimality of the dual solution relates to the reward sub-optimality and constraint violation for~\cref{problem:cmdp}.
\begin{restatable}{lemma}{dualoptimalitytosublemma}\label{lemma:dual_optimality_to_sub}
\looseness=-1
Under~\cref{assumption:slater_condition}, suppose $\overline{\lambda} \in \R^m_{+}$ is an $\Delta$-optimal dual solution, i.e., for $\Delta > 0$, $\cD(\overline{\lambda}) - \cD(\lstar) \leq \Delta$.
Then, the policy 
$\pistarbl \coloneqq \argmax_\pi V^{\pi}_{R(\bar{\lambda}), \tau}(\rho)$ has the following guarantee %
for~\cref{problem:cmdp}: 
\begin{align*}
V^{\pi^*}_r(\rho) - V^{\pistarbl}_{r}(\rho) &\leq \gO\parens*{\frac{\sqrt{\Delta}}{{\sqrt{\tau}}}} + \gO(\tau) 
\quad \text{;} \quad \max\braces*{b_i - V^{\pistarbl}_{c_i}(\rho), 0} \leq \gO\parens*{\frac{\sqrt{\Delta}}{{\sqrt{\tau}}}} \quad \forall i \in [m] \, .
\end{align*}
\end{restatable}
\noindent We next characterize the convergence of~\cref{algorithm:dd} and bound the dual sub-optimality. In order to do so, we first establish that the dual objective is smooth by leveraging the Lipschitz property of entropy-regularized policies with respect to $\lambda$. Unlike prior derivations~\citep{ying2022dual,gladin2023algorithm}, this result does not rely on any strong coverage assumptions on either the initial state distribution or the optimal policy. \begin{restatable}{lemma}{dualsmooth}\label{lemma:dual_smooth}
\looseness=-1
For $\tau > 0$ and any $\lambda \in \R^m_+$, the dual function $\cD(\lambda)$ is $G_\tau \coloneqq \frac{m}{\tau(1 - \gamma)^3}$-smooth.
\end{restatable}
Assuming that the inexact dual gradients are bounded at each iteration,
using this smoothness property, we now prove a generic convergence guarantee for~\cref{algorithm:dd}. Specifically, the following result corresponds to analyzing projected gradient descent on a convex, smooth function with bounded, inexact gradients. 
\begin{restatable}{lemma}{smoothmetalemma}\label{lemma:meta_lemma_smooth}
Under~\cref{assumption:slater_condition}, for any $\tau > 0$, assuming that for all $t \in [0,T-1]$, $\norm{\nabla \cD(\lambda_t) - \widetilde{\nabla} \cD(\lambda_t)} \leq \delta_t$,~\cref{algorithm:dd} using step-size $\beta \leq \frac{1}{G_\tau} $, $\lambda_0 = 0$, and $\mu = 0$ converges as: for $T \geq 1$, 
\begin{equation}
\cD(\bar{\lambda}_T) - \cD(\lstar) \leq  \frac{\normsq{\lstar}}{2 \, \beta \, T}  +  \frac{\abs{\Lambda}}{T} \textstyle \sum_{t=0}^{T-1} \delta_t  \numberthis  \label{eq:dd_smooth_average_sub_opt}
\end{equation}
where $\bar{\lambda}_T \coloneqq \frac{1}{T} \sum_{t=1}^{T} \lambda_t$, and $\abs*{\Lambda} = \frac{\sqrt{m} \, (2 + 2 \tau \, \ln(A))}{(1 - \gamma) \, \min_{i \in [m]}\zeta_i}$.
\end{restatable}
\noindent The above results shows that the dual sub-optimality decreases as $\gO(\nicefrac{1}{T})$ up to a bias introduced by the inexact gradients. Hence, the remaining step is to control the bias introduced by using inexact dual gradients. Specifically, in order to combine~\cref{lemma:dual_optimality_to_sub} with the dual sub-optimality guarantee of~\cref{lemma:meta_lemma_smooth}, we require the accumulated bias term to be sufficiently small. To this end, in~\cref{lemma:bound_bias_subgradient} we bound the bias by the sub-optimality of the inner sub-problem.
\begin{restatable}{lemma}{boundbiassubgradientlemma}\label{lemma:bound_bias_subgradient}
For any $\lambda \in \R^m_+$, $\tau > 0$ and policy $\tilde{\pi}$,
the inexact gradient such that $[\bdg]_i = V^{\tilde{\pi}}_{c_i}(\rho) - b_i$ for all $i \in [m]$ results in the following bias:
\begin{align*}
\norm{\dg - \bdg} \leq 
\frac{\sqrt{ 2 \, m} \, \sqrt{\max_\pi V^{\pi}_{R(\lambda), \tau}(\rho) - V^{\tilde{\pi}}_{R(\lambda), \tau}(\rho)}}{(1 - \gamma)^{\nicefrac{3}{2}} \, \sqrt{\tau}} \,
\end{align*}
\end{restatable}
\noindent Consequently, the bias can be controlled through the accuracy of the inner black-box policy optimization method in~\cref{algorithm:dd}. 
Hence, to obtain a single deployable policy, we instantiate the abstract \texttt{PolicyOptimization} method with entropy-regularized natural policy gradient (\texttt{E-NPG})~\citep{cen2022fast,lan2023policy}.

\subsection{Instantiating the Algorithm}\label{section:tabular}
For a generic bounded pseudo-reward $\diamond\in[0,U]^{SA}$, $U>0$, we employ \texttt{E-NPG} for $K$ iterations.
At inner-iteration $k \in[0,K-1]$, \texttt{E-NPG} iteratively updates policy $\pi_k$ using an estimate of the soft state-action value function $\hat{Q}^{\pi_k}_{\diamond, \tau} \in \R^{SA}$ using the following update:
$\forall (s, a) \in \cS \times \cA$
\begin{equation}
    \pi_{k+1}(a | s) \propto \parens*{\pi_k(a | s)}^{1 - \frac{\eta \, \tau}{1 - \gamma}} \, \exp\parens*{\frac{\eta \, \hat{Q}^{\pi_k}_{\diamond, \tau}(s, a)}{1 - \gamma}}  \label{eq:inexact_entropy_npg_update}
\end{equation}
where $\eta > 0$ denotes the policy step-size.
Furthermore, the following result from Cen et al.~\cite[Theorem 1]{cen2022fast} shows that \texttt{E-NPG} converges linearly for any bounded pseudo-reward up to an floor that depends on inexactness in the state-action value function.
For example, such an error can arise when evaluating the policy $\pi_k$ using a finite number of trajectories.
\begin{restatable}{lemma}{inexactentropynpgsupnormconvergencelemma}\label{lemma:inexact_npg}
\looseness=-1  
For any pseudo-reward $\diamond \in [0, U]^{SA}$, $U > 0$, 
assuming that $\sup_{k \geq 0}\|\hat{Q}^{\pi_k}_{\diamond, \tau} - Q^{\pi_k}_{\diamond, \tau}\|_\infty \leq \gE$, the inexact \texttt{E-NPG} update in~\cref{eq:inexact_entropy_npg_update} with $\tau > 0$, $\eta \leq \frac{1 - \gamma}{\tau}$ results in the following convergence: for $K \geq 0$
\begin{align*}
    \|Q^{\pi^*_{\diamond,\tau}}_{\diamond, \tau} - Q^{\pi_{K+1}}_{\diamond, \tau}\|_\infty &\leq \gamma \bracks*{(1 -\eta \, \tau)^{K} \, C + E} \\
    \|\ln(\pi^*_{\diamond,\tau}) - \ln(\pi_{K+1})\|_\infty &\leq \frac{2}{\tau} \bracks*{(1 -\eta \, \tau)^{K} \, C + E} \\
    V^{\pi^*_{\diamond,\tau}}_{\diamond, \tau}(\rho) - V^{\pi_{K+1}}_{\diamond, \tau}(\rho) &\leq (2 + \gamma) \, [C \, (1 - \eta \, \tau)^{K} + E]
\end{align*}
where $\pi^*_{\diamond,\tau} \coloneqq \argmax_\pi V^{\pi}_{\diamond, \tau}(\rho)$, $C \coloneqq \|Q^{\pi^*_{\diamond,\tau}}_{\diamond, \tau} - Q^{\pi_0}_{\diamond, \tau}\|_\infty + 2 \, \tau \, (1 - \frac{\eta \, \tau}{1 - \gamma}) \, \|\ln( \pi^*_{\diamond, \tau}) - \ln(\pi_0)\|_\infty$ and $E \coloneqq \frac{2 \, \gE}{1 - \gamma} \, (1 + \frac{\gamma}{\eta \, \tau})$.
\end{restatable}
\noindent For a cleaner exposition, we first consider the case where the update is exact,  i.e., $\hat{Q}^{\pi_k}_{\diamond,\tau}=Q^{\pi_k}_{\diamond,\tau}$ which implies that $E=0$. We defer handle a general form when $E > 0$ in~\cref{appendix:linear_fa}. Due to the projection in~\cref{algorithm:dd}, the pseudo-reward $R(\lambda_t)$ remains bounded for all outer-iterations $t \in [0,T-1]$. Therefore, applying the above lemma with $\diamond = R(\lambda_t)$ results in the following convergence result,
\begin{equation}
 \max_{\pi} V^{\pi}_{R(\lambda_t), \tau}(\rho) - V^{\pi_{K+1}}_{R(\lambda_t), \tau}(\rho) \leq \gO((1 - \eta\,\tau)^{K}) \, \label{eq:npg_grad_error}.
\end{equation}
Hence, \texttt{E-NPG} serves as an efficient solver for
the inner sub-problem.

Putting everything together, we characterize the convergence rate of~\cref{algorithm:dd} and prove the following theorem in~\cref{appendix:smooth_k_ln_eps}. To enable a fair comparison to prior work, we characterize the convergence using the total number of policy gradient evaluations across the entire procedure. This corresponds to $K \, T$ gradient evaluations (corresponding to $K$ inner-loops for each of the $T$ outer-iterations) and the final $J$ iterations to recover a deployable policy.
\begin{restatable}{theorem}{smoothddlastiteratetheorem}\label{theorem:meta_smooth}
Under~\cref{assumption:slater_condition}, for a given $\eps \in (0, 1)$,~\cref{algorithm:dd} (i) with $\tau = \gO(\eps)$, $T = \gO(\nicefrac{1}{\eps^4})$, $\beta = \frac{\tau \, (1 - \gamma)^3}{m}$, $\mu = 0$, $\lambda_0 = 0$ and (ii) using the exact \texttt{E-NPG} update with $\eta = \gO\parens*{\frac{1 - \gamma}{\tau}}$ and $K = J = \gO(\ln(\nicefrac{1}{\eps}))$ returns an 
$\gO(\eps)$-optimal policy to~\cref{problem:cmdp} after $\gO\parens*{\nicefrac{1}{\eps^4} \ln(\nicefrac{1}{\eps})}$ gradient evaluations.
\end{restatable}
\paragraph{Comparison to prior work.}\citet{ying2022dual} employ Nesterov acceleration to minimize the dual and attain a faster $\tilde{\gO}(\nicefrac{1}{\eps^2})$ last-iterate rate.  
However, this rate comes at the cost of a more complicated algorithm. In the next section, we investigate an alternative approach that uses an $\ell_2$ regularization term in the dual objective, and achieves the same improved convergence rate.

\subsection{Faster Rates with $\ell_2$ regularization}\label{section:ell_two}
For the $\ell_2$-regularization parameter $\mu > 0$, the update in~\cref{line:update} 
of~\cref{algorithm:dd} corresponds to inexact gradient descent on the regularized dual:
\begin{equation}\label{eq:dual_entropy_ell2}
\F(\lambda) \coloneqq \max_{\pi} \lag(\pi, \lambda) + \frac{\mu}{2} \, \normsq{\lambda} .
\end{equation} 
Furthermore, by~\cref{lemma:reg_optimal_dual_bounded}, we can verify that the unique minimizer $\lstarmu  \coloneqq \argmin_{\lambda \in \R^m_+} \F(\lambda)$ is contained within $\Lambda$.
Analogous to~\cref{lemma:dual_optimality_to_sub}, we first relate the accuracy of the regularized dual solution to the reward sub-optimality and constraint violation for~\cref{problem:cmdp}.
By exploiting the KKT conditions in~\cref{lemma:reg_dual_kkt}, we obtain the following result. 
\begin{restatable}{lemma}{regdualdisttosublemma}\label{lemma:reg_dual_distance_to_sub}
Under~\cref{assumption:slater_condition}, for any $\mu > 0$, $\tau > 0$ and $\lambda \in \Lambda$, the policy $\pistarl \coloneqq \argmax_\pi V^{\pi}_{R(\lambda), \tau}(\rho)$ satisfies: 
\begin{align*}
    V^{\pi^*}_{r}(\rho) - V^{\pistarl}_{r}(\rho) 
    &\leq \gO\parens*{G_\tau \, \norm{\lambda - \lstarmu}} + \gO(\tau) \, , \\
    \max\braces*{b_i - V^{\pistarl}_{c_i}(\rho), 0} 
   &\leq \gO\parens*{G_\tau \norm{\lambda - \lstarmu}} + \gO(\mu)  \quad \forall i \in [m]
\end{align*}
where $G_\tau \coloneqq \frac{m}{\tau \, (1 - \gamma)^3}$ and $\lstarmu = \argmin_{\lambda \in \Lambda} \F(\lambda)$.
\end{restatable}
\noindent Hence, controlling the reward sub-optimality and constraint violation requires controlling the dual sub-optimality $\norm{\lambda - \lstarmu}$. Assuming that the inexactness in the gradients are bounded,
similar to~\cref{lemma:meta_lemma_smooth}, we consider projected gradient descent on a smooth, strongly-convex objective with inexact gradients.
\begin{restatable}{lemma}{regduallinearconvergencelemma}\label{lemma:reg_dual_linear_convergence}
Under~\cref{assumption:slater_condition}, assuming that for all $t \in [0,T-1]$, $\norm{\bdgt - \dgt} \leq \delta_t$,~\cref{algorithm:dd} using step-size $\beta \leq \nicefrac{2}{(G_{\tau} + 2 \, \mu)}$, $\lambda_0 = 0$ results in the following convergence: for $T \geq 1$
\looseness=-1
\begin{align*}
    \norm{\lambda_T - \lstarmu} \numberthis \label{eq:reg_gd_general}
    \leq \exp\parens*{-\frac{\beta \mu T}{2}} \norm{\lstarmu} + \beta \sum_{t=0}^{T-1} \parens*{1 - \frac{\beta \, \mu}{2}}^{T - t - 1} \, \delta_t 
\end{align*}
\end{restatable}
\noindent In~\cref{lemma:reg_optimal_dual_bounded}, we show that the magnitude of $\lstarmu$ is bounded.
In contrast to~\cref{lemma:meta_lemma_smooth} where the bias $\delta_t$ accumulates additively, here, the accumulated bias decays geometrically over the previous iterations. This discounting is key to obtaining a faster convergence rate. Combining the regularized dual convergence in~\cref{lemma:reg_dual_distance_to_sub} with~\cref{lemma:reg_dual_linear_convergence}, and using \texttt{E-NPG} as the policy optimization method, we now present a faster last-iterate rate in the following theorem.
The proof is similar to~\cref{theorem:meta_smooth} and is provided in~\cref{appendix:k_ln_eps_smooth_reg}.
\begin{restatable}{theorem}{regsmoothnessltworate}\label{theorem:reg_smoothness_l2_rate}
\looseness=-1
Under~\cref{assumption:slater_condition}, for a given $\eps \in (0, 1)$,~\cref{algorithm:dd} with (i) $\tau = \mu = \gO(\eps)$, $T = \tilde{\gO}(\nicefrac{1}{\eps^2})$ $\beta = \gO(\eps)$, $\lambda_0 = 0$
(ii) using the exact \texttt{E-NPG} update with $\eta = \gO(\frac{1-\gamma}{\tau})$ and $K = J = \gO(\ln(\nicefrac{1}{\eps}))$ returns $\gO(\eps)$-optimal policy to~\cref{problem:cmdp} after $\tilde{\gO}\parens*{\nicefrac{1}{\eps^2}}$ gradient evaluations.
\end{restatable}
\paragraph{Comparison to prior work.} Compared with other bilevel methods in~\cref{table:last}, our method achieves an $\tilde{\gO}(\nicefrac{1}{\eps^2})$ last-iterate rate without any coverage assumptions. This matches the rate of the inexact \texttt{DD} method of~\citet{ying2022dual}, while avoiding a more involved Nesterov acceleration scheme.

The previous results require solving the inner sub-problem to a high accuracy. This makes the resulting algorithm cumbersome to deploy in practice. Consequently, in the next section, we show only a constant (even a single) number of policy optimization steps of \texttt{E-NPG} suffice for attain last-iterate convergence.

\section{From Dual Descent with $\texttt{NPG}$ to Primal-Dual}\label{section:dd_to_pd}
We first present the main idea and show warm-starting \texttt{E-NPG} from the previous policy iterate allows a constant number of updates to control the accumulated bias.
We then establish the corresponding last-iterate guarantee.
For clarity, we consider when $\mu > 0$, which yields a sharper convergence rate. The case when $\mu = 0$ follows from a similar analysis and we defer the corresponding details to~\cref{appendix:k_one_smooth}.

\looseness=-1
In order to gain some intuition as to why using a constant number of policy updates is possible, recall from~\cref{lemma:reg_dual_linear_convergence}, 
obtaining an accurate dual solution using~\cref{algorithm:dd} 
requires control of the accumulated discounted bias: $\beta \, \sum_{t=0}^{T-1} (1 - \nicefrac{\beta \, \mu}{2})^{T-t-1} \delta_t$  where $\norm{\nabla \cD(\lambda_t)-\widetilde{\nabla}\cD(\lambda_t)} \leq \delta_t$. Instead of bounding the bias $\delta_t$ uniformly for each iteration $t \geq 0$, we instead control the aggregate bias across iterations.
To this end, observe that for each $i \in [m]$, $[\bdgt]_i = V^{\pi_t}_{c_i}(\rho) - b_i \leq \frac{1}{1 - \gamma}$ since $c_i \in [0, 1]^{SA}$. Hence, the inexact dual gradients are uniformly bounded which implies consecutive dual iterates remain close, i.e., $\norm{\lambda_t - \lambda_{t+1}} \leq \gO(\beta \, \mu)$. As a result, the pseudo-reward and primal iterates across consecutive iterations do not change significantly. This suggests that warm-starting the \texttt{E-NPG} update in~\cref{eq:inexact_entropy_npg_update} at $\pi_t$ allows the policy iterates to efficiently track the target policy $\pi^*_{R(\lambda_{t+1}), \tau}$. Therefore, we do not require solving each inner sub-problem to high-accuracy, and consequently show that a constant (independent of $\eps$) number of \texttt{E-NPG} updates will suffice.

We now formalize this intuition. For a fixed $t \geq 0$, let 
$R_t \coloneqq R(\lambda_t)$, $\pi^*_t \coloneqq \argmax_{\pi} V^{\pi}_{R_t, \tau}(\rho)$ and $Q^{\pi_t}_{t} \coloneqq Q^{\pi_t}_{R_t, \tau}$.
The following lemma introduces the tracking potential used in our analysis. 
In particular, $\Phi_t$ quantifies how well the initial policy $\pi_t$ tracks the optimal policy $\pi_t^*$ of the inner sub-problem at each iteration $t$. 
\begin{restatable}{lemma}{biasboundedbyphilemma}\label{lemma:bias_bounded_by_phi}
For all $t \geq 0$,    
\vspace{-1ex}
\begin{equation*}
\norm{\dgt - \bdgt} \leq \frac{4 \sqrt{m} \, A \, \Phi_t}{\tau \, (1 - \gamma)^2}  \numberthis \label{eq:tracking_bound10}
\end{equation*}
where $\Phi_t =\frac{1}{1 - \gamma} \|Q^{\pistart}_{t} - Q^{\pi_{t}}_{t}\|_\infty + \frac{\tau}{4}  \|\ln(\pistart) - \ln(\pi_t)\|_\infty$.
\end{restatable}
\noindent The goal is now to derive a recursive bound that relates $\Phi_{t+1}$ to $\Phi_t$.
Combining the contraction property of the \texttt{E-NPG} update from~\cref{lemma:inexact_npg} and accounting for the change in the target policy between consecutive dual iterates yields the following tracking bound.
\begin{restatable}{lemma}{phitrackinglemma}\label{lemma:phi_tracking_main}
Let $\pi_{t+1}$ denote the policy obtained from $K>0$ iterations of the exact \texttt{E-NPG} update with step-size $\eta=(1-\gamma)/\tau$ and $\diamond = R_{t+1}$. Then,
\begin{equation*}
\Phi_{t+1} \leq \frac{(1 + \gamma) \, \gamma^{K-1}}{2} \, \parens*{\Phi_t+ \frac{2 \, m \,  \supnorm{\lt - \ltt}}{(1 - \gamma)^2}}
\end{equation*}
\end{restatable}
\noindent Hence $\Phi_{t+1}$ depends on $\Phi_t$ (the error at iteration $t$), a drift term $\supnorm{\lt - \ltt}$ that captures the error incurred because the target policy changes as the dual variable moves from $\lambda_t$ to $\lambda_{t+1}$ and a multiplicative $\gO(\gamma^{K-1})$ term which captures the effect of the \texttt{E-NPG} updates. Since projections are non-expansive, we have $\supnorm{\lt - \ltt} \leq \gO(\beta \mu)$.  Thus, the drift term can be controlled by choosing a small enough dual step-size $\beta$ and regularization $\mu$. Moreover, since $\gamma < 1$, ~\cref{lemma:bias_bounded_by_phi,lemma:phi_tracking_main} imply that even a single \texttt{E-NPG} update can control the gradient bias across outer-iterations. 

We now characterize the convergence of~\cref{algorithm:dd} when using a single \texttt{E-NPG} update and prove the following theorem in~\cref{appendix:k_one_smooth_reg}. 
With $K=1$, each iteration of~\cref{algorithm:dd} performs one \texttt{E-NPG} update followed by one dual update.
Hence, the bilevel \texttt{DD} method reduces to a simple \texttt{PD} method.
\begin{restatable}{theorem}{elltwolastiterateKone}\label{theorem:l2_last_iterate_K_1}
Under~\cref{assumption:slater_condition}, for a given $\eps \in (0, 1)$,~\cref{algorithm:dd} with (i) $\tau = \mu = \gO(\eps)$, $T = \tilde{\gO}(\nicefrac{1}{\eps^5})$, $\beta = \gO(\eps^4)$,  $\lambda_0 = 0$ and (ii) using the exact \texttt{E-NPG} update with $\eta = \frac{1 - \gamma}{\tau}$, $K = 1$, $J = \gO(\ln(\nicefrac{1}{\eps}))$ returns an $\gO(\eps)$-optimal policy to~\cref{problem:cmdp} after $\tilde{\gO}(\nicefrac{1}{\eps^5})$ gradient evaluations.
\end{restatable}
\paragraph{Comparison to prior work.} We now compare to prior \texttt{PD} methods with last-iterate guarantees in~\cref{table:last}. 
Without any coverage assumptions our method achieves a fast $\tilde{\gO}(\nicefrac{1}{\eps^5})$ last-iterate rate while avoiding any coverage assumptions and retaining a practical policy update.
In contrast to~\citet{ding2023last} and~\citet{montenegro2024last}, our rate does not depend on potentially unbounded constants and is much faster than specialized finite-horizon method of~\citet{muller2024truly}.

\cref{theorem:l2_last_iterate_K_1} reveals that the inner MDP does not need to be solved to high accuracy at every iteration. 
We next refine the tracking analysis to obtain a sharper characterization of the benefit of additional inner-iterations, yielding an improved convergence rate when $K=1$ and progressively faster rates as $K$ increases.

\section{Tighter Tracking Analysis}\label{section:truncate}
We first motivate the how warming-starting the initial policy yields a tighter bound on the gradient bias then establish the corresponding last-iterate guarantee. Similar to~\cref{section:dd_to_pd}, we consider when $\mu > 0$ and defer the case when $\mu = 0$ to~\cref{appendix:k_geq_one_smooth}.

The previous tracking bound holds for an arbitrary initialization. In contrast, we show that by warm-starting the initial policy to be sufficiently close to the solution of the first sub-problem, we can obtain a faster convergence rate. In particular, warm-starting ensures that all future policy iterates remain within a neighbourhood of the corresponding target policy (the solution to the inner sub-problem). In this neighbourhood, we show that \texttt{E-NPG} exhibits a quadratic rate of convergence. Consequently, the algorithm guarantees a uniform bias bound that improves exponentially with the number of inner-iterations.

To begin, we first relate the bias to the log-policy error between the current policy $\pi_t$ and the target policy $\pi^*_t$. 
\begin{restatable}{lemma}{biasboundedbyLlemma}\label{lemma:bias_bounded_by_L}
For any $\tau > 0$ and $t \geq 0,$
\begin{align*}
   \norm{\nabla \cD(\lambda_t) - \bdgt} &\leq \frac{\sqrt{m} \, \abs{\cA}}{(1 - \gamma)^2} \, \underbrace{\supnorm{\ln(\pi^*_t) - \ln(\pi_t)}}_{\coloneqq L_t}
\end{align*}
\end{restatable}
\noindent The above bound follows from the proof of~\cref{lemma:bias_bounded_by_phi}, but notably does not depend on $\tau$. In~\cref{appendix:tools_kgeq1}, we prove that each \texttt{E-NPG} update decreases $L_t$ quadratically. This result follows by first relating $L_{t+1}$, the log-policy error  after one update to the KL divergence. Then using Hoeffding's lemma, we prove that this divergence is bounded by $L_t^2$, and obtain the following result. 
\begin{restatable}{lemma}{lrackinglemma}\label{lemma:l_tracking_main}
Let $\pi_{t+1}$ denote the policy obtained from $K \geq 1$ iterations of the exact \texttt{E-NPG} update with step-size $\eta=(1-\gamma)/\tau$ and $\diamond = R_{t+1}$. Then,
\begin{equation*}
L_{t+1} \leq  \left(\frac{\gamma}{1 - \gamma}\right)^{2^K -1}\left(\frac{2 \, \sqrt{m}}{\tau \, (1- \gamma)} \, \norm{\ltt - \lt} + L_t\right)^{2^K}
\end{equation*}
\end{restatable}
\noindent In contrast to~\cref{lemma:phi_tracking_main} which corresponds to a linear contraction in $\Phi_t$, the log-policy error $L_t$ contracts superlinearly with the number of inner-iterations $K$.
However, this sharper recursion yields a contraction only when both the log-policy error and the drift term are sufficiently small. As before, the drift term can be controlled by choosing a small step-size $\beta$ and dual-regularizer $\mu$.
Now in order to control the log-policy error uniformly, we require the initial error $L_0$ to be sufficiently small.
Hence, if the initial policy satisfies $L_0 \leq \gO((\nicefrac{\beta}{\tau})^{2^K})$, then, we show that for a fixed $K \geq 1$,  
the above result ensures that $L_t \leq \gO(\nicefrac{\beta}{\tau})^{2^K})$ for all $t \geq 0$, and thus yields a uniform bound on the bias. Using this bound with~\cref{lemma:bias_bounded_by_L} and the dual sub-optimality of~\cref{lemma:reg_dual_distance_to_sub} we prove the following theorem in~\cref{appendix:k_geq_one_smooth_reg}.
\begin{restatable}{theorem}{ellinterpolatelastiteratetheorem}
\label{theorem:l2_last_iterate_K_geq_1}
Under~\cref{assumption:slater_condition}, for a given $\eps \in (0, 1)$,~\cref{algorithm:dd} with (i) $\tau = \mu = \gO(\eps)$, $T = \gO(\nicefrac{1}{\eps^{2 + 3/2^K}})$, $\beta = \gO\parens*{\eps^{\nicefrac{3 + 2^K}{2^K}}}$, $\lambda_0 = 0$ (ii) using the exact \texttt{E-NPG} update with $\eta = \frac{1 - \gamma}{\tau}$, $K \geq 1$, $J=\gO(\ln(\nicefrac{1}{\eps}))$ and (iii) warm-starting the initial policy such that $L_0 \leq \gO(\nicefrac{\beta\mu}{\tau})^{2^K})$ 
return an $\gO(\eps)$-optimal policy for~\cref{problem:cmdp} after $\widetilde{\gO}(\nicefrac{1}{\eps^{2 + 3/2^K}})$ gradient evaluations.
\end{restatable}
\noindent When $K \geq 1$, the initial condition on $L_0$ is satisfied by running $\gO(\ln(\nicefrac{1}{\eps}))$ \texttt{E-NPG} iterations on $R(\lambda_0)$ since \texttt{E-NPG} converges linearly (\cref{lemma:inexact_npg}). 
Let us now consider when $K=1$. In this case, the algorithm attains an $\tilde{\gO}(\nicefrac{1}{\eps^{3.5}})$ last-iterate rate, improving the result of~\cref{theorem:l2_last_iterate_K_1} by exploiting a more accurate initialization. To the best of our knowledge, this method results achieves the fastest last-iterate convergence for \texttt{PD} methods in the tabular setting and does not require any coverage assumptions (refer to~\cref{table:last}). We further experimentally evaluate our method (see~\cref{fig:compare}).
More generally, the rate improves exponentially with $K$ and as $K \to \infty$, the method attains an $\tilde{\gO}(\nicefrac{1}{\eps^2})$ convergence, matching~\cref{theorem:reg_smoothness_l2_rate}.
Hence, combining these two cases, this analysis interpolates between the \texttt{PD} and \texttt{DD} regimes.
\begin{figure}[H]
    \centering
    \includegraphics[scale=1.1]{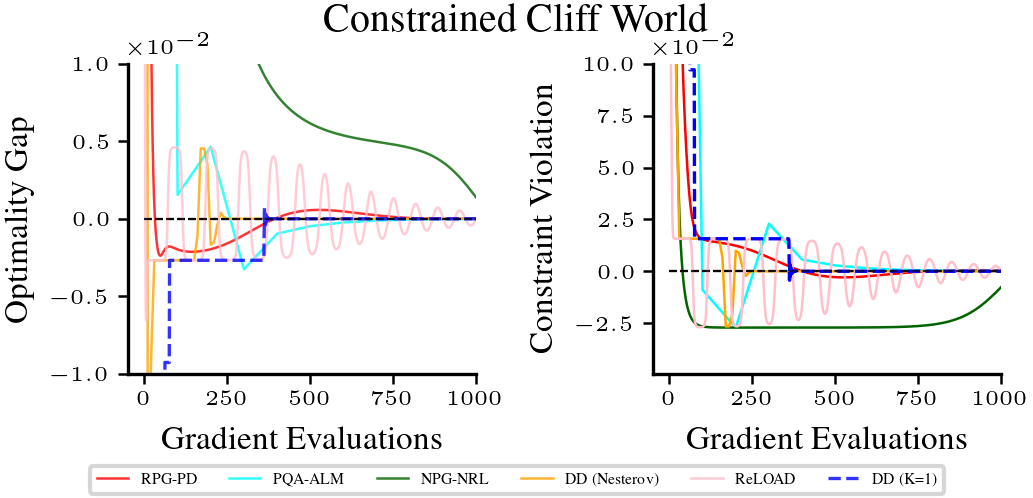}
    \caption{Comparing our regularized inexact \texttt{DD} method $\texttt{DD (K=1)}$ when $K=1$, $\mu > 0$ and warm-starting the initial policy with $50$ iterations. 
    We compare against the following practical last-iterate methods from prior work: \texttt{RPG-PD}~\citep{ding2023last}, \texttt{PQA-ALM}~\citep{lu2026augmented}, \texttt{NPG-NRL}~\citep{muller2024truly}, \texttt{ReLoad}~\citep{moskovitz2023reload}, and the related DD method of~\citet{ying2022dual} (\texttt{DD (Nesterov)}) 
    For all methods, the hyper-parameters were selected via a grid-search (see~\cref{appendix:experiments_mdp} for details)
    We select the best variant such that the final policy satisfies $b - V^{\pi}_c(\rho) \leq 0.001$.   
    Our method converges towards zero optimality gap and constraint violation the fastest without any oscillations.
    }
    \label{fig:compare}
\end{figure}

\section{Conclusion}
We provide a new analysis of \texttt{DD} methods that establishes last-iterate guarantees without requiring impractical coverage assumptions. 
Building on this, we further introduce a truncated bilevel algorithm that interpolates between \texttt{PD} and \texttt{DD} methods, yielding an practical algorithm with theoretical guarantees.
In the tabular setting, this approach achieves a $\widetilde{\gO}(\nicefrac{1}{\eps^{3.5}})$ last-iterate rate and extends to the linear function approximation setting through a black-box reduction.
Our proposed \texttt{PD} method attains the best known last-iterate rates without relying on any coverage assumptions.
For future work, we plan to validate our method on large-scale control tasks and extend the theoretical analysis to consider the general function approximation setting.

\bibliographystyle{plainnat}
\bibliography{ref}

@book{altman2021constrained,
    title = {Constrained Markov decision processes},
    author = {Altman, Eitan},
    year = {2021},
    publisher = {Routledge},
}

@article{ding2020natural,
    title = {Natural policy gradient primal-dual method for constrained markov
             decision processes},
    author = {Ding, Dongsheng and Zhang, Kaiqing and Basar, Tamer and Jovanovic,
              Mihailo},
    journal = {Advances in Neural Information Processing Systems},
    volume = {33},
    pages = {8378--8390},
    year = {2020},
}

@inproceedings{moskovitz2023reload,
    title = {Reload: Reinforcement learning with optimistic ascent-descent for
             last-iterate convergence in constrained mdps},
    author = {Moskovitz, Ted and O’Donoghue, Brendan and Veeriah, Vivek and
              Flennerhag, Sebastian and Singh, Satinder and Zahavy, Tom},
    booktitle = {International Conference on Machine Learning},
    pages = {25303--25336},
    year = {2023},
    organization = {PMLR},
}

@article{ding2023last,
    title = {Last-iterate convergent policy gradient primal-dual methods for
             constrained mdps},
    author = {Ding, Dongsheng and Wei, Chen-Yu and Zhang, Kaiqing and Ribeiro,
              Alejandro},
    journal = {Advances in Neural Information Processing Systems},
    volume = {36},
    pages = {66138--66200},
    year = {2023},
}

@article{agarwal2021theory,
    title = {On the theory of policy gradient methods: Optimality, approximation
             , and distribution shift},
    author = {Agarwal, Alekh and Kakade, Sham M and Lee, Jason D and Mahajan,
              Gaurav},
    journal = {Journal of Machine Learning Research},
    volume = {22},
    number = {98},
    pages = {1--76},
    year = {2021},
}

@inproceedings{mei2020global,
    title = {On the global convergence rates of softmax policy gradient methods},
    author = {Mei, Jincheng and Xiao, Chenjun and Szepesvari, Csaba and
              Schuurmans, Dale},
    booktitle = {International conference on machine learning},
    pages = {6820--6829},
    year = {2020},
    organization = {PMLR},
}

@article{paternain2019constrained,
    title = {Constrained reinforcement learning has zero duality gap},
    author = {Paternain, Santiago and Chamon, Luiz and Calvo-Fullana, Miguel and
              Ribeiro, Alejandro},
    journal = {Advances in Neural Information Processing Systems},
    volume = {32},
    year = {2019},
}

@book{sutton2018reinforcement,
    title = {Reinforcement learning: An introduction},
    author = {Sutton, Richard S and Barto, Andrew G},
    year = {2018},
    publisher = {MIT press},
}

@article{zhang2021convergence,
    title = {On the convergence and sample efficiency of variance-reduced policy
             gradient method},
    author = {Zhang, Junyu and Ni, Chengzhuo and Szepesvari, Csaba and Wang,
              Mengdi and others},
    journal = {Advances in Neural Information Processing Systems},
    volume = {34},
    pages = {2228--2240},
    year = {2021},
}

@inproceedings{jain2022towards,
    title = {Towards painless policy optimization for constrained mdps},
    author = {Jain, Arushi and Vaswani, Sharan and Babanezhad, Reza and
              Szepesvari, Csaba and Precup, Doina},
    booktitle = {Uncertainty in Artificial Intelligence},
    pages = {895--905},
    year = {2022},
    organization = {PMLR},
}

@misc{yuan2022general,
      title={A general sample complexity analysis of vanilla policy gradient}, 
      author={Rui Yuan and Robert M. Gower and Alessandro Lazaric},
      year={2022},
      eprint={2107.11433},
      archivePrefix={arXiv},
      primaryClass={cs.LG}
}

@article{lan2023policy,
  title={Policy mirror descent for reinforcement learning: Linear convergence, new sampling complexity, and generalized problem classes},
  author={Lan, Guanghui},
  journal={Mathematical programming},
  volume={198},
  number={1},
  pages={1059--1106},
  year={2023},
  publisher={Springer}
}

@article{yuan2022linear,
  title={Linear convergence of natural policy gradient methods with log-linear policies},
  author={Yuan, Rui and Du, Simon S and Gower, Robert M and Lazaric, Alessandro and Xiao, Lin},
  journal={arXiv preprint arXiv:2210.01400},
  year={2022}
}

@article{liu2025sample,
  title={Sample Complexity Bounds for Linear Constrained MDPs with a Generative Model},
  author={Liu, Xingtu and Yang, Lin F and Vaswani, Sharan},
  journal={arXiv preprint arXiv:2507.02089},
  year={2025}
}

@article{kakade2001natural,
  title={A natural policy gradient},
  author={Kakade, Sham M},
  journal={Advances in neural information processing systems},
  volume={14},
  year={2001}
}

@article{kiefer1960equivalence,
  title={The equivalence of two extremum problems},
  author={Kiefer, Jack and Wolfowitz, Jacob},
  journal={Canadian Journal of Mathematics},
  volume={12},
  pages={363--366},
  year={1960},
  publisher={Cambridge University Press}
}

@article{
mondal2026lastiterate,
title={Last-Iterate Convergence of General Parameterized Policies in Constrained {MDP}s},
author={Washim Uddin Mondal and Vaneet Aggarwal},
journal={Transactions on Machine Learning Research},
issn={2835-8856},
year={2026},
url={https://openreview.net/forum?id=JedrMCZC6l},
note={}
}

@InProceedings{lin2025rethinking,
  title = 	 { Optimistic Actor-Critic with Parametric Policies for Linear Markov Decision Processes },
  author =       {Lin, Max Qiushi and Asad, Reza and Tan, Kevin and Ishfaq, Haque and Szepesvari, Csaba and Vaswani, Sharan},
  booktitle = 	 {Proceedings of The 29th International Conference on Artificial Intelligence and Statistics},
  pages = 	 {2575--2583},
  year = 	 {2026},
  editor = 	 {Khan, Emtiyaz and Li, Yingzhen and Solin, Arno and Ramdas, Aaditya},
  volume = 	 {300},
  series = 	 {Proceedings of Machine Learning Research},
  month = 	 {02--05 May},
  publisher =    {PMLR},
  url = 	 {https://proceedings.mlr.press/v300/lin26d.html}
}

@article{liu2025convergence,
  title={On the convergence of projected policy gradient for any constant step sizes},
  author={Liu, Jiacai and Li, Wenye and Lin, Dachao and Wei, Ke and Zhang, Zhihua},
  journal={Journal of Machine Learning Research},
  volume={26},
  number={193},
  pages={1--35},
  year={2025}
}

@inproceedings{julian2002qos,
  title={QoS and fairness constrained convex optimization of resource allocation for wireless cellular and ad hoc networks},
  author={Julian, David and Chiang, Mung and O'Neill, Daniel and Boyd, Stephen},
  booktitle={Proceedings. twenty-first annual joint conference of the IEEE computer and communications societies},
  volume={2},
  pages={477--486},
  year={2002},
  organization={IEEE}
}

@article{buratti2009overview,
  title={An overview on wireless sensor networks technology and evolution},
  author={Buratti, Chiara and Conti, Andrea and Dardari, Davide and Verdone, Roberto},
  journal={Sensors},
  volume={9},
  number={9},
  pages={6869--6896},
  year={2009},
  publisher={Molecular Diversity Preservation International (MDPI)}
}

@article{muller2024truly,
  title={Truly no-regret learning in constrained mdps},
  author={M{\"u}ller, Adrian and Alatur, Pragnya and Cevher, Volkan and Ramponi, Giorgia and He, Niao},
  journal={arXiv preprint arXiv:2402.15776},
  year={2024}
}

@inproceedings{ying2022dual,
  title={A dual approach to constrained markov decision processes with entropy regularization},
  author={Ying, Donghao and Ding, Yuhao and Lavaei, Javad},
  booktitle={International Conference on Artificial Intelligence and Statistics},
  pages={1887--1909},
  year={2022},
  organization={PMLR}
}

@article{zhan2023policy,
  title={Policy mirror descent for regularized reinforcement learning: A generalized framework with linear convergence},
  author={Zhan, Wenhao and Cen, Shicong and Huang, Baihe and Chen, Yuxin and Lee, Jason D and Chi, Yuejie},
  journal={SIAM Journal on Optimization},
  volume={33},
  number={2},
  pages={1061--1091},
  year={2023},
  publisher={SIAM}
}

@article{osband2019behaviour,
  title={Behaviour suite for reinforcement learning},
  author={Osband, Ian and Doron, Yotam and Hessel, Matteo and Aslanides, John and Sezener, Eren and Saraiva, Andre and McKinney, Katrina and Lattimore, Tor and Szepesvari, Csaba and Singh, Satinder and others},
  journal={arXiv preprint arXiv:1908.03568},
  year={2019}
}

@article{kober2013reinforcement,
  title={Reinforcement learning in robotics: A survey},
  author={Kober, Jens and Bagnell, J Andrew and Peters, Jan},
  journal={The International Journal of Robotics Research},
  volume={32},
  number={11},
  pages={1238--1274},
  year={2013},
  publisher={SAGE Publications Sage UK: London, England}
}

@article{uc2023survey,
  title={Survey on reinforcement learning for language processing},
  author={Uc-Cetina, Victor and Navarro-Guerrero, Nicol{\'a}s and Martin-Gonzalez, Anabel and Weber, Cornelius and Wermter, Stefan},
  journal={Artificial Intelligence Review},
  volume={56},
  number={2},
  pages={1543--1575},
  year={2023},
  publisher={Springer}
}

@article{mnih2015human,
  title={Human-level control through deep reinforcement learning},
  author={Mnih, Volodymyr and Kavukcuoglu, Koray and Silver, David and Rusu, Andrei A and Veness, Joel and Bellemare, Marc G and Graves, Alex and Riedmiller, Martin and Fidjeland, Andreas K and Ostrovski, Georg and others},
  journal={nature},
  volume={518},
  number={7540},
  pages={529--533},
  year={2015},
  publisher={Nature Publishing Group}
}

@article{bura2022dope,
  title={DOPE: Doubly optimistic and pessimistic exploration for safe reinforcement learning},
  author={Bura, Archana and HasanzadeZonuzy, Aria and Kalathil, Dileep and Shakkottai, Srinivas and Chamberland, Jean-Francois},
  journal={Advances in neural information processing systems},
  volume={35},
  pages={1047--1059},
  year={2022}
}

@article{montenegro2024last,
  title={Last-iterate global convergence of policy gradients for constrained reinforcement learning},
  author={Montenegro, Alessandro and Mussi, Marco and Papini, Matteo and Metelli, Alberto Maria},
  journal={Advances in Neural Information Processing Systems},
  volume={37},
  pages={126363--126416},
  year={2024}
}

@article{cen2022fast,
  title={Fast global convergence of natural policy gradient methods with entropy regularization},
  author={Cen, Shicong and Cheng, Chen and Chen, Yuxin and Wei, Yuting and Chi, Yuejie},
  journal={Operations Research},
  volume={70},
  number={4},
  pages={2563--2578},
  year={2022},
  publisher={INFORMS}
}

@article{nachum2017bridging,
  title={Bridging the gap between value and policy based reinforcement learning},
  author={Nachum, Ofir and Norouzi, Mohammad and Xu, Kelvin and Schuurmans, Dale},
  journal={Advances in neural information processing systems},
  volume={30},
  year={2017}
}

@book{boyd2004convex,
  title={Convex optimization},
  author={Boyd, Stephen and Vandenberghe, Lieven},
  year={2004},
  publisher={Cambridge university press}
}

@article{paternain2022safe,
  title={Safe policies for reinforcement learning via primal-dual methods},
  author={Paternain, Santiago and Calvo-Fullana, Miguel and Chamon, Luiz FO and Ribeiro, Alejandro},
  journal={IEEE Transactions on Automatic Control},
  volume={68},
  number={3},
  pages={1321--1336},
  year={2022},
  publisher={IEEE}
}

@article{li2024faster,
  title={Faster algorithm and sharper analysis for constrained Markov decision process},
  author={Li, Tianjiao and Guan, Ziwei and Zou, Shaofeng and Xu, Tengyu and Liang, Yingbin and Lan, Guanghui},
  journal={Operations Research Letters},
  volume={54},
  pages={107107},
  year={2024},
  publisher={Elsevier}
}

@inproceedings{gladin2023algorithm,
  title={Algorithm for constrained markov decision process with linear convergence},
  author={Gladin, Egor and Lavrik-Karmazin, Maksim and Zainullina, Karina and Rudenko, Varvara and Gasnikov, Alexander and Takac, Martin},
  booktitle={International Conference on Artificial Intelligence and Statistics},
  pages={11506--11533},
  year={2023},
  organization={PMLR}
}

@article{ho2016generative,
  title={Generative adversarial imitation learning},
  author={Ho, Jonathan and Ermon, Stefano},
  journal={Advances in neural information processing systems},
  volume={29},
  year={2016}
}

@inproceedings{bai2022achieving,
  title={Achieving zero constraint violation for constrained reinforcement learning via primal-dual approach},
  author={Bai, Qinbo and Bedi, Amrit Singh and Agarwal, Mridul and Koppel, Alec and Aggarwal, Vaneet},
  booktitle={Proceedings of the AAAI Conference on Artificial Intelligence},
  volume={36},
  number={4},
  pages={3682--3689},
  year={2022}
}

@article{vaswani2022near,
  title={Near-optimal sample complexity bounds for constrained mdps},
  author={Vaswani, Sharan and Yang, Lin and Szepesv{\'a}ri, Csaba},
  journal={Advances in Neural Information Processing Systems},
  volume={35},
  pages={3110--3122},
  year={2022}
}

@inproceedings{geist2019theory,
  title={A theory of regularized markov decision processes},
  author={Geist, Matthieu and Scherrer, Bruno and Pietquin, Olivier},
  booktitle={International conference on machine learning},
  pages={2160--2169},
  year={2019},
  organization={PMLR}
}

@article{liu2024elementary,
  title={Elementary analysis of policy gradient methods},
  author={Liu, Jiacai and Li, Wenye and Wei, Ke},
  journal={arXiv preprint arXiv:2404.03372},
  year={2024}
}

@book{floudas1995nonlinear,
  title={Nonlinear and mixed-integer optimization: fundamentals and applications},
  author={Floudas, Christodoulos A},
  year={1995},
  publisher={Oxford University Press}
}

@book{nesterov2013introductory,
  title={Introductory lectures on convex optimization: A basic course},
  volume={87},
  author={Nesterov, Yurii},
  year={2013},
  publisher={Springer Science \& Business Media}
}

@article{lu2026augmented,
  title={Augmented Lagrangian Method for Last-Iterate Convergence for Constrained MDPs},
  author={Lu, Michael and Lin, Max Qiushi and Chen, Mo and Vaswani, Sharan},
  journal={arXiv preprint arXiv:2605.11694},
  year={2026}
}

@article{frank1956algorithm,
  title={An algorithm for quadratic programming},
  author={Frank, Marguerite and Wolfe, Philip},
  journal={Naval research logistics quarterly},
  volume={3},
  number={1-2},
  pages={95--110},
  year={1956},
  publisher={Wiley Online Library}
}

@InProceedings{tingting2025safe,
  title = 	 {A Safe Exploration Approach to Constrained Markov Decision Processes},
  author =       {Ni, Tingting and Kamgarpour, Maryam},
  booktitle = 	 {Proceedings of The 28th International Conference on Artificial Intelligence and Statistics},
  pages = 	 {3592--3600},
  year = 	 {2025},
  editor = 	 {Li, Yingzhen and Mandt, Stephan and Agrawal, Shipra and Khan, Emtiyaz},
  volume = 	 {258},
  series = 	 {Proceedings of Machine Learning Research},
  month = 	 {03--05 May},
  publisher =    {PMLR},
  url = 	 {https://proceedings.mlr.press/v258/ni25a.html}
}

\newpage
\appendix
\crefalias{section}{appendix}
\crefalias{subsection}{appendix}
\newcommand{\appendixTitle}{%
\vbox{
    \centering
	\hrule height 4pt
	\vskip 0.2in
	{\LARGE \bf Supplementary Material}
	\vskip 0.2in
	\hrule height 1pt 
}}
\appendixTitle
\startcontents        %
\printcontents{}{1}{} %
\pagebreak 
\section{Handling Linear Function Approximation}\label{appendix:handling_linear_fa}
We first show that the inexact \texttt{E-NPG} analysis extends to function approximation by treating the soft state-action value-function estimation error as an additional bias term. We then bound this error for log-linear policies and derive the corresponding last-iterate convergence rate.

To generalize across large state-action spaces, we consider the log-linear policy~\citep{agarwal2021theory,yuan2022linear}.
Given policy features $\varphi: \cS \times \cA \to \R^d$ and parameters $\theta \in \R^d$ with $d \ll SA$, a log-linear policy is expressed as: 
\begin{equation}\label{eq:log_linear_policy}
\pi(a | s) = \frac{\exp(\dpd{\varphi(s, a), \theta})}{\sum_{a'} \exp(\dpd{\varphi(s, a'), \theta})} \, , \quad \forall (s, a) \in \cS \times \cA
\end{equation}
where the parameterization is implicit in the notation.
For this policy class, we employ \texttt{Q-E-NPG}~\citep{agarwal2021theory}, a counterpart of \texttt{E-NPG} when employing linear function approximation (FA).
In particular, \texttt{Q-E-NPG} projects the soft Q-function onto the span of the policy features then updates the policy in the direction of this projection.

To start, we first define $\Pr^{\pi}[s_h = s, a_h = a | s_0]$ as the probability of visiting state $s$ and taking action $a$ at timestep $h$ when starting at the initial state $s_0$ and the define state-action occupancy measure $\nu^{\pi} \coloneqq 1 - \gamma) \, \sum_{s_0} \rho(s_0) \, \sum_{h=0}^\infty \gamma^h \Pr^{\pi}[s_h = s, a_h = a | s_0]$.
Furthermore, for a fixed out-iteration $t \geq 0$ and dual variable $\lambda_t \in \R^m_{+}$, we consider the pseudo-reward $R_t \coloneqq R(\lambda_t) = r + \sum_{i=1}^m c_i \, [\lambda_t]_i$. For each inner-iteration $k \in [0, K-1]$
\texttt{Q-E-NPG} iteratively considers the approximate solution to the following weighted least-squares problem:
\begin{equation}\label{eq:bad_least_sqaures}
\bar{w}_k \approx\argmin_{w \in \R^d} \sum_{(s, a) \in \cS \times \cA}\left[\nu^{\pi_k}(s, a) \, ( Q^{\pi_k}_{R_t, \tau}(s, a) - \dpd{\varphi(s, a), w} )^2\right]    
\end{equation}
then updates the parameters in the direction of $\dpd{\varphi,\bar{w}_k}$ using the step-size $\eta$,
\begin{equation*}
    \theta_{k+1} = \theta_k + \eta \, \bar{w}_k
\end{equation*} 
Here, $\dpd{\varphi,\bar{w}_k} \in \R^{SA}$ is a linear approximation of the soft state-action value-function $Q^{\pi_k}_{R_t, \tau} \in \R^{SA}$.
Hence for the log-linear policy parameterization, the above update is equivalent to the inexact \texttt{E-NPG} update in~\cref{eq:inexact_entropy_npg_update}, with $\diamond = R_t$ and  $\hat Q^{\pi_k}_{\diamond,\tau} = \langle\varphi, \bar{w}_k\rangle$.
Thus, the tracking analysis of the inexact dual gradients (\cref{lemma:l_tracking_main,lemma:phi_tracking_main}) applies with an additional error term arising from function approximation.
We characterize how this error propagates through the accumulated inexact gradient bias and defer the technical details to~\cref{appendix:tools_k1,appendix:tools_kgeq1}.
Therefore, the remaining challenge is to control the approximation error $\norm{Q^{\pi_k}_{R_t,\tau}-\dpd{\varphi,\bar{w}_k}}_\infty$.

Standard analyses of typically characterize this error under the compatible function approximation (CFA) framework~\citep{kakade2001natural,agarwal2021theory,ding2023last}. 
In particular, the resulting guarantee depend on the transfer bias~\citep{agarwal2021theory}, which quantifies how well a function learned under the current policy distribution generalizes to the state-action distribution of the optimal policy.
Controlling this quantity requires sufficient coverage of the optimal policy's state-action distribution and can be arbitrarily small when $\nu^{\pi_k}$ has poor coverage~\citep{agarwal2021theory,yuan2022linear}. 
Hence, to avoid introducing any policy-dependent coverage assumptions, we instead estimate the linear approximation of $Q^{\pi_k}_{R_t, \tau}$ under a fixed state-action distribution chosen independently of the current policy.

In order to do so, we select a (small) subset of the entire state-action space that has provides sufficient coverage based on the policies features.
To this end, let $\cC \subseteq \cS \times \cA$ denote a coreset and $\varsigma \in \Delta(\cC)$ denote its corresponding state-action distribution. 
Using $\varsigma$, we now consider the following \texttt{Q-E-NPG} update with $\eta = \frac{1 - \gamma}{\tau}$:
\begin{align*}
    w_k \approx \argmin_{w \in \R^d} \ell^k_{t, \tau}(w, \varsigma) \, 
    \text{ ; } \ell^k_{t, \tau}(w) &= \sum_{(s, a) \in \cC}[\varsigma(s, a) [Q^{\pi_k}_{R_t, \tau}(s,a) - \dpd{w, \varphi(s, a)}]^2] 
   \numberthis \label{eq:general_regression_loss}  \\
    \pi_{k+1}(\cdot | s) &\propto \exp\left(\tfrac{1}{\tau} \, \dpd{\varphi(s, \cdot), w_k}\right) \numberthis \label{eq:q_npg_update} \,.
\end{align*}
In contrast to~\cref{eq:bad_least_sqaures}, the least-squares problem in~\cref{eq:general_regression_loss} is no longer minimization over the entire state-action space, resulting in a more practical update.
Furthermore, following prior work~\citep{jain2022towards,lin2025rethinking,ding2023last},
we make the following assumptions to control the approximation error.
\begin{assumption}[Bias]\label{assumption:q_realizability}
$\sup_{t, k} \, \min_{w \in \mathbb{R}^d}  
\max_{(s, a)} \abs{Q^{\pi_k}_{R_t, \tau}(s, a) - \dpd{w, \varphi(s, a)}}
\leq \eps_{\bias}$.
\end{assumption}
\begin{assumption}[Optimization Error]\label{assumption:eps_opt}
 $\sup_{t, k} \abs*{\ell^k_{t, \tau}(w_k, \varsigma) - \min_{w \in \mathbb{R}^d} \ell^k_{t, \tau}(w, \varsigma)} \leq \eps_{\opt}$.
\end{assumption}
\noindent \cref{assumption:q_realizability} depends on the expressivity of the policy feature $\varphi$. 
For instance, in the case when $\varphi$ are ``one-hot'' vectors of dimension $d=SA$, $\eps_{\bias} = 0$. 
Furthermore, for sufficiently expressive features, $\eps_{\bias}$ can be small.
\cref{assumption:eps_opt} accounts for the optimization error when minimizing the least-squares loss of~\cref{eq:general_regression_loss}. 
In general, this procedure admits linear convergence, and hence, $\eps_{\opt}$ can be made arbitrary small.
Under these assumptions, the approximation error admits the following bound.
\begin{restatable}{lemma}{coresetlemma}\label{lemma:coreset}
Under~\cref{assumption:slater_condition,assumption:q_realizability,assumption:eps_opt}, for a fixed outer-iteration $t \geq 0$ and entropy parameter $\tau > 0$, for all $k \in [0, K-1]$:
\begin{align*}
    \supnorm{Q^{\pi_k}_{R_t, \tau} - \dpd{\varphi, w_k}} \leq \eps_{\bias} \parens*{1 + \sup_{(s, a) \in \cS \times \cA} \norm{\varphi(s, a)}_{G^{\dagger}_\varsigma}} + \sqrt{\eps_{\opt}}\,,
\end{align*}
where $G_\varsigma \coloneqq \sum_{(s, a) \in \gC} \varsigma(s, a) \, \varphi(s, a) \, \varphi(s, a)^\top$ and $G_\varsigma^\dagger$ is the pseudo-inverse of $G_\varsigma$.
\end{restatable}
By the Kiefer-Wolfowitz theorem~\citep{kiefer1960equivalence}, there always exists an coreset $\cC$ with $\abs{\gC} \leq \frac{d (d+1)}{2}$ and a corresponding distribution $\varsigma \in \Delta(\cC)$ such that $\sup_{(s, a)} \norm{\varphi(s, a)}_{G^{\dagger}_\varsigma} \leq \sqrt{d}$.  
Hence, by~\cref{lemma:coreset} this implies the approximation error $\supnorm{Q^{\pi_k}_{R_t, \tau} - \dpd{\varphi, w_k}}$ can be bounded by $\gO(\sqrt{d} \, \eps_{\bias} + \sqrt{\eps_{\opt}})$.
For simplicity, we will assume such a coreset is given but note its constructing can be viewed as a standard experimental design problem in statistics (e.g., G-optimal design via the Frank-Wolfe algorithm~\citep{frank1956algorithm}).
Importantly, this construction relies only on the policy features $\varphi$. Hence, despite the changing pseudo-rewards, the coreset can be precomputed and reused across all outer-iterations.

Putting everything together, we attain the following last-iterate rate when extending~\cref{theorem:l2_last_iterate_K_geq_1} to the linear FA setting. 
\begin{restatable}{corollary}{smoothscddlastitkgeqoneratecorollary} \label{corollary:smooth_l2_last_iterate_K_geq_1_fa}
Under~\cref{assumption:slater_condition,assumption:q_realizability,assumption:eps_opt}, for a given $\eps \in (0, 1)$ assuming $\eps_{\bias} \leq \gO(\eps^4)$ and $\eps_{\opt} \leq \gO(\eps^8)$,~\cref{algorithm:dd} with (i) $\tau = \mu = \gO(\eps)$, $T = \tilde{\gO}(\nicefrac{1}{\eps^{2 + 3/2^K}})$, $\beta = \gO\parens*{\eps^\frac{3 + 2^K}{2^K}}$, $\lambda_0 = 0$ (ii) using the \texttt{Q-E-NPG} update with $\eta = \frac{1 - \gamma}{\tau}$, $K \geq 1$ and
(iii) the initial policy satisfies $L_{0} \leq B$ for a known constant $B > 0$,~\cref{algorithm:dd} returns an $\gO(\eps)$-policy to~\cref{problem:cmdp} after
$\widetilde{\gO}(\nicefrac{1}{\eps^{2 + 3/2^K}})$ gradient evaluations.
\end{restatable}
\paragraph{Comparison to prior work.} 
Under the compatible function approximation~\citep{agarwal2021theory} (CFA) framework,~\citet{ding2023last}, propose a related \texttt{PD} method and attain a $\gO(\nicefrac{1}{\eps^6})$ last-iterate rate. However, their analysis is limited to a single constraint and requires projecting onto a restricted simplex. 
More recently, under the same CFA framework,~\citet{montenegro2024last} establish a similar $\gO(\nicefrac{1}{\eps^6})$ last-iterate rate and is recently improved to $\gO(\nicefrac{1}{\eps^4})$ by~\citet{mondal2026lastiterate}. 
To achieve these faster rates, both methods additionally imposing a Fisher non-degeneracy assumption~\citep{zhan2023policy}.  
This assumption, however, fails to hold for log-linear policies~\citep[Proposition C.4]{tingting2025safe}.
Furthermore, under only a coverage assumption,~\citet{lu2026augmented} attains an $\gO(\nicefrac{1}{\eps^6})$ rate but requires $\eps_{\bias} \leq \gO(\eps^{16})$. 
In contrast to all prior results (refer to~\cref{table:last}), with $K=1$, we achieve superior $\gO(\nicefrac{1}{\eps^{3.5}})$ rate without any converge assumptions and a tighter dependence of $\eps_{\bias}$.

\section{Summary of Results}\label{appendix:results}
\newcommand{\ratecell}[2]{%
  \begin{tabular}[c]{@{}c@{}}#1\\[-1pt]\scriptsize{[\cref{#2}]}\end{tabular}
}

\begin{table}[H]
\caption{Total number of gradient evaluations required to obtain an $\gO(\eps)$-optimal policy to~\cref{problem:cmdp}. The columns correspond to different numbers of \texttt{NPG} steps $K$ used for each dual update. The dagger ($\dagger$) indicates result rate requires an additional $\gO(\ln(\nicefrac{1}{\eps}))$ gradient evaluations.}
\centering
\renewcommand{\arraystretch}{1.35}
\setlength{\tabcolsep}{7pt}
\begin{tabular}{@{}l c c c@{}}
\toprule
& \multicolumn{3}{c}{\textbf{Number of Inner-Iterations Using NPG }} \\
\cmidrule(l){2-4}
\textbf{Dual Objective}
& $K=\gO(\ln(\nicefrac{1}{\eps}))$
& $K= 1$
& $K\geq 1^\dagger$ \\
\midrule
Smooth
& \ratecell{$\tilde{\gO}(\frac{1}{\eps^4})$}{theorem:meta_smooth}
& \ratecell{$\tilde{\gO}(\frac{1}{\eps^7})$}{theorem:smooth_last_iterate_K_1}
& \ratecell{$\tilde{\gO}\parens*{\frac{1}{\eps^{4 + \frac{3}{2^K}}}}$}{theorem:smooth_last_iterate_K_geq_1} \\
\addlinespace
Smooth + $\ell_2$-regularized
& \ratecell{$\tilde{\gO}(\frac{1}{\eps^2})$}{theorem:reg_smoothness_l2_rate}
& \ratecell{$\tilde{\gO}(\frac{1}{\eps^5})$}{theorem:l2_last_iterate_K_1}
& \ratecell{$\tilde{\gO}\parens*{\frac{1}{\eps^{2 + \frac{3}{2^K}}}}$}{theorem:l2_last_iterate_K_geq_1}
\\
\bottomrule
\end{tabular}
\label{table:rates_summary}
\end{table}

\section{Framework}\label{appendix:inexact_dd}
\subsection{Proof of~\cref{lemma:dual_optimality_to_sub}}
\dualoptimalitytosublemma*
\begin{proof}
Let $\pistar$ denote the optimal policy of~\cref{problem:cmdp}. 
To characterize the reward sub-optimality of $\pistar$ and $\pistarbl \coloneqq \argmax_\pi V^{\pi}_{R(\bar{\lambda}), \tau}(\rho)$, we first 
relate the performance  between $\pistar$ and $\pi^*_\tau$:
\begin{align*}
    V^{\pistar}_r(\rho) &\leq V^{\pistar}_{r, \tau}(\rho)  \tag{Entropy is non-negative} \\
    &\leq V^{\pi^*_\tau}_{r, \tau}(\rho)  \tag{$\pi^*_\tau$ is the optimal policy of~\cref{problem:soft_max_min_lag}} \\
    &= V^{\pi^*_\tau}_{r}(\rho) + \tau \, \gH(\rho, \pi^*_\tau) \tag{Definition of $V^{\pi}_{r, \tau}(\rho)$}\\
    &\leq V^{\pi^*_\tau}_{r}(\rho) + \frac{\ln(A) \, \tau}{1 - \gamma} \tag{Since $\gH(\rho, \pi) \leq \frac{\ln(A)}{1 - \gamma}$}\\
\intertext{Subtracting $V^{\pistarbl}_r(\rho)$ from both sides,}
\implies V^{\pistar}_{r}(\rho) - V^{\pistarbl}_r(\rho) &\leq V^{\pi^*_\tau}_{r}(\rho) - V^{\pistarbl}_r(\rho) +  \frac{\ln(A) \, \tau}{1 - \gamma} \numberthis \label{eq:sub_opt_lag_bound_0}
\end{align*}
We now bound the sub-optimality gap.
Since $r \in [0, 1]^{SA}$, using~\cref{eq:vdl_upper_bound_lag} from~\cref{lemma:generic_vdl_upper_bound} with $\diamond = r$, $U=1$, $\pi' = \pistarbl$, $\pi = \pi^*_\tau$ and $\lambda = \bar{\lambda}$,
\begin{equation}
   \abs*{V^{\pi^*_\tau}_{r}(\rho) - V^{\pistarbl}_r(\rho)}  \leq 
  \frac{\sqrt{2 }}{(1 - \gamma)^{\nicefrac{3}{2}}} \, \frac{\sqrt{\abs{\lag(\pi^*_\tau, \blambda) - \lag(\pistarbl, \blambda)}}}{\sqrt{\tau}} \,  \label{eq:delta_reward_bound}. 
\end{equation}
Next, we relate $\abs{\lag(\pistarbl, \blambda) - \lag(\pi^*_\tau, \blambda)}$ in terms of the dual sub-optimality gap. 
By the saddle-point property, $\lstar$ minimizes the Lagrangian at $\pi^*_\tau$, and hence,
\begin{align*}
   \cL_\tau(\pi^*_\tau, \lstar) = \min_{\lambda' \in \R^m_+} \cL_\tau(\pi^*_\tau, \lambda') \leq \cL_\tau(\pi^*_\tau, \blambda)  \,.
\end{align*}
Rearranging and adding $\lag(\pistarbl, \blambda)$ to both sides,
\begin{align*}
    \implies \lag(\pistarbl, \blambda) - \lag(\pi^*_\tau, \blambda) &\leq \lag(\pistarbl, \blambda) - \lag(\pi^*_\tau, \lstar) \, .
\end{align*}
Furthermore under~\cref{assumption:slater_condition}, by~\cref{lemma:strong_duality}, strong duality holds for~\cref{problem:soft_max_min_lag}, i.e.,
\begin{align*}
    \cL_\tau(\pi^*_\tau, \lstar) &= \cD(\lstar)
\end{align*}
Moreover, since $\pistarbl = \argmax_{\pi} \lag(\pi, \blambda)$, we have 
\begin{equation*}
 \lag(\pistarbl, \blambda) = \cD(\blambda)   
\end{equation*}
Substituting these relations into the previous inequality and assuming $\cD(\blambda) - \cD(\lstar) \leq \Delta$, for a $\Delta > 0$
\begin{align*}
    \implies \lag(\pistarbl, \blambda) - \lag(\pi^*_\tau, \blambda)  \leq \cD(\blambda) - \cD(\lstar) \leq \Delta \numberthis \label{eq:sub_opt_lag_bound_1}
\end{align*}
Putting~\cref{eq:delta_reward_bound,eq:sub_opt_lag_bound_0,eq:sub_opt_lag_bound_1} together,
\begin{align*}
   \implies V^{\pistar}_{r}(\rho) - V^{\pistarbl}_r(\rho)  \leq \gO\parens*{\frac{\sqrt{\Delta}}{\sqrt{\tau}}} + \gO(\tau) \, .
\end{align*}
We now bound the constraint violation.  For any fixed $i \in [m]$ we have
\begin{align*}
    \max\braces*{b_i - V^{\pistarbl}_{c_i}(\rho), 0} &\leq \max\braces*{V^{\pi^*_\tau}_{c_i}(\rho) - V^{\pistarbl}_{c_i}(\rho), 0} \tag{Since $\pi^*_\tau$ is the optimal policy of~\cref{problem:soft_max_min_lag}, $V^{\pi^*_\tau}_{c_i}(\rho) \geq b_i$} \\
    &\leq \abs*{V^{\pi^*_\tau}_{c_i}(\rho) - V^{\pistarbl}_{c_i}(\rho)} \numberthis \label{eq:sub_opt_lag_bound_3} 
\end{align*}
Since $c_i \in [0, 1]^{SA}$, using~\cref{eq:vdl_upper_bound_lag} from~\cref{lemma:generic_vdl_upper_bound} with $\diamond = c_i$, $U = 1$, $\pi' = \pistarbl$, $\pi = \pi^*_\tau$ and $\lambda = \bar{\lambda}$
\begin{equation}
   \abs*{V^{\pi^*_\tau}_{c_i}(\rho) - V^{\pistarbl}_{c_i}(\rho)}  \leq 
   \frac{\sqrt{2 }}{(1 - \gamma)^{\nicefrac{3}{2}}} \frac{\sqrt{\abs{\lag(\pi^*_\tau, \blambda) - \lag(\pistarbl, \blambda)}}}{\sqrt{\tau}} \,  \label{eq:delta_reward_bound_2}. 
\end{equation}
Putting~\cref{eq:delta_reward_bound_2,eq:sub_opt_lag_bound_1,eq:sub_opt_lag_bound_3} we have
\begin{align*}
\max\braces*{b_i - V^{\pistarbl}_{c_i}(\rho), 0} \leq \gO\parens*{\frac{\sqrt{\Delta}}{\sqrt{\tau}}}
\end{align*}
\end{proof}

\subsection{Proof of~\cref{lemma:dual_smooth}}
We now prove $\cD(\lambda)$ is smooth without any converge assumptions. In order to do so, we first introduce and prove a couple of helpful technical tools. 
For clarity, denote the Boltzmann distribution 
for any $q \in \R^A$ and $\tau > 0$ as
\begin{equation*}
\sigma_\tau^*(q) \coloneqq \frac{\exp\parens*{\frac{1}{\tau}q}}{\sum_{a'} \exp\parens*{\frac{1}{\tau}q(a')}}.
\end{equation*}

\begin{proposition}\label{proposition:pistarl_lipschitz}
For any $\tau > 0$ and $q \in \R^A$, let
$\sigma_\tau^*(q) \coloneqq \argmax_{p \in \Delta_A} \braces*{\dpd{p, q} + \tau \, H(p)}$
where $H(p) \coloneqq \sum_{a} - p(a) \ln p(a)$. 
Then for any $q ,q' \in \R^A$ 
\begin{align}
    \norm{\sigma_\tau^*(q) - \sigma_\tau^*(q')}_1 &\leq \frac{1}{\tau} \, \supnorm{q - q'} \\ 
    \supnorm{\ln \sigma_\tau^*(q) - \ln \sigma_\tau^*(q')} &\leq \frac{2}{\tau} \, \supnorm{q - q'} \, \label{eq:sup_to_sup_lipschitz}.
\end{align}
\end{proposition}

\begin{proof}    
We first prove the first inequality. Using the fact that negative entropy $-H(p) \coloneqq \sum_{a} p(a) \, \ln(p(a))$ is $1$-strongly convex with respect to the $1$-norm, we have
for any $p, p' \in \Delta_A$
\begin{align*}
   -\tau H(p) + \tau H(p') + \dpd{ \tau \, \nabla H(p'), p - p'} &\geq \frac{\tau}{2} \, \norm{p - p'}_1^2
    \intertext{Exchanging $p$ and $p'$, we have}
   -\tau H(p') + \tau H(p) + \dpd{ \tau \, \nabla H(p), p' - p} &\geq \frac{\tau}{2} \, \norm{p' - p}_1^2  \, .
   \end{align*}
Adding the above inequalities together and rearranging, 
   \begin{align*}
   \norm{p - p'}^2_1 \leq \, \frac{1}{\tau} \, \dpd{\tau \,\nabla H(p') - \tau \, \nabla H(p), p - p'} \numberthis \label{eq:rel_smooth_entropy}.
\end{align*}
For a fixed $q \in \R^A$ note that
\begin{equation*}
   \sigma_\tau^*(q) = \frac{\exp\parens*{\frac{1}{\tau}q}}{\underbrace{\sum_{a'} \exp\parens*{\frac{1}{\tau}q(a')}}_{\coloneqq Z(q)}   }= \argmax_{p \in \Delta_A} \braces*{\dpd{p, q} + \tau \, H(p)} 
\end{equation*}
Hence for any $q, q' \in \R^A$, setting~\cref{eq:rel_smooth_entropy} with $p' = \sigma^*_\tau(q')$ and $p = \sigma^*_\tau(q)$ we have,
\begin{align*}
   \norm{\sigma^*_\tau(q) - \sigma^*_\tau(q')}^2_1  & \leq \frac{1}{\tau}
  \dpd{\tau \, \nabla H(\sigma^*_\tau(q')) - \tau \, \nabla H(\sigma^*_\tau(q)), \sigma^*_\tau(q) - \sigma^*_\tau(q')} 
\end{align*}
Given that $[\nabla H(\sigma^*_\tau(q))]_a = -q(a) / \tau + \ln Z(q) - 1$ for any $q \in \R^A$, we can further simplify the inner-product,
\begin{align*}
\MoveEqLeft
\dpd{\tau \, \nabla H(\sigma^*_\tau(q')) - \tau \, \nabla H(\sigma^*_\tau(q)), \sigma^*_\tau(q) - \sigma^*_\tau(q')} \\ &= 
\dpd{q - q' + \tau \, (\ln(Z(q')) - \ln(Z(q)) \, 1), \sigma^*_\tau(q) - \sigma^*_\tau(q')}\\
&= \dpd{q - q', \sigma^*_\tau(q) - \sigma^*_\tau(q')}  \tag{$\dpd{1, \sigma^*_\tau(q) - \sigma^*_\tau(q')} = 0$}
\end{align*}
\begin{align*}
\implies \norm{\sigma^*_\tau(q) - \sigma^*_\tau(q')}^2_1  &\leq \frac{1}{\tau}\,\supnorm{q' - q} \, \norm{\sigma^*_\tau(q) - \sigma^*_\tau(q')}_1 \tag{Using H\"older's inequality}
\end{align*}
Dividing both sides by $\norm{\sigma^*_\tau(q) - \sigma^*_\tau(q')}_1$ proves the first inequality.

We now prove the second inequality. Given that for any $a \in \cA$,
    $[\ln\parens*{\sigma^*_\tau(q)}]_a = \frac{q(a)}{\tau} - \ln(Z(q))$, we have
\begin{align*}
    [\ln\parens*{\sigma^*_\tau(q)} - \ln\parens*{\sigma^*_\tau(q')}]_a
    &= \underbrace{\frac{1}{\tau} \, [q - q']_a}_{\coloneqq \delta_\tau(a)} - \parens*{\ln(Z(q)) - \ln(Z(q'))} \numberthis \label{eq:log_diff_decomp} \, .
\end{align*}
We now bound the difference of the logarithms.
Since $q(a') \leq q'(a') + \tau \, \max_{a} \delta_\tau(a)$ for every $a' \in \cA$, by the monotonicity of
$x \to \exp(x/\tau)$ we have
\begin{align*}
    Z(q) = \sum_{a'} \exp\parens*{\frac{q(a')}{\tau}}
    \leq \exp\parens*{\max_{a} \delta_\tau(a)} \sum_{a'} \exp\parens*{\frac{q'(a')}{\tau}}
    = \exp\parens*{\max_{a} \delta_\tau(a)} \, Z(q') \, .
\end{align*}
Similarly, since $q'(a') \leq q(a') - \tau \min_{a} \delta_\tau(a)$ for every $a' \in \cA$, we have
\begin{equation*}
   Z(q') \leq \exp\parens*{-\min_{a} \delta_\tau(a)} \, Z(q)  \implies Z(q) \geq \exp\parens*{\min_{a} \delta_\tau(a)} \, Z(q')
\end{equation*}
Hence, combining both inequalities, 
\begin{align*}
   \exp\parens*{\min_{a} \delta_\tau(a)} \, Z(q')  &\leq Z(q) \leq \exp\parens*{\max_{a} \delta_\tau(a)} \, Z(q') \\
   \intertext{Taking the logarithms and rearranging,}
\implies    \min_{a} \delta_\tau(a) & \leq \ln(Z(q)) - \ln(Z(q')) \leq  \max_{a} \delta_\tau(a)
    \numberthis \label{eq:logZ_sandwich} \, .
\end{align*}
Combining~\cref{eq:log_diff_decomp} and~\cref{eq:logZ_sandwich}, for every $a \in \cA$,
\begin{align*}
    \delta_\tau(a) - \max_{a'} \delta_\tau(a')
     \leq  [\ln\parens*{\sigma^*_\tau(q)} - \ln\parens*{\sigma^*_\tau(q')}]_a
     &\leq \delta_\tau(a) - \min_{a'} \delta_\tau(a') \\ 
    \implies \abs*{[\ln\parens*{\sigma^*_\tau(q)} - \ln\parens*{\sigma^*_\tau(q')}]_a} &\leq \max\braces*{\abs*{\delta_\tau(a) - \min_{a'} \delta_\tau(a')}, \abs*{\delta_\tau(a) - \max_{a'} \delta_\tau(a')}} \\
    &\leq \abs*{\max_{a'} \delta_\tau(a') - \min_{a''} \delta_\tau(a'')} \\
    &\leq \abs*{\max_{a'} \delta_\tau(a')} + \abs*{\min_{a''} \delta_\tau(a'')}\tag{Using triangle inequality} \\
    &\leq 2 \, \supnorm{\delta_\tau}
\end{align*}
Taking the maximum over $a$,
\begin{align*}
    \implies \supnorm{\ln\parens*{\sigma^*_\tau(q)} - \ln\parens*{\sigma^*_\tau(q')}}
     \leq 2 \, \supnorm{\delta_\tau}
    = \frac{2}{\tau} \, \supnorm{q - q'} \, .
\end{align*}
\end{proof}

\begin{lemma}\label{lemma:qstar_lipschitz}
For any $\tau > 0$, and $\lambda, \lambda' \in \R_+^m$,
\begin{equation*}
    \supnorm{Q^{\pistarl}_{R(\lambda), \tau} - Q^{\pistarlp}_{R(\lambda'), \tau}} \leq \frac{1}{1 - \gamma} \, \norm{\lambda - \lambda'}_1 \, .
\end{equation*}
\end{lemma}

\begin{proof}
\begin{align*}
&\supnorm{Q^{\pistarl}_{R(\lambda), \tau} - Q^{\pistarlp}_{R(\lambda'), \tau}} \\
&= \supnorm{\cT_{R(\lambda), \tau} Q^{\pistarl}_{R(\lambda), \tau} - \cT_{R(\lambda'), \tau} Q^{\pistarlp}_{R(\lambda'), \tau}} \tag{$Q^{\pistarl}_{R(\lambda),\tau}$ is a fixed point of $\cT_{R(\lambda), \tau}$} \\
&=\supnorm{\cT_{R(\lambda), \tau} Q^{\pistarl}_{R(\lambda), \tau} -  \cT_{R(\lambda), \tau} Q^{\pistarlp}_{R(\lambda'), \tau} +\cT_{R(\lambda), \tau} Q^{\pistarlp}_{R(\lambda'), \tau} - \cT_{R(\lambda'), \tau} Q^{\pistarlp}_{R(\lambda'), \tau}} \tag{Add/Subtract $\cT_{R(\lambda), \tau} Q^{\pistarlp}_{R(\lambda'), \tau}$}\\ 
&\leq \supnorm{\cT_{R(\lambda), \tau} Q^{\pistarl}_{R(\lambda), \tau} -  \cT_{R(\lambda), \tau} Q^{\pistarlp}_{R(\lambda'), \tau}}  + \supnorm{\cT_{R(\lambda), \tau} Q^{\pistarlp}_{R(\lambda'), \tau} - \cT_{R(\lambda'), \tau} Q^{\pistarlp}_{R(\lambda'), \tau}}
\tag{Using triangle inequality}
\end{align*}
To bound the first term $\supnorm{\cT_{R(\lambda), \tau} Q^{\pistarl}_{R(\lambda), \tau} -  \cT_{R(\lambda), \tau} Q^{\pistarlp}_{R(\lambda'), \tau}}$, recall that for any $Q \in \R^{SA}$, the soft Bellman operator $\cT_{R(\lambda), \tau}$ is defined as
\begin{equation*}
    \bracks*{\cT_{R(\lambda), \tau} Q}(s,a) = [R(\lambda)](s,a) + \gamma \sum_{s'} P(s' | s, a) \, \tau \ln \sum_{a'} \exp\parens*{\frac{1}{\tau} Q(s', a')} \, .
\end{equation*}
Then for a fixed $(s, a) \in (\cS \times \cA)$ pair
\begin{align*}
    &\abs*{\bracks*{\cT_{R(\lambda), \tau} Q^{\pistarl}_{R(\lambda), \tau} - \cT_{R(\lambda), \tau} Q^{\pistarlp}_{R(\lambda'), \tau}}(s,a)} \\
    &= \gamma \abs*{\sum_{s'} P(s'|s,a) \bracks*{\tau \ln \sum_{a'} \exp\parens*{\frac{1}{\tau} Q^{\pistarl}_{R(\lambda), \tau}(s', a')} - \tau \ln \sum_{a'} \exp\parens*{\frac{1}{\tau} Q^{\pistarlp}_{R(\lambda'), \tau}(s', a')}}}  \\
    \intertext{Using triangle inequality }
    &\leq  \gamma \sum_{s'} P(s'|s,a) \abs*{\tau \ln \sum_{a'} \exp\parens*{\frac{1}{\tau} Q^{\pistarl}_{R(\lambda),\tau}(s', a')} - \tau \ln \sum_{a'} \exp\parens*{\frac{1}{\tau} Q^{\pistarlp}_{R(\lambda'),\tau}(s', a')}} \numberthis \label{eq:lip_qlambda_bound_0}
\end{align*}
Furthermore, for a fixed $(s', a') \in \cS \times \cA$ pair, since 
\begin{equation*}
Q^{\pistarl}_{R(\lambda), \tau}(s', a') - Q^{\pistarlp}_{R(\lambda'),\tau}(s', a') \leq
\supnorm{Q^{\pistarl}_{R(\lambda), \tau}(s', \cdot)  -Q^{\pistarlp}_{R(\lambda'), \tau}(s', \cdot)}   
\end{equation*}
we have
\begin{align*}
\implies \exp\parens*{\frac{1}{\tau}  Q^{\pistarl}_{R(\lambda), \tau}(s', a')} &\leq  \exp\parens*{\frac{1}{\tau} \parens*{Q^{\pistarlp}_{R(\lambda'),\tau}(s', a') +
\supnorm{Q^{\pistarl}_{R(\lambda), \tau}(s', \cdot)  -Q^{\pistarlp}_{R(\lambda'), \tau}(s', \cdot)}}} \\
    \implies \tau \ln \sum_{a'} \exp\parens*{\frac{1}{\tau} Q^{\pistarl}_{R(\lambda), \tau}(s', a')}     &\leq \tau \ln \sum_{a'} \exp\parens*{\frac{1}{\tau} \parens*{Q^{\pistarlp}_{R(\lambda'), \tau}(s', a') + \supnorm{Q^{\pistarl}_{R(\lambda), \tau}(s', \cdot) - Q^{\pistarlp}_{R(\lambda'), \tau}(s', \cdot)}}} \tag{Summing over all $a' \in \cA$ and by the monotonicity of  $\ln(\cdot)$}\\
&\leq \supnorm{Q^{\pistarl}_{R(\lambda), \tau}(s', \cdot) - Q^{\pistarlp}_{R(\lambda'), \tau}(s', \cdot)} + \tau \ln \sum_{a'} \exp\parens*{\frac{1}{\tau} Q^{\pistarlp}_{R(\lambda'), \tau}(s', a')} \, , 
\intertext{Swapping $Q^{\pistarl}_{R(\lambda), \tau}$ and $Q^{\pistarlp}_{R(\lambda'),\tau}$ and following a similar argument}
\implies \tau \ln \sum_{a'} \exp(\frac{1}{\tau} Q^{\pistarlp}_{R(\lambda'), \tau}(s', a')) &\leq \supnorm{Q^{\pistarl}_{R(\lambda), \tau}(s', \cdot) - Q^{\pistarlp}_{R(\lambda'), \tau}(s', \cdot)} + \tau \ln \sum_{a'} \exp(\frac{1}{\tau} Q^{\pistarl}_{R(\lambda), \tau}(s', a'))
\end{align*}
\begin{align*}
\implies \abs*{
\tau \ln \sum_{a'} \exp\parens*{\frac{1}{\tau} Q^{\pistarl}_{R(\lambda), \tau}(s', a')} - \tau \ln \sum_{a'} \exp\parens*{\frac{1}{\tau} Q^{\pistarlp}_{R(\lambda'), \tau}(s', a')}} &\leq \supnorm{Q^{\pistarl}_{R(\lambda), \tau} - Q^{\pistarlp}_{R(\lambda'), \tau}} \numberthis \label{eq:lip_qlambda_bound_1}\, ,   
\end{align*}
Combining~\cref{eq:lip_qlambda_bound_0,eq:lip_qlambda_bound_1}, we have
\begin{align*}
    &\abs*{\parens*{\cT_{R(\lambda), \tau} Q^{\pistarl}_{R(\lambda), \tau} - \cT_{R(\lambda), \tau} Q^{\pistarlp}_{R(\lambda'), \tau}}(s,a)} \\ &\leq \gamma \sum_{s'} P(s' | s, a) \,\supnorm{Q^{\pistarlp}_{R(\lambda'), \tau}(s', \cdot) - Q^{\pistarl}_{R(\lambda), \tau}(s', \cdot)}  \leq \gamma \supnorm{Q^{\pistarlp}_{R(\lambda'), \tau} - Q^{\pistarl}_{R(\lambda), \tau}} 
    \tag{$\sum_{s'} P(s' | s, a) = 1$}
\end{align*}
Taking the max over all $(s, a) \in \cS \times \cA$ pairs,
\begin{align*}
    \implies  \supnorm{\cT_{R(\lambda), \tau} Q^{\pistarl}_{R(\lambda), \tau} -  \cT_{R(\lambda), \tau} Q^{\pistarlp}_{R(\lambda'), \tau}}  &\leq \gamma \supnorm{Q^{\pistarlp}_{R(\lambda'), \tau} - Q^{\pistarl}_{R(\lambda), \tau}}  \, . \numberthis \label{eq:gamma_contraction}
\end{align*}
We now bound the second term $\supnorm{\cT_{R(\lambda), \tau} Q^{\pistarlp}_{R(\lambda'), \tau} - \cT_{R(\lambda'), \tau} Q^{\pistarlp}_{R(\lambda'), \tau}}$.
Since $\cT_{R(\lambda), \tau}$ and $\cT_{R(\lambda'), \tau}$ only differ by their pseudo-reward, for any $(s, a) \in \cS \times \cA$ 
\begin{align*}
&\abs*{\bracks*{\cT_{R(\lambda), \tau} Q^{\pistarlp}_{R(\lambda'), \tau}}(s,a) - \bracks*{\cT_{R(\lambda'), \tau} Q^{\pistarlp}_{R(\lambda'), \tau}}(s,a)}  \\
&= \abs*{[R(\lambda)](s,a) - [R(\lambda')](s,a)}  
= \abs*{\sum_{i=1}^m (\lambda_i - \lambda'_i) \, \, c_i(s,a)} \   \tag{$[R(\lambda)](s, a) = r(s, a) + \sum_{i=1}^m \lambda_i \, c_i(s, a)$} \\
&\leq  \sum_{i=1}^m \abs*{\lambda_i - \lambda'_i} \tag{Since $c_i \in [0, 1]^{SA}$ and using triangle inequality}
\end{align*}
Taking the max over all $(s, a) \in \cS \times \cA$ pairs,
\begin{align*}
\implies \supnorm{\cT_{R(\lambda), \tau} Q^{\pistarlp}_{R(\lambda'), \tau} - \cT_{R(\lambda'), \tau} Q^{\pistarlp}_{R(\lambda'), \tau}} &\leq \norm{\lambda - \lambda'}_1 \numberthis \label{eq:q_r_lambda_norm_bound}
\end{align*}
Putting everything together, we have
\begin{align*}
   &\supnorm{Q^{\pistarl}_{R(\lambda), \tau} - Q^{\pistarlp}_{R(\lambda'), \tau}} \\ &\leq
  \supnorm{\cT_{R(\lambda), \tau} Q^{\pistarl}_{R(\lambda), \tau} -  \cT_{R(\lambda), \tau} Q^{\pistarlp}_{R(\lambda'), \tau}}  + \supnorm{\cT_{R(\lambda), \tau} Q^{\pistarlp}_{R(\lambda'), \tau} - \cT_{R(\lambda'), \tau} Q^{\pistarlp}_{R(\lambda'), \tau}}  \\
  &\leq \gamma \,  \supnorm{Q^{\pistarl}_{R(\lambda), \tau} - Q^{\pistarlp}_{R(\lambda'), \tau}} + \supnorm{\cT_{R(\lambda), \tau} Q^{\pistarlp}_{R(\lambda'), \tau} - \cT_{R(\lambda'), \tau} Q^{\pistarlp}_{R(\lambda'), \tau}} \tag{Using~\cref{eq:gamma_contraction}} \\
  &\leq \gamma \,  \supnorm{Q^{\pistarl}_{R(\lambda), \tau} - Q^{\pistarlp}_{R(\lambda'), \tau}} + \norm{\lambda - \lambda'}_1\, \tag{Using~\cref{eq:q_r_lambda_norm_bound}}  \, .
  \end{align*}
and rearranging the above inequality completes the proof.
\end{proof}

\dualsmooth*
\begin{proof}
For any $\tau > 0$ and $\lambda, \lambda' \in \R^m_+$,
\begin{align*}
\norm{\nabla \cD(\lambda) - \nabla \cD(\lambda')} &\leq \sqrt{m} \, \norm{\nabla \cD(\lambda) - \nabla \cD(\lambda')}_\infty \\
&= \sqrt{m} \max_{i \in [m]}\abs*{V^{\pistarl}_{c_i}(\rho) - b_i -(V^{\pistarlp}_{c_i}(\rho) - b_i)} \tag{By~\cref{proposition:dual_gradient}, $[\nabla \cD(\lambda)]_i = V^{\pistarl}_{c_i}(\rho) - b_i$} \\
&= \sqrt{m} \max_{i \in [m]}\abs*{V^{\pistarl}_{c_i}(\rho) - V^{\pistarlp}_{c_i}(\rho)}  \, .
\intertext{By the value difference lemma, for any $i \in [m]$, }
V^{\pistarl}_{c_i}(\rho) - V^{\pistarlp}_{c_i}(\rho) &= \frac{1}{1 - \gamma} \sum_{s \in \cS} d^{\pistarlp}(s) \dpd{\pistarl(\cdot | s) - \pistarlp(\cdot | s), Q^{\pistarl}_{c_i}(s, \cdot)} \\
 &\leq  \frac{1}{1 - \gamma} \sum_{s \in \cS} d^{\pistarlp}(s) 
\norm{\pistarl(\cdot | s) - \pistarlp(\cdot | s)}_1 \supnorm{Q^{\pistarl}_{c_i}(s, \cdot)}  \tag{Using H\"older's inequality} \\
 &\leq  \frac{1}{(1 - \gamma)^2} \sum_{s \in \cS} d^{\pistarlp}(s) 
\norm{\pistarl(\cdot | s) - \pistarlp(\cdot | s)}_1 \tag{$c_i \in [0, 1]^{SA}$} \\
&\leq \frac{1}{(1 - \gamma)^2} \max_s  \norm{\pistarl(\cdot | s) - \pistarlp(\cdot | s)}_1 \, \underbrace{\sum_{s \in \cS} d^{\pistarlp}(s)}_{=1} \\
\implies \norm{\nabla \cD(\lambda) - \nabla \cD(\lambda')} &\leq \frac{\sqrt{m}}{(1 - \gamma)^2}  \, \max_s \norm{\pistarl(\cdot | s) - \pistarlp(\cdot | s)}_1 \numberthis \label{eq:dual_smooth_bound_by_max_pi}
\end{align*}
We now relate
$\max_s \norm{\pistarl(\cdot | s) - \pistarlp(\cdot | s)}_1$ to $\norm{\lambda - \lambda'}$. 
Since 
$\pistarl$ is the optimal policy of an entropy-regularized MDP, for any $s \in \cS$, 
\begin{align*}
 \pistarl(\cdot | s)  = \frac{\exp\parens*{\frac{1}{\tau}Q^{\pistarl}_{R(\lambda), \tau}(s, \cdot)}}{\sum_{a'} \exp\parens*{\frac{1}{\tau}Q^{\pistarl}_{R(\lambda), \tau}(s, a')}} = \argmax_{p \in \Delta_A} \, \braces*{\dpd{p, Q^{\pistarl}_{R(\lambda), \tau}(s, \cdot)} + \tau H(p)}
\end{align*}
Hence, according to~\cref{proposition:pistarl_lipschitz} with $q =Q^{\pistarl}_{R(\lambda), \tau}(s, \cdot)$ and $q' = Q^{\pistarlp}_{R(\lambda'), \tau}(s, \cdot)$
\begin{align*}
   \norm{\pistarl(\cdot | s) - \pistarlp(\cdot | s)}_1 &\leq \frac{1}{\tau} \, \supnorm{Q^{\pistarl}_{R(\lambda), \tau}(s, \cdot) - Q^{\pistarlp}_{R(\lambda'), \tau}(s, \cdot)} \\
   &\leq \frac{1}{\tau} \, \supnorm{Q^{\pistarl}_{R(\lambda), \tau} - Q^{\pistarlp}_{R(\lambda'), \tau}} \\
   &\leq \frac{1}{\tau \, (1 - \gamma)} \norm{\lambda - \lambda'}_1 \tag{Using~\cref{lemma:qstar_lipschitz} } \\ 
\implies \max_{s \in \cS} \norm{\pistarl(\cdot | s) - \pistarlp(\cdot | s)}_1  &\leq \frac{\sqrt{m}}{\tau \, (1 - \gamma)} \, \norm{\lambda - \lambda'} \tag{For $x \in \R^m$, $\norm{x}_1 \leq \sqrt{m} \, \norm{x}$}
\end{align*}
Combining the above inequality with~\cref{eq:dual_smooth_bound_by_max_pi},
\begin{align*}
\implies \norm{\nabla \cD(\lambda) - \nabla \cD(\lambda')} &\leq \frac{m}{\tau (1 - \gamma)^3}  \, \norm{\lambda - \lambda'} \, .
\end{align*}
\end{proof}

\subsection{Proof of~\cref{lemma:meta_lemma_smooth}}
\smoothmetalemma*
\begin{proof}
By~\cref{lemma:dual_smooth}, the dual function $\cD(\lambda)$ is $G_\tau = \frac{m}{\tau \, (1 - \gamma)^3}$-smooth. Hence,
for all $t \geq 0$,
\begin{equation*}
   \cD(\ltt) \leq \cD(\lt) + \dpd{\nabla \cD(\lt), \ltt - \lt}  + \frac{G_\tau}{2} \, \normsq{\lt - \ltt} 
\end{equation*}
Furthermore, since $\cD(\lambda)$ is convex, for all $u, v \in \R^m$,
\begin{align*}
   \cD(v) \geq \cD(u) + \dpd{\nabla \cD(u), v - u}  
   \implies \dpd{\nabla \cD(u), u - v} & \geq \cD(u) - \cD(v) 
\end{align*}
Using this fact with $u = \lt$, $v = \lstar$, and rearranging,
\begin{align*}
\implies \cD(\lt) \leq \cD(\lstar) + \dpd{\nabla \cD(\lt), \lt - \lstar}   \numberthis \, \label{eq:dd_norm_bound3}
\end{align*}
Adding the above inequalities,
\begin{align*}
\implies &\cD(\ltt) \leq \cD(\lstar) + \dpd{\nabla \cD(\lt), \ltt - \lstar}  + \frac{G_\tau}{2} \, \normsq{\lt - \ltt}
\end{align*}
Add/Subtract $\dpd{\bdgt, \ltt - \lstar} $,
\begin{align}
\implies \cD(\ltt) \leq \cD(\lstar) + \underbrace{\dpd{\nabla \cD(\lt) - \bdgt, \ltt - \lstar}}_{\text{Term (i)}} +\underbrace{\dpd{\bdgt, \ltt - \lstar}}_{\text{Term (ii)}} + \frac{G_\tau}{2} \, \normsq{\lt - \ltt}  \label{eq:meta_smooth_bound_0}
\end{align}
We now bound Term (i),
\begin{align*}
\abs*{\dpd{\dgt -  \bdgt, \ltt - \lstar}} &\leq \norm{\bdgt - \dgt} \, \norm{\ltt - \lstar} \tag{Using Cauchy-Schwarz} \\
&\leq \delta_t \,  \norm{\ltt - \lstar} \tag{Assuming for all $ t \in [T]$, $\norm{\bdgt - \dgt} \leq \delta_t$}  \\
&\leq \delta_t \, \abs*{\Lambda}  \tag{Under~\cref{assumption:slater_condition}, by~\cref{lemma:optimal_dual_bounded}, $\norm{\ltt - \lstar}\leq \frac{\sqrt{m} \, (2 + 2 \, \tau \, \ln(A))}{(1 - \gamma) \, \min_{i \in [m]} \zeta_i} \coloneqq \abs*{\Lambda}$} 
\end{align*}
\begin{align*}
\implies \dpd{ \dgt - \bdgt, \ltt - \lstar} &\leq \delta_t \, \abs*{\Lambda} \numberthis  \, \label{eq:meta_smooth_bound_1}.
\end{align*}
To bound Term (ii), note the update in~\cref{algorithm:dd} (\cref{line:update}) with $\mu = 0$, 
\begin{equation*}
 \ltt = \text{Proj}_{\Lambda}[\lt - \beta \, \bdgt]   \, ,
\end{equation*}
is a projection onto a convex set $\Lambda$. Hence, for any $z \in \Lambda$
\begin{align*}
    \dpd{\ltt - z, \ltt - (\lt - \beta \, \bdgt)} &\leq 0 \\
    \intertext{Setting $z=\lstar$ and rearranging}
    \implies  \dpd{\ltt - \lstar, \bdgt} &\leq  \frac{1}{\beta} \, \dpd{\ltt - \lstar, \lt - \ltt} 
    \numberthis \label{eq:meta_smooth_bound_2}
\end{align*}
Putting~\cref{eq:meta_smooth_bound_0,eq:meta_smooth_bound_1,eq:meta_smooth_bound_2} together, 
\begin{align*}
    \MoveEqLeft
    \implies \cD(\ltt) \leq \cD(\lstar) + \delta_t \, \abs*{\Lambda} 
    +\frac{1}{\beta} \, \dpd{\ltt - \lstar, \lt - \ltt}  +\frac{G_\tau}{2} \, \normsq{\lt - \ltt} \\ 
    &= \cD(\lstar) + \delta_t \, \abs*{\Lambda}+\frac{1}{2 \, \beta} \parens*{\normsq{\lt - \lstar} - \normsq{\ltt - \lstar} - \normsq{\lt - \ltt}} + \frac{G_\tau}{2} \, \normsq{\lt - \ltt} \tag{For any $u, v \in \R^m$ $2 \, \dpd{u, v} = \normsq{u+v} - \normsq{u} - \normsq{v}$} \\
    &= \cD(\lstar) + \delta_t \, \abs*{\Lambda} + \frac{1}{2 \, \beta} \parens*{\normsq{\lt - \lstar} - \normsq{\ltt - \lstar} } + \underbrace{\parens*{\frac{G_\tau}{2} - \frac{1}{2 \, \beta}}}_{\leq 0}\, \normsq{\lt - \ltt}  \tag{Since $\beta \leq \frac{1}{G_\tau}$}
\end{align*}
\begin{align*}
\implies \cD(\ltt) - \cD(\lstar) &\leq \delta_t \, \abs{\Lambda} + \frac{1}{2 \, \beta} \parens*{\normsq{\lt - \lstar} - \normsq{\ltt - \lstar} } \numberthis \label{eq:smooth_recurse_1}
\end{align*}
Using~\cref{eq:smooth_recurse_1}, summing from $t=0$ to $T-1$, then dividing by $T$,
\begin{align*}
    \sum_{t=0}^{T-1} \frac{\cD(\ltt) - \cD(\lstar)}{T} &\leq \frac{\normsq{\lambda_0 - \lstar}}{2 \, \beta \, T} - \underbrace{\frac{\normsq{\lambda_T - \lstar}}{2 \, \beta \, T}}_{\geq 0} + \frac{\abs{\Lambda} \, \sum_{t=0}^{T-1} \delta_t}{T} \\
    &\leq \frac{\normsq{\lstar}}{2 \, \beta \, T}  + \frac{\abs{\Lambda} \, \sum_{t=0}^{T-1} \delta_t}{T} \tag{Since $\lambda_0 = 0$}
\end{align*}
Using Jensen's inequality since $\cD(\lambda)$ is convex and defining $\bar{\lambda}_T \coloneqq \frac{1}{T} \sum_{t=1}^{T} \lambda_{t}$,
\begin{equation*}
    \cD(\bar{\lambda}_T) - \cD(\lstar) \leq \frac{\normsq{\lstar}}{2 \, \beta \, T}  + \frac{\abs{\Lambda} \, \sum_{t=0}^{T-1} \delta_t}{T}  \, .
\end{equation*}
\end{proof}

\subsection{Proof of~\cref{lemma:bound_bias_subgradient}}
\boundbiassubgradientlemma*
\begin{proof}
According to~\cref{proposition:dual_gradient}, for a fixed $\lambda \in \R^m_{+}$ the gradient of the dual function of~\cref{problem:soft_max_min_lag} is defined as
\begin{align*}
[\nabla \cD(\lambda)]_i &\coloneqq V^{\pistargtau}_{c_i}(\rho) - b_i \quad \forall i \in [m]. \\
\intertext{where $\pistargtau = \argmax_{\pi} \lag(\pi, \lambda) = \argmax_{\pi} V^{\pi}_{R(\lambda), \tau}(\rho)$. Additionally, for a fixed policy $\tilde{\pi}$, denote the inexact gradient of the dual function as}
[\widetilde{\nabla} \cD(\lambda)]_{i} &\coloneqq V^{\tilde{\pi}}_{c_i}(\rho) - b_i  \quad \forall i \in [m].
\end{align*}
As a result, we have
\begin{align*}
\implies \norm{\dg - \bdg} &\leq \sqrt{m} \, \norm{\dg - \bdg}_\infty \\
&= \sqrt{m} \max_{i \in [m]}\abs*{V^{\pistarl}_{c_i}(\rho) - V^{\tilde{\pi}}_{c_i}(\rho)} \numberthis \label{eq:dd_sub_bound_0}.
\end{align*}
For a fixed $i \in [m]$, since $c_i \in [0, 1]^{SA}$, using~\cref{eq:vdl_upper_bound_lag} from~\cref{lemma:generic_vdl_upper_bound} with $\diamond = c_i$, $U=1$, $\pi' = \pistarl$, and $\pi = \tilde{\pi}$
\begin{equation*}
\abs*{V^{\pistarl}_{c_i}(\rho) - V^{\tilde{\pi}}_{c_i}(\rho)} \leq  \frac{\sqrt{2}}{(1 - \gamma)^{\nicefrac{3}{2}}} \frac{\sqrt{\abs*{\cL_\tau(\pistarl, \lambda) - \cL_\tau(\tilde{\pi}, \lambda)}}}{\sqrt{\tau}} \label{eq:dd_sub_bound_1} \, 
\end{equation*}
Further simplifying $\lag(\pistarl, \lambda) - \lag(\tilde{\pi}, \lambda)$ we have,
\begin{align*}
   &\lag(\pistarl, \lambda) - \lag(\tilde{\pi}, \lambda)  \\
   &= V^{\pistarl}_{r, \tau}(\rho) + \sum_{i=1}^m \lambda_i [V^{\pistarl}_{c_i}(\rho) - b_i] - V^{\tilde{\pi}}_{r, \tau}(\rho) - \sum_{i=1}^m \lambda_i [V^{\tilde{\pi}}_{c_i}(\rho) - b_i] \tag{Definition of $\lag(\pi, \lambda)$} \\
   &= V^{\pistarl}_{R(\lambda), \tau}(\rho) - V^{\tilde{\pi}}_{R(\lambda), \tau}(\rho) \tag{$R(\lambda) \coloneqq r + \sum_{i=1}^m \lambda_i \, c_i$}
\end{align*}
\begin{equation*}
   \implies \abs*{V^{\pistarl}_{c_i}(\rho) - V^{\tilde{\pi}}_{c_i}(\rho)}  \leq  
  \frac{\sqrt{2}}{(1 - \gamma)^{\nicefrac{3}{2}} \, \sqrt{\tau}}  \, \sqrt{V^{\pistarl}_{R(\lambda), \tau}(\rho) - V^{\tilde{\pi}}_{R(\lambda), \tau}(\rho)}
\end{equation*}
Combining the above inequality with~\cref{eq:dd_sub_bound_0} completes the proof.
\end{proof}    

\subsection{Helper Lemmas}
\begin{lemma}\label{lemma:generic_vdl_upper_bound}
For any pseudo-reward $\diamond \in [0, U]^{SA}$, $U > 0$ and policies $\pi$, $\pi'$
\begin{align}
    \abs{V^{\pi'}_{\diamond}(\rho) - V^{\pi}_\diamond(\rho)} &\leq \frac{U}{(1 - \gamma)^2} \, \sqrt{\sum_{s \in \cS} d^\pi(s) \, \norm{\pi(\cdot | s) - \pi'(\cdot | s)}^2_1} \, .\label{eq:vdl_upper_bound_norm} \\
    \intertext{Furthermore, for any $\lambda \in \R^m_+$,  if $\pi' = \argmax_{\pi} \cL_\tau(\pi, \lambda)$,}
   \abs*{V^{\pi'}_{\diamond}(\rho) - V^{\pi}_{\diamond}(\rho)} &\leq \frac{U \, \sqrt{2}}{(1 - \gamma)^{\nicefrac{3}{2}}} \,  \frac{\sqrt{\abs*{\cL_\tau(\pi', \lambda) - \cL_\tau(\pi, \lambda)}}}{\sqrt{\tau}} \, . \label{eq:vdl_upper_bound_lag}
\end{align}
\end{lemma}
\begin{proof}
We will first prove~\cref{eq:vdl_upper_bound_norm}. By the value difference lemma,
\begin{align*}
V^{\pi'}_{\diamond}(\rho) - V^{\pi}_{\diamond}(\rho) &= \frac{1}{1 - \gamma} \sum_{s \in \cS} d^{\pi}(s) \dpd{\pi'(\cdot | s) - \pi(\cdot | s), Q^{\pi'}_{\diamond}(s, \cdot)} \\
 &\leq  \frac{1}{1 - \gamma} \sum_{s \in \cS} d^{\pi}(s) 
\norm{\pi'(\cdot | s) - \pi(\cdot | s)}_1 \supnorm{Q^{\pi'}_{\diamond}(s, \cdot)}  \tag{Using H\"older's inequality} \\
&\leq \frac{U}{(1 - \gamma)^2} \sum_{s \in \cS} d^{\pi}(s)  \norm{\pi'(\cdot | s) - \pi(\cdot | s)}_1 \tag{Since $\diamond \in [0, U]^{SA}$}\\   
&= \frac{U}{(1 - \gamma)^2} \E_{s \sim d^{\pi}}\bracks*{\sqrt{\norm{\pi'(\cdot | s) - \pi(\cdot | s)}_1^2}} \\
&\leq \frac{U}{(1 - \gamma)^2} \sqrt{\E_{s \sim d^{\pi}}\bracks*{\norm{\pi'(\cdot | s) - \pi(\cdot | s)}_1^2}} \tag{Using Jensen's inequality} \\
\implies \abs*{V^{\pi'}_{\diamond}(\rho) - V^{\pi}_{\diamond}(\rho)} &\leq \frac{U}{(1 - \gamma)^2} \sqrt{\sum_{s \in \cS} d^{\pi}(s) \norm{\pi'(\cdot | s) - \pi(\cdot | s)}_1^2}  \, .
\end{align*}
\noindent We next prove~\cref{eq:vdl_upper_bound_lag}. Observe that for any $\lambda \in \R^m_+$,
\begin{align*}
    \lag(\pi, \lambda) &= V^{\pi}_{r, \tau}(\rho) + \sum_{i=1}^m \lambda_i \, [V^{\pi}_{c_i}(\rho) - b_i] = V^{\pi}_{R(\lambda), \tau}(\rho) - \sum_{i=1}^m \lambda_i \, b_i \tag{$R(\lambda) \coloneqq r + \sum_{i=1}^m \lambda_i \, c_i$} \\
    \implies \lag(\pi', \lambda) - \lag(\pi, \lambda) &= V^{\pi'}_{R(\lambda), \tau}(\rho) - V^{\pi}_{R(\lambda), \tau}(\rho) \\
    &= \frac{\tau}{1 - \gamma} \sum_{s \in \cS} d^\pi(s) \KL(\pi(\cdot | s) \| \pi'(\cdot | s)) \tag{Using~\cref{lemma:soft_suboptimality} with $\diamond = R(\lambda)$} \\
    \implies \abs*{\cL_\tau(\pi', \lambda) - \cL_\tau(\pi, \lambda)} &\geq \frac{\tau}{(1 - \gamma) \, 2} \, \sum_{s \in \cS} d^\pi(s) \norm{\pi(\cdot | s) - \pi'(\cdot | s)}^2_1  \\
    \tag{By Pinsker's inequality, for any $p, q \in \Delta_A$ $\KL(p \| q) \geq \frac{1}{2} \, \norm{p - q}^2_1$} \\
    \intertext{Combining the above inequality with~\cref{eq:vdl_upper_bound_norm},}
   \implies \abs*{V^{\pi'}_{\diamond}(\rho) - V^{\pi}_{\diamond}(\rho)} &\leq \frac{U \, \sqrt{2}}{(1 - \gamma)^{\nicefrac{3}{2}}} \, \frac{\sqrt{\abs*{\cL_\tau(\pi', \lambda) - \cL_\tau(\pi, \lambda)}}}{\sqrt{\tau}}
\end{align*}
\end{proof}

\begin{lemma}\label{lemma:optimal_dual_bounded}
Under~\cref{assumption:slater_condition}, the optimal dual variables $\lstar$ are bounded as
\begin{equation*}
    [\lstar]_i \leq \frac{2 + 2 \, \tau \, \ln(A)}{(1 - \gamma) \, \zeta_i} \quad \forall i \in [m]
\end{equation*}
where $\zeta_i = V^{\bar{\pi}}_{c_i}(\rho) - b_i > 0$ and $\bar{\pi}$ is the policy that satisfies Slater's condition.  Consequently, for any $\lambda \in \Lambda$,
\begin{equation*}
    \norm{\lambda - \lstar} \leq \frac{\sqrt{m} \, (2 + 2 \, \tau \, \ln(A))}{(1 - \gamma) \, \min_{i \in [m]} \, \zeta_i}
\end{equation*}
\end{lemma}
\begin{proof}
This proof follows from~\cite[Lemma 4.1]{jain2022towards} and is added for completeness.

Let $\pi^*_\tau$ denote the optimal policy of~\cref{problem:soft_max_min_lag}.  Hence,
\begin{align*}
    V^{\pi^*_\tau}_{r, \tau}(\rho) &= \min_{\lambda \in \R^m_+} \max_{\pi} \lag(\pi, \lambda)  \\
    &= \max_{\pi} \min_{\lambda \in \R^m_+} \lag(\pi, \lambda) \tag{Under~\cref{assumption:slater_condition}, by~\cref{lemma:strong_duality} strong duality holds} \\
    &= \max_{\pi} \lag(\pi, \lstar) \\
    &= \max_{\pi} V^{\pi}_{r, \tau}(\rho) + \sum_{i=1}^m [\lstar]_i \, [V^{\pi}_{c_i}(\rho)  - b_i]
\end{align*}
Under~\cref{assumption:slater_condition}, there exist a policy $\bar{\pi}$ such that $V^{\bar{\pi}}_{c_i}(\rho) - b_i = \zeta_i > 0$ for all $i \in [m]$. Thus,
\begin{align*}
   V^{\pi^*_\tau}_{r, \tau}(\rho)
   &\geq V^{\bar{\pi}}_{r, \tau}(\rho)
   + \sum_{i=1}^m \lambda_i^* \, (V^{\bar{\pi}}_{c_i}(\rho) - b_i) \\
   &= V^{\bar{\pi}}_{r, \tau}(\rho) + \sum_{i=1}^m \lambda_i^* \, \zeta_i \\
   \implies  \sum_{i=1}^m \lambda_i^* \, \zeta_i
&\leq V^{\pi^*_\tau}_{r, \tau}(\rho) - V^{\bar{\pi}}_{r, \tau}(\rho).
\end{align*}
Since $\zeta_i > 0$ for all $i \in [m]$, we have
\begin{align*}
\lambda_i^*
&\leq \frac{V^{\pi^*_\tau}_{r, \tau}(\rho) - V^{\bar{\pi}}_{r, \tau}(\rho)}{\zeta_i} \leq \frac{2 + 2 \, \tau \, \log(A)}{(1 - \gamma) \, \zeta_i} \tag{Since $V^{\pi}_{r, \tau}(\rho) \in \bracks*{0, \frac{2 + 2 \, \tau \, \log(A)}{1 - \gamma}}$}.
\end{align*}
Moreover, since for all $i \in [m]$, $\lambda_i \in \bracks*{0, \frac{2 + 2 \, \tau \, \log(A)}{(1 - \gamma) \, \min_{i \in [m]} \, \zeta_i}}$
\begin{align*}
    \supnorm{\lambda - \lstar} &\leq \frac{2 + 2 \, \tau \, \log(A)}{(1 - \gamma) \, \min_{i \in [m]} \, \zeta_i}  \implies \norm{\lambda - \lstar} \leq \frac{\sqrt{m} \, (2 + 2 \, \tau \, \log(A))}{(1 - \gamma) \, \min_{i \in [m]} \, \zeta_i}
\end{align*}
\end{proof}

\section{Convergence Results Using Smoothness}
\subsection{Using $K = \gO(\ln(\nicefrac{1}{\eps}))$ Inner-Iterations}\label{appendix:smooth_k_ln_eps}
\smoothddlastiteratetheorem*
\begin{proof}
We first construct a uniform bound of the dual gradient error. By~\cref{lemma:bound_bias_subgradient}, for all $t \geq 0$, we can relate the approximate gradient to the following optimization term:
\begin{equation*}
   \norm{\dgt - \bdgt} \leq  \frac{\sqrt{ 2 \, m }}{(1 - \gamma)^{\nicefrac{3}{2}} \, \sqrt{\tau}} \sqrt{ \max_{\pi} V^{\pi}_{R(\lambda_t), \tau}(\rho) - V^{\pi_{t}}_{R(\lambda_t), \tau}(\rho)}
\end{equation*}
To characterize the sub-optimality gap $\max_{\pi} V^{\pi}_{R(\lambda_t), \tau}(\rho) - V^{\pi_{t}}_{R(\lambda_t), \tau}(\rho)$, we instantiate the \texttt{PolicyOptimization} method in~\cref{algorithm:dd} using \texttt{NPG} update in~\cref{eq:entropy_npg_update}.
At iteration $t$, by~\cref{lemma:entropy_npg_supnorm_convergence}, with $\diamond = R(\lambda_t)$, $\pi_0 = \pi_{t-1}$, $\eta = \frac{1 - \gamma}{w \, \tau}$ and $w \geq 1$, the policy $\pi_{t} = \pi_{K}$ satisfies
\begin{align*}
    \max_{\pi} V^{\pi}_{R(\lambda_t), \tau}(\rho) - V^{\pi_{t}}_{R(\lambda_t), \tau}(\rho) \leq \gO\parens*{\frac{\tau}{T^2}} \, .
\end{align*}
after
\begin{equation}\label{eq:npg_updates}
   K \geq \gO\parens*{\frac{w}{1 - \gamma}\ln\parens*{\frac{T}{\tau}}} 
\end{equation}
iterations.
Hence, for all $t \in [T]$
\begin{equation*}
   \norm{\dgt - \bdgt} \leq  \gO\parens*{\frac{1}{T}} \coloneqq \delta_t.
\end{equation*}
and let $\hat{\delta} = \max_{t \in [T]} \delta_t$.

We now characterize the dual convergence using the above uniform gradient error. 
By~\cref{lemma:dual_smooth}, $\cD(\lambda)$ is $G_\tau = \frac{m}{\tau \, (1 - \gamma)^3}$-smooth.
Hence, under~\cref{assumption:slater_condition}, by~\cref{lemma:meta_lemma_smooth} with $\delta_t = \hat{\delta}$,~\cref{algorithm:dd} using step-size $\beta = \frac{1}{G_\tau}$, $\mu = 0$ and $\lambda_0 =0$ results in the following convergence:
\begin{align*}
\cD(\bar{\lambda}_T) - \cD(\lstar) \leq \frac{G_\tau \, \normsq{\lstar}}{2 \, T}  + \abs{\Lambda} \, \hat{\delta}  \leq \gO\parens*{\frac{G_\tau}{T}} = \gO\parens*{\frac{1}{\tau \, T}}
\end{align*}

Thus, by~\cref{lemma:dual_optimality_to_sub} with $\Delta = \gO\parens*{\frac{1}{\tau \, T}}$, the average dual variable $\bar{\lambda}_T$ returned by~\cref{algorithm:dd} satisfies
\begin{align*}
V^{\pi^*}_r(\rho) - V^{\pistarblT}_{r}(\rho) &\leq \tilde{\gO}\parens*{\frac{1}{{\sqrt{T} \, \tau}}} + \gO(\tau)  \\
\max\braces*{b_i - V^{\pistarblT}_{c_i}(\rho), 0} &\leq \tilde{\gO}\parens*{\frac{1}{{\sqrt{T} \, \tau}}} \quad \forall i \in [m] \, .
\end{align*}
Setting $\tau = \gO(\eps)$ and $T = \gO(\nicefrac{1}{\eps^4})$ ensures
\begin{align*}
V^{\pi^*}_r(\rho) - V^{\pistarblT}_{r}(\rho) \leq \gO(\eps) \quad \text{ and } \max\braces*{b_i - V^{\pistarblT}_{c_i}(\rho), 0} \leq \gO(\eps) \quad \forall i \in [m] \, .
\end{align*}
Then by~\cref{proposition:transfer_rate}, applying \texttt{NPG} to the pseudo-reward $R(\bar{\lambda}_T)$ recovers a policy $\pi_N$ that is also an $\gO(\eps)$-optimal policy after
$N=\gO\parens*{\ln\parens*{\frac{1}{\eps}}}$ iterations.

Finally, we account for the total number of gradient evaluations. Each of the $T = \gO(\nicefrac{1}{\eps^4})$ outer iterations requires $K = \gO(\ln(\nicefrac{T}{\tau})) = \gO(\ln(\nicefrac{1}{\eps}))$ inner-iterations, and the final policy optimization requires an additional $N = \gO(\ln(1/\eps))$ iterations. Therefore, the total number of gradient evaluations is
\begin{align*}
    TK + N &= \gO\parens*{T \ \ln\parens*{\frac{T}{\tau}}} + \gO\parens*{\ln\parens*{\frac{1}{\eps}}} = \tilde{\gO}\parens*{\frac{1}{\eps^4}}
\end{align*}
\end{proof}

\subsection{Using $K = 1$ Inner-Iterations}\label{appendix:k_one_smooth}
For convenience, we use the following notation in the remaining sections. At each iteration $t \geq 0$, $\tau > 0$, define 
\begin{align*}
    R_t \coloneqq r + \sum_{i=1}^m \lambda_t[i] \, c_i \quad \text{;} \quad
    \pi^*_t \coloneqq \argmax_{\pi} V^{\pi}_{R_t, \tau}(\rho) \quad \text{;} \quad
    Q^{\pi}_{t} \coloneqq Q^{\pi}_{R_t, \tau} \, .
\end{align*}
\biasboundedbyphilemma*
\begin{proof}
\begin{align*}
\norm{\nabla \cD(\lambda_t) - \widetilde{\nabla} \cD(\lambda_t)} &\leq \sqrt{m} \, \norm{\nabla \cD(\lambda_t) - \widetilde{\nabla} \cD(\lambda_t)}_\infty \\
&= \sqrt{m} \max_{i \in [m]}\abs*{V^{\pistart}_{c_i}(\rho) - b_i -(V^{\pi_t}_{c_i}(\rho) - b_i)} \tag{Since $[\bdgt]_i = V^{\pi_t}_{c_i}(\rho) - b_i$} \\
&= \sqrt{m} \max_{i \in [m]}\abs*{V^{\pistart}_{c_i}(\rho) - V^{\pi_t}_{c_i}(\rho)} 
\intertext{By the value difference lemma, for any $i \in [m]$, }
V^{\pistart}_{c_i}(\rho) - V^{\pi_t}_{c_i}(\rho) &= \frac{1}{1 - \gamma} \sum_{s \in \cS} d^{\pi_t}(s) \dpd{\pistart(\cdot | s) - \pi_t(\cdot | s), Q^{\pistart}_{c_i}(s, \cdot)} \\
 &\leq  \frac{1}{1 - \gamma} \sum_{s \in \cS} d^{\pi_t}(s) 
\norm{\pistart(\cdot | s) - \pi_t(\cdot | s)}_1 \supnorm{Q^{\pistart}_{c_i}(s, \cdot)}  \tag{Using H\"older's inequality} \\
 &\leq  \frac{1}{(1 - \gamma)^2} \sum_{s \in \cS} d^{\pi_t}(s) 
\norm{\pistart(\cdot | s) - \pi_t(\cdot | s)}_1 \tag{$c_i \in [0, 1]^{SA}$} \\
\implies \norm{\nabla \cD(\lambda_t) - \widetilde{\nabla} \cD(\lambda_t)} &\leq \frac{\sqrt{m}}{(1 - \gamma)^2} \sum_{s} d^{\pi_t}(s) \, \norm{\pistart(\cdot | s) - \pi_t(\cdot | s)}_1
\end{align*}
Furthermore, using the fact that $e^x$ is $1$-Lipschitz on $(-\infty, 0]$, for any $(s, a) \in \cS \times \cA$,
\begin{align*}
\abs{\pistart(a | s) - \pi_t(a | s)} 
= \abs{\exp(\ln(\pistart(a | s))) - \exp(\ln(\pi_t(a | s)))}
\leq \abs*{\ln(\pistart(a | s)) - \ln(\pi_t(a | s))}
\end{align*}
\begin{align*}
\implies \norm{\nabla \cD(\lambda_t) - \widetilde{\nabla} \cD(\lambda_t)} &\leq \frac{\sqrt{m}}{(1 - \gamma)^2} \sum_{s} d^{\pi_t}(s) \, \norm{\ln(\pistart(\cdot | s)) - \ln(\pi_t(\cdot | s))}_1 \\ 
&\leq \frac{\sqrt{m} \, A}{(1 - \gamma)^2} \sum_{s} d^{\pi_t}(s) \, \supnorm{\ln(\pistart(\cdot | s)) - \ln(\pi_t(\cdot | s))} \tag{Since for $x \in \R^A$ $\norm{x}_1 \leq A \supnorm{x}$}\\
&\leq \frac{\sqrt{m} \, A}{(1 - \gamma)^2} \max_s \supnorm{\ln(\pistart(\cdot | s)) - \ln(\pi_t(\cdot | s))}  \underbrace{\sum_{s} d^{\pi_t}(s)}_{=1} \,  \numberthis \label{eq:relate_to_linear_fa}\\
&= \frac{4 \sqrt{m} \, A}{\tau \, (1 - \gamma)^2} \, \frac{\tau}{4}\supnorm{\ln(\pistart) - \ln(\pi_t)} \\
&\leq \frac{4 \sqrt{m} \, A}{\tau \, (1 - \gamma)^2} \parens*{\underbrace{\frac{1}{1 - \gamma} \, \supnorm{Q^{\pistart}_{t} - Q^{\pi_{t}}_{t}}}_{\geq 0} + \frac{\tau}{4}\supnorm{\ln(\pistart) - \ln(\pi_t)}}
\end{align*}
Defining $\Phi_t \coloneqq \frac{1}{1 - \gamma} \, \supnorm{Q^{\pistart}_{t} - Q^{\pi_{t}}_{t}} + \frac{\tau}{4}\supnorm{\ln(\pistart) - \ln(\pi_t)}$,
\begin{align*}
\implies \norm{\dgt - \bdgt} &\leq \frac{4 \sqrt{m} \, A}{\tau \, (1 - \gamma)^2} \, \Phi_t
\end{align*}
\end{proof}

\phitrackinglemma*
\begin{proof}
At iteration $t+1$, let $\pi_{t+1} =\texttt{NPG}(\pi_t, R_{t+1}, \tau, K)$ using the \texttt{NPG} update in~\cref{eq:entropy_npg_update}.
By~\cref{lemma:entropy_npg_supnorm_convergence} with $\eta = \frac{1 - \gamma}{\tau}$, $\diamond = R_{t+1}$, $\pi_0 = \pi_t$, and $K \geq 1$ returns a policy $\pi_{t+1} = \pi_K$ satisflying
\begin{align}
\supnorm{Q^{\pistartt}_{t+1} - Q^{\pi_{t+1}}_{t+1}} &\leq \gamma^K \, \supnorm{Q^{\pistartt}_{t+1} - Q^{\pi_{t}}_{t+1}} \label{eq:tracking_bound1} \\
\supnorm{\ln(\pi^*_{t+1}) - \ln(\pi_{t+1})} &\leq \frac{2 \, \gamma^{K-1}}{\tau}  \supnorm{Q^{\pistartt}_{t+1} - Q^{\pi_{t}}_{t+1}} \label{eq:tracking_bound2}
\end{align}
The result follows using the above relationships with $E=0$ from~\cref{lemma:generic_phi_recurse}.
\end{proof}

\begin{restatable}{theorem}{smoothddlastitkoneeratetheorem}\label{theorem:smooth_last_iterate_K_1}
Under~\cref{assumption:slater_condition}, for a given $\eps \in (0, 1)$,~\cref{algorithm:dd} with (i) $\tau = \gO(\eps)$, $T = \gO(\nicefrac{1}{\eps^7})$, $\beta = \gO(\eps^4)$, $\mu = 0$, $\lambda_0 = 0$ and (ii) using the exact \texttt{E-NPG} update in~\cref{eq:entropy_npg_update} with $\eta = \frac{1 - \gamma}{\tau}$, $K = 1$, $J = \gO(\ln(\nicefrac{1}{\eps})$ return an $\gO(\eps)$-optimal policy for~\cref{problem:cmdp} after $\widetilde{\gO}\parens*{\nicefrac{1}{\eps^7}}$ gradient evaluations.
\end{restatable}
\begin{proof}
We first characterize the convergence of the average dual iterate. By~\cref{lemma:bias_bounded_by_phi}, for all $t \geq 0$, we can relate the approximate dual gradient to an optimization error term:
\begin{equation*}
\norm{\dgt - \bdgt} \leq \frac{4 \sqrt{m} \, A}{\tau \, (1 - \gamma)^2} \, \Phi_t
\end{equation*}
where $\Phi_t \coloneqq \frac{1}{1 - \gamma} \, \supnorm{Q^{\pistart}_{t} - Q^{\pi_{t}}_{t}} + \frac{\tau}{4} \, \supnorm{\ln(\pistart) - \ln(\pi_t)}$. 
Moreover, by~\cref{lemma:dual_smooth}, the dual $\cD(\lambda)$ is $\frac{m}{\tau \, (1 - \gamma)^3}$-smooth. Hence, under~\cref{assumption:slater_condition}, by~\cref{lemma:meta_lemma_smooth} with $\delta_t = \frac{4 \sqrt{m} \, A}{\tau \, (1 - \gamma)^2} \, \Phi_t$,~\cref{algorithm:dd} using step-size $\beta \leq \frac{\tau (1 - \gamma)^3}{m}$, $\mu =0$ and $\lambda_0 = 0$ results in the following convergence
\begin{align*}
   \cD(\bar{\lambda}_T) - \cD(\lstar) &\leq  \frac{\normsq{\lstar}}{2 \, \beta \, T}  +  \frac{4 \, \sqrt{m} \, A \, \abs{\Lambda}}{\tau \, (1 - \gamma)^2 T} \sum_{t=0}^{T-1} \Phi_t \numberthis \label{eq:smoothness_recursion_phi}
\end{align*}
It remains to bound $\frac{1}{T} \, \sum_{t=0}^{T-1} \Phi_t$. We next derive a recursive bound on the optimization error, relating iteration $t+1$ to iteration $t$ with additional drift term between $\lambda_t$ and $\lambda_{t+1}$.

At iteration $t+1$, instantiating the $\texttt{PolicyOptimization}$ method using the \texttt{NPG} update in~\cref{eq:entropy_npg_update} with $\eta = \frac{1 - \gamma}{\tau}$, $\diamond = R_{t+1}$, $\pi_0 = \pi_t$, and $K=1$. The policy returned $\pi_{t+1} = \pi_K$ thus satisfies,
by~\cref{lemma:phi_tracking_main} with $K=1$
\begin{align}
\implies \Phi_{t+1} \leq  \frac{1+\gamma}{2} \, \Phi_t + \frac{m \, (1+\gamma)}{(1 - \gamma)^2} \supnorm{\lt - \ltt}  \label{eq:tracking_bound9}
\end{align}
Using the Lipschitz property of $\bdgt$, we next bound the drift term $\supnorm{\lt - \ltt}$. 
Recall for any $\lambda \in \R^m_+$, $i \in [m]$,
\begin{align*}
   [\bdg]_i = V^{\tilde{\pi}}_{c_i}(\rho) - b_i &\leq \frac{1}{1 - \gamma} \tag{Since $c_i \in [0, 1]^{SA}$ and $b_i \geq 0$}.\\ 
   \implies \normsq{\bdg}  &\leq \frac{m}{(1 - \gamma)^2} \numberthis \label{eq:bias_gradient_bound_norm_square}
\end{align*}
Hence,
\begin{align*} 
\implies \supnorm{\lt - \ltt} &\leq \norm{\lt - \ltt} \tag{$\supnorm{\cdot} \leq \norm{\cdot}$} \\
&= \norm{\projL[\lt] - \text{Proj}_{\Lambda}[\lt - \beta \, \bdgt]} 
\tag{By the update in~\cref{algorithm:dd} (\cref{line:update}) with $\mu = 0$ and since $\lt \in \Lambda$} \\
&\leq \beta \, \norm{\bdgt} \tag{Since projections are non-expansive and $\lt \in \Lambda$} \\
\implies \supnorm{\lt - \ltt} \leq \norm{\lt - \ltt} &\leq \frac{\sqrt{m}}{1 - \gamma} \, \beta \numberthis \label{eq:lt_ltt_ub} \\
\implies \Phi_{t+1} &\leq  \frac{1+\gamma}{2} \, \Phi_t + \frac{m^{3/2} \, (1+\gamma)}{(1 - \gamma)^3} \, \beta  \numberthis \label{eq:phi_recurse}
\end{align*}
Note that since $\gamma < 1$, $\frac{1+\gamma}{2} < 1$.
Hence, by~\cref{proposition:recursion-const} with $\kappa = \frac{1+\gamma}{2}$, $\nu = \frac{m^{3/2} \, (1+\gamma)}{(1 - \gamma)^3} \, \beta$,
\begin{align*}
    \frac{1}{T} \sum_{t=0}^{T-1} \Phi_t &\leq \frac{2 \, \Phi_0}{(1 - \gamma) \, T} + \frac{2 \, m^{3/2} \, (1+\gamma) \, \beta}{(1 - \gamma)^4}
\end{align*}
Using the above bound with~\cref{eq:smoothness_recursion_phi}, we have
\begin{align*}
   \cD(\bar{\lambda}_T) - \cD(\lstar) &\leq  \frac{\normsq{\lstar}}{2 \, \beta \, T}  +  \frac{4 \, \sqrt{m} \, A \, \abs{\Lambda}}{\tau \, (1 - \gamma)^2 T} \sum_{t=0}^{T-1} \Phi_t  \\
   &\leq \gO\parens*{\frac{1}{\beta \, T}} + \gO\parens*{\frac{1}{\tau \, T}} + \gO\parens*{\frac{\beta}{\tau}}
\end{align*}
Using the above dual sub-optimality guarantee from~\cref{lemma:dual_optimality_to_sub} with $\Delta = 
\gO\parens*{\frac{1}{\beta \, T}} + \gO\parens*{\frac{1}{\tau \, T}} + \gO\parens*{\frac{\beta}{\tau}}$, we have
\begin{align*}
    V^{\pi^*}_r(\rho) - V^{\pistarblT}_r(\rho) &\leq \gO\parens*{\frac{1}{\sqrt{\tau}} \parens*{\frac{1}{\sqrt{\beta \, T}} + \frac{1}{\sqrt{\tau \, T} } + \frac{\sqrt{\beta}}{\sqrt{\tau}}}}
    + \gO(\tau) \\
    &\leq \gO\parens*{\frac{1}{\sqrt{\tau \, \beta \, T}}} + \gO\parens*{\frac{\sqrt{\beta}}{\tau}} + \gO(\tau) \\
   \max\braces*{b_i - V^{\pistarblT}_{c_i}(\rho), 0} &\leq \gO\parens*{\frac{1}{\sqrt{\tau}} \parens*{\frac{1}{\sqrt{\beta \, T}} + \frac{1}{\sqrt{\tau \, T} } + \frac{\sqrt{\beta}}{\sqrt{\tau}}}}
   \quad \forall i \in [m] \,  \\
   &\leq \gO\parens*{\frac{1}{\sqrt{\tau \, \beta \, T}}} 
  + \gO\parens*{\frac{\sqrt{\beta}}{\tau}} 
   \quad \forall i \in [m] \, .
\end{align*}
Hence, setting $\tau = \gO(\eps)$, $\beta = \gO(\eps^4)$ and $T = \gO(\nicefrac{1}{\eps^7})$ ensures
\begin{align*}
V^{\pi^*}_r(\rho) - V^{\pistarblT}_{r}(\rho) \leq \gO(\eps) \quad \text{ and } \max\braces*{b_i - V^{\pistarblT}_{c_i}(\rho), 0} \leq \gO(\eps) \quad \forall i \in [m] \, .
\end{align*}
Then by~\cref{proposition:transfer_rate}, applying \texttt{NPG} update to the pseudo-reward $R(\bar{\lambda}_T)$ recovers a policy $\pi_N$ that is also $\gO(\eps)$-optimal after $N=\gO\parens*{\ln\parens*{\nicefrac{1}{\eps}}}$ iterations.

Finally, we account for the total number of gradient evaluations. Each of the $T = \gO(\nicefrac{1}{\eps^7})$ outer iterations requires $K = 1$ iterations, and the final policy optimization requires an additional $N = \gO(\ln(1/\eps))$ iterations. Therefore, the total number of gradient evaluations is
\begin{align*}
    TK + N &= \gO\parens*{T} + \gO\parens*{\ln\parens*{\frac{1}{\eps}}} = \tilde{\gO}\parens*{\frac{1}{\eps^7}}
\end{align*}

\end{proof}

\subsection{Interpolating between $K=1$ and $K > 1$ Inner-Iterations}\label{appendix:k_geq_one_smooth}
\biasboundedbyLlemma*
\begin{proof}
Follows from~\cref{eq:relate_to_linear_fa} in~\cref{lemma:bias_bounded_by_phi}.
\end{proof}
\lrackinglemma*
\begin{proof}

At iteration $t +1$, let $\pi_{t+1} =\texttt{NPG}(\pi_t, R_{t+1}, \tau, K)$ using the \texttt{NPG} update in~\cref{eq:entropy_npg_update}. 
Furthermore let $\hat{\pi}_0 = \pi_t$ and $\hat{\pi}_k$ denote the intermediate policy for each $k \in [0, K-1]$.
Then, by~\cref{lemma:entropy_npg_supnorm_convergence} with $\eta = \frac{1- \gamma}{\tau}$, $\diamond = R_{t+1}$, $\pi_0 = \hat{\pi}_k$, the sequence of policies satisfy, for all $k \in [0, K-1]$
\begin{equation*}
\supnorm{\ln(\pi^*_{t+1}) - \ln(\hat{\pi}_{k+1})} \leq \frac{2}{\tau} \parens*{\supnorm{Q^{\pi^*_{t+1}}_{t+1} - Q^{\hat{\pi}_k}_{t+1}}}
\end{equation*}
The result then follows from~\cref{lemma:genericLrecurse} with $E = 0$.
\end{proof}

\begin{theorem}\label{theorem:smooth_last_iterate_K_geq_1}
Under~\cref{assumption:slater_condition}, for a given $\eps \in (0, 1)$,~\cref{algorithm:dd} with (i) $\tau = \gO(\eps)$, $T = \tilde{\gO}(\nicefrac{1}{\eps^{4 + 3/2^K}})$, $\beta = \gO\parens*{\eps^\frac{3 + 2^K}{2^K}}$, $\mu = 0$, $\lambda_0 = 0$ (ii) using the exact \texttt{E-NPG} update in~\cref{eq:entropy_npg_update} with $\eta = \frac{1 - \gamma}{\tau}$, $K \geq 1$, $J = \gO(\ln(\nicefrac{1}{\eps}))$ and
(iii) the initial policy satisfies $L_0 \leq \gO(\nicefrac{\beta}{\tau})^{2^K}$~\cref{algorithm:dd} return an $\gO(\eps)$-optimal policy for~\cref{problem:cmdp} after $\widetilde{\gO}(\nicefrac{1}{\eps^{4 + 3/2^K}})$ gradient evaluations.
\end{theorem}

\begin{proof}
This proof proceeds similar to~\cref{theorem:smooth_last_iterate_K_1}.
By~\cref{lemma:bias_bounded_by_L}, for all $t \geq 0$, we can relate the approximate dual gradient to an optimization error term:
\begin{equation*}
\norm{\dgt - \bdgt} \leq \frac{\sqrt{m} \, A}{(1 - \gamma)^2} \, L_t
\end{equation*}
where $L_t \coloneqq \supnorm{\ln(\pistart) - \ln(\pi_t)}$. 
Moreover, by~\cref{lemma:dual_smooth}, the dual $\cD(\lambda)$ is $\frac{m}{\tau \, (1 - \gamma)^3}$-smooth. Hence, under~\cref{assumption:slater_condition}, by~\cref{lemma:meta_lemma_smooth} with $\delta_t = \frac{\sqrt{m} \, A}{(1 - \gamma)^2} \, L_t$,~\cref{algorithm:dd} using step-size $\beta  \leq \frac{\tau (1 - \gamma)^3}{m}$, $\mu =0$ and $\lambda_0 = 0$ results in the following convergence
\begin{align*}
   \cD(\bar{\lambda}_T) - \cD(\lstar) &\leq  \frac{\normsq{\lstar}}{2 \, \beta \, T}  +  \frac{\sqrt{m} \, A \, \abs{\Lambda}}{(1 - \gamma)^2 \, T} \sum_{t=0}^{T-1} L_t  \leq \frac{\normsq{\lstar}}{2 \, \beta \, T} + \frac{\sqrt{m} \, A \, \abs{\Lambda}}{(1 - \gamma)^2} \, \max_{t \in [0, T-1]} L_t
   \numberthis \label{eq:smoothness_recursion_L}
\end{align*}
We next analyze the sequence $\{L_t\}_{t=0}^{T-1}$ in order to derive a uniform bound. 

At iteration $t+1$, let $\pi_{t+1}$ denote the policy returned after running $K$ iterations of the \texttt{NPG} update in~\cref{eq:entropy_npg_update} with $\eta = \frac{1 - \gamma}{\tau}$ and $\diamond = R_{t+1}$ initialized with policy $\pi_0 = \pi_t$.
Then by~\cref{lemma:l_tracking_main} and defining $n \coloneqq 2^K$, $C_1 \coloneqq \frac{\gamma}{1 - \gamma}$,
\begin{align*}
   L_{t+1} \leq  C_1^{n-1} \, \parens*{\frac{2 \sqrt{m}}{\tau \, (1 - \gamma)} \, \norm{\ltt - \lt} + L_t}^n
\end{align*}
Furthermore, using the update in~\cref{line:update} of~\cref{algorithm:dd} with $\mu=0$ since $\lt \in \Lambda$, by~\cref{eq:lt_ltt_ub}
\begin{align*}
    \frac{2 \sqrt{m}}{\tau \, (1 - \gamma)} \, \norm{\lt - \ltt} &\leq \underbrace{\frac{2 \, m }{(1 - \gamma)^2}}_{\coloneqq C_2} \, \frac{\beta}{\tau} \\
    \implies L_{t+1} &\leq C_1^{n - 1} \parens*{\frac{C_2 \, \beta}{\tau} + L_t}^{n} \, .
\end{align*}
Assuming the initial policy $\pi_{0}$ satisfies,
\begin{equation}
    L_0 = \supnorm{\ln(\pi^*_0) - \ln(\pi_{0})} \leq 2^n \, C_1^{n -1} \, \parens*{\frac{C_2 \, \beta}{\tau}}^n \label{eq:warm-start-pi0}
\end{equation}
and the step-size satisfies 
\begin{equation*}
 \beta \leq \frac{\tau \, (1 - \gamma)^3}{8 \, m \, \gamma} = \frac{\tau}{4 \, C_1 \, C_2}    \implies \frac{C_1 \, C_2 \, \beta}{\tau} \leq \frac{1}{4}
\end{equation*}
we have,
\begin{align*}
2^n \, \parens*{\frac{C_1 \, C_2 \, \beta}{\tau}}^{n-1} \leq 2^n \, 4^{-(n-1)} \leq 2^{2 - n} = 2^{2 - 2^K} \leq 1 \tag{Since $n=2^K \geq 2$}.
\end{align*}
Hence, by~\cref{lemma:quadratic_recursion} with $\kappa = C_1^{n - 1}$, $\nu = \frac{C_2 \, \beta}{\tau}$ and $e= 0$ we have, for all $t \geq 0$
\begin{equation*}
    L_t \leq 2^n \, C_1^{n-1} \, \parens*{\frac{C_2 \, \beta}{\tau}}^n = \gO\parens*{\parens*{\frac{\beta}{\tau}}^n}
\end{equation*}
Using the above inequality with~\cref{eq:smoothness_recursion_L},
\begin{align*}
   \cD(\bar{\lambda}_T) - \cD(\lstar) &\leq  \frac{\normsq{\lstar}}{2 \, \beta \, T}  +  \frac{\sqrt{m} \, A \, \abs{\Lambda}}{(1 - \gamma)^2 \, T} \sum_{t=0}^{T-1} L_t \leq \gO\parens*{\frac{1}{\beta \, T} + \parens*{\frac{\beta}{\tau}}^n} \, 
\end{align*}
Combining the above result with~\cref{lemma:dual_optimality_to_sub} with $\Delta = \gO\parens*{\frac{1}{\beta \, T} + \parens*{\frac{\beta}{\tau}}^n}$
\begin{align*}
V^{\pi^*}_r(\rho) - V^{\pistarblT}_{r}(\rho) \leq \gO\parens*{\frac{1}{\sqrt{\tau}} \, \parens*{\frac{1}{\sqrt{\beta \, T}} + \parens*{\frac{\beta}{\tau}}^{n/2}}} + \gO(\tau) 
\\ 
\max\braces*{b_i - V^{\pistarblT}_{c_i}(\rho), 0} \leq \gO\parens*{\frac{1}{\sqrt{\tau}} \, \parens*{\frac{1}{\sqrt{\beta \, T}} + \parens*{\frac{\beta}{\tau}}^{n/2}}} \quad \forall i \in [m] \, .
\end{align*}

To obtain an $\gO(\eps)$-optimal policy, we require each term to be $\gO(\eps)$. Setting $\tau = \gO(\eps)$, it suffices to set the step-size such that
\begin{align*}
    \gO\parens*{\frac{1}{\sqrt{\eps}} \, \frac{\beta^{n/2}}{\eps^{n/2}}} &\leq \gO(\eps) 
    \implies \beta \leq \gO\parens*{\eps^\frac{3 + n}{n}}
\intertext{Next, to determine the number of outer-iterations $T$,}
    \gO\parens*{\frac{1}{\sqrt{\eps}} \, \parens*{\frac{1}{\sqrt{\beta \, T}}}} &\leq \gO(\eps) 
    \implies \gO\parens*{\frac{1}{\sqrt{\beta \, T}}} \leq \gO\parens*{\eps^{3/2}} \implies T \geq \gO\parens*{\frac{1}{\beta \, \eps^3}} = 
    \gO\parens*{\frac{1}{\eps^{\frac{4n + 3}{n}}}}  \,.
\end{align*}
Since $n = 2^K$, further setting $\beta = \gO\parens*{\eps^\frac{3 + 2^K}{2^K}}$, $T = \gO\parens*{\frac{1}{\eps^{4 + \frac{3}{2^K}}}}$, we have
\begin{align*}
V^{\pi^*}_r(\rho) - V^{\pistarblT}_{r}(\rho) \leq \gO(\eps) \quad \text{ and } \max\braces*{b_i - V^{\pistarblT}_{c_i}(\rho), 0} \leq \gO(\eps) \quad \forall i \in [m] \, .
\end{align*}
Then by~\cref{proposition:transfer_rate}, applying \texttt{NPG} update to the pseudo-reward $R(\bar{\lambda}_T)$ recovers a policy $\pi_N$ that is also $\gO(\eps)$-optimal after $N=\gO\parens*{\ln\parens*{\nicefrac{1}{\eps}}}$ iterations.

Finally, we account for the total number of gradient evaluations.   Each of the $T = \gO\parens*{\frac{1}{\eps^{4 + \frac{3}{2^K}}}}$ outer iterations requires $K \geq 1$ iterations.
For the first iteration only, in order to satisfy~\cref{eq:warm-start-pi0}, we require
\begin{align*}
 L_0 &\leq 2^n \, C_1^{n-1} \, \parens*{\frac{C_2 \, \beta}{\tau}}^n \leq \gO(\eps^3)
\end{align*}
By~\cref{lemma:entropy_npg_supnorm_convergence} with $\diamond = R(\lambda_0)$ and $\eta = \frac{1 - \gamma}{\tau}$ ensures
$L_0 = \supnorm{\ln(\pi^*_0) - \ln(\pi_0)} \leq \Theta(\eps^3)$ after $K_0 = \gO(\ln(\nicefrac{1}{\eps}))$ initial iterations. 
Finally, to recover a deployable policy  the final policy optimization requires an additional $N = \gO(\ln(1/\eps))$ iterations. Therefore, the total number of gradient evaluations is
$\tilde{\gO}\parens*{\frac{1}{\eps^{4 + \frac{3}{2^K}}}}$.
\end{proof}

\section{Faster Convergence Rates Using $\ell_2$ Regularization}
Throughout this section we analyze~\cref{algorithm:dd} with dual regularization $\ell_2$ parameter $\mu > 0$.
Recall that the dual update in~\cref{algorithm:dd} (\cref{line:update}), 
$\ltt = \projL\bracks*{(1 - \mu \, \betat) \, \lt - \betat \, \bdgt}$, is precisely the projected gradient descent with inexact gradients on the \emph{regularized dual}, i.e.,
\begin{equation}\label{eq:reg_dual}
    \F(\lambda) \coloneqq \cD(\lambda) + \frac{\mu}{2} \, \normsq{\lambda} \quad \text{;} \quad
    [\nabla \F(\lambda)]_i = \underbrace{V^{\pistarl}_{c_i}(\rho) - b_i}_{= [\dg]_i} + \mu \, \lambda_i \quad \forall i \in [m] \, ,
\end{equation}
and let $\bFgt \coloneqq \bdgt + \mu \, \lt$. We denote its minimizer by
\begin{equation}\label{eq:reg_dual_argmin}
   \lstarmu \coloneqq \argmin_{\lambda \in \R^m_+} \F(\lambda) \, ,
\end{equation}
and since $\cD(\lambda)$ is convex, $\F(\lambda)$ is $\mu$-strongly convex. 
Moreover, by the smoothness of $\cD(\lambda)$ established in~\cref{lemma:dual_smooth}, $\F$ is $\frac{m}{\tau \, (1 - \gamma)^3} + \mu$-smooth.
Before proceeding with the main theorem, we present a couple of technical tools. 
\begin{lemma}\label{lemma:reg_optimal_dual_bounded}
Under~\cref{assumption:slater_condition}, for any $\mu > 0$ and $\tau > 0$, the regularized dual  $\lstarmu = \argmin_{\lambda \in \R^m_{+}} \F(\lambda)$ satisfies
\begin{align*}
    [\lstarmu]_i \leq \frac{2 + 2 \, \tau \, \ln(A)}{(1 - \gamma) \, \zeta_i} \quad \forall i \in [m]
\end{align*}
where $\zeta_i = V^{\bar{\pi}}_{c_i}(\rho) - b_i > 0$ and $\bar{\pi}$ is the policy that satisfies Slater's condition. Consequently, for any $\lambda \in \Lambda$,
\begin{equation*}
    \norm{\lambda - \lstarmu} \leq \frac{\sqrt{m} \, (2 + 2 \tau \, \ln(A))}{(1 - \gamma) \, \min_{i \in [m]} \, \zeta_i}
\end{equation*}
\end{lemma}
\begin{proof}
Since $\lstarmu = \argmin_{\lambda \in \R^m_{+}} \F(\lambda)$, we have
\begin{align*}
   \F(\lstarmu) &\leq \F(0) = \cD(0) = \max_{\pi} \lag(\pi, 0) = \max_{\pi} V^{\pi}_{r, \tau}(\rho) \numberthis \label{eq:reg_dual_ub}
\end{align*}
Furthermore,
\begin{align*}
   \F(\lstarmu) &= \cD(\lstarmu) + \frac{\mu}{2} \, \normsq{\lstarmu} \geq \cD(\lstarmu) \tag{Since $\frac{\mu}{2} \normsq{\lstarmu} \geq 0$} \\
   &= \max_\pi \lag(\pi, \lstarmu)  \geq \lag(\bar{\pi}, \lstarmu) 
  \tag{Under~\cref{assumption:slater_condition}, there exist a policy $\bar{\pi}$ such that $\zeta_i = V^{\bar{\pi}}_{c_i}(\rho) - b_i  > 0$ for all $i \in [m]$} 
   \\
   &= V^{\bar{\pi}}_{r, \tau}(\rho) + \sum_{i=1}^m [\lstarmu]_i \, \parens*{V^{\bar{\pi}}_{c_i}(\rho) - b_i} \tag{Definition of $\lag(\pi, \lambda)$} \\
   &= V^{\bar{\pi}}_{r, \tau}(\rho) + \sum_{i=1}^m [\lstarmu]_i \, \zeta_i  \\
   \implies \sum_{i=1}^m [\lstarmu]_i \, \zeta_i&\leq \F(\lstarmu) - V^{\bar{\pi}}_{r, \tau}(\rho) \numberthis \label{eq:reg_dual_lb}
\end{align*}
Combining~\cref{eq:reg_dual_ub,eq:reg_dual_lb} and noting that since $\lstarmu \in \R^m_+$ and $\zeta_i > 0$ for all $i \in [m]$, every term of the sum is non-negative, and hence for each $i \in [m]$,
\begin{align*}
   [\lstarmu]_i &\leq \frac{\max_\pi V^{\pi}_{r, \tau}(\rho) - V^{\bar{\pi}}_{r, \tau}(\rho)}{\zeta_i} \leq \frac{2 + 2 \, \tau \, \ln(A)}{(1 - \gamma) \, \zeta_i} \tag{Since $V^{\pi}_{r, \tau}(\rho) \in \bracks*{0, \frac{2 + 2 \, \tau \, \ln(A)}{1 - \gamma}}$} \, .
\end{align*}
Moreover, since for all $i \in [m]$, $\lambda_i \in \bracks*{0, \frac{2 + 2 \, \tau \, \ln(A)}{(1 - \gamma) \, \min_{i \in [m]} \, \zeta_i}}$
\begin{align*}
    \supnorm{\lambda - \lstarmu} &\leq \frac{2 + 2 \, \tau \, \ln(A)}{(1 - \gamma) \, \min_{i \in [m]} \, \zeta_i}  \implies \norm{\lambda - \lstarmu} \leq \frac{\sqrt{m} \, (2 + 2 \, \tau \, \ln(A))}{(1 - \gamma) \, \min_{i \in [m]} \, \zeta_i}
\end{align*}
\end{proof}

\begin{lemma}\label{lemma:reg_dual_kkt}
For any $\mu > 0$ and $\tau > 0$, the regularized dual solution $\lstarmu \coloneqq \argmin_{\lambda \in \R^m_+} \F(\lambda)$ satisfies, for all $i \in [m]$,
\begin{align*}
   [\nabla \F(\lstarmu)]_i &\geq 0  \numberthis \label{eq:reg_kkt} \\ 
   [\lstarmu]_i \, [\nabla \F(\lstarmu)]_i &= 0 \numberthis \label{eq:reg_kkt_2} \\
    \max\braces*{b_i - V^{\pistarlmu}_{c_i}(\rho), 0} &= \mu \, [\lstarmu]_i  \numberthis \label{eq:reg_bias_identity} \\
    \dpd{\lstarmu, V^{\pistarlmu}_{c}(\rho) - b} &= - \mu \, \normsq{\lstarmu} \leq 0 \numberthis \label{eq:reg_complementary}
\end{align*}
where $\pistarlmu \coloneqq \argmax_{\pi} \lag(\pi, \lstarmu)$.
\end{lemma}
\begin{proof}
To prove the first and second identities, we consider the Lagrangian to the optimization problem in~\cref{eq:reg_dual_argmin}
\begin{equation*}
   \mathscr{L}(\lambda, \nu) \coloneqq \F(\lambda) - \sum_{i=1}^m \nu_i \, \lambda_i \, = \cD(\lambda) + \frac{\mu}{2}\, \normsq{\lambda} - \sum_{i=1}^m \nu_i \, \lambda_i.
\end{equation*}
where $\nu_i \geq 0$ is the Lagrangian multiplier for each $i \in [m]$.
Since  $\cD(\lambda)$ is convex and differentiable and the constraints are linear, the KKT conditions are both necessary and sufficient for optimality
of $\lstarmu$~\citep{boyd2004convex}. Therefore there exists $\nu^* \in \R^m$ such that
\begin{align*}
   \nabla_\lambda \mathscr{L}(\lstarmu, \nu^*) = \nabla \F(\lstarmu) - \nu^* &= 0 \tag{Stationarity} \\
   \nu_i^* &\geq 0 \quad \forall i \in [m] \tag{Dual feasibility} \\
   [\lstarmu]_i &\geq 0 \quad \forall i \in [m] \tag{Primal feasibility} \\
   \nu_i^* \, [\lstarmu]_i &= 0 \quad \forall i \in [m] \tag{Complementary slackness}
\end{align*}
Substituting $[\nu^*]_i = [\nabla \F(\lstarmu)]_i$ from stationarity into dual feasibility and  complementary slackness gives, for all $i \in [m]$,
\begin{equation*}
   [\nabla \F(\lstarmu)]_i \geq 0 \quad \text{;} \quad [\lstarmu]_i \, [\nabla \F(\lstarmu)]_i = 0 \, ,
\end{equation*}

To prove the third and forth inequalities, we consider the cases where for a fixed $i \in [m]$,
$[\lstarmu]_i = 0$ or $[\lstarmu]_i > 0$. If $[\lstarmu]_i = 0$, then by~\cref{eq:reg_kkt}
\begin{align*}
    0 \leq [\nabla \F(\lstarmu)]_i = V^{\pistarlmu}_{c_i}(\rho) - b_i + \mu \, \underbrace{[\lstarmu]_i}_{=0} \implies V^{\pistarlmu}_{c_i}(\rho) \geq b_i \, .
\end{align*}
\begin{equation*}
\implies \max\braces*{b_i - V^{\pistarlmu}_{c_i}(\rho), 0} = 0 = \mu \, [\lstarmu]_i    \, .
\end{equation*}
Otherwise if $[\lstarmu]_i > 0$,  by~\cref{eq:reg_kkt_2},
\begin{align*}
    V^{\pistarlmu}_{c_i}(\rho) - b_i + \mu \, [\lstarmu]_i = 0
    \implies b_i - V^{\pistarlmu}_{c_i}(\rho) = \mu \, [\lstarmu]_i > 0
\end{align*}
\begin{equation*}
\implies \max\braces*{b_i - V^{\pistarlmu}_{c_i}(\rho), 0} = \mu \, [\lstarmu]_i  \, .  
\end{equation*}
Combining both cases proves~\cref{eq:reg_bias_identity}.
Finally, multiplying $[\lstarmu]_i \, [\nabla \F(\lstarmu)]_i$ we have
\begin{equation*}
   [\lstarmu]_i \, [\nabla \F(\lstarmu)]_i = 0  \implies [\lstarmu]_i \, (V^{\pistarlmu}_{c_i}(\rho) - b_i) = - \mu \, [\lstarmu]^2_i \quad \forall i \in [m]
\end{equation*}
and summing over each $i \in [m]$  gives~\cref{eq:reg_complementary}.
\end{proof}

\regdualdisttosublemma*
\begin{proof}
Let $\mu > 0$, $\tau > 0$ and $\lambda \in \Lambda$.
We first consider the reward sub-optimality gap.
\begin{align*}
\MoveEqLeft V^{\pistarl}_{r}(\rho) + \tau \, \gH(\rho, \pistarl) + \sum_{i=1}^m \lambda_i \, [V^{\pistarl}_{c_i}(\rho) - b_i] = \lag(\pistarl, \lambda) \tag{$\pistarl \coloneqq \argmax_{\pi} \lag(\pi, \lambda)$}\\ 
   &\geq \lag(\pi^*, \lambda) = V^{\pi^*}_{r}(\rho) + \tau \, \gH(\rho, \pi^*) + \sum_{i=1}^m \lambda_i \, \underbrace{ \, [V^{\pi^*}_{c}(\rho) - b_i]}_{\geq 0} \\
   \tag{$\pi^*$ is the optimal policy for~\cref{problem:cmdp}} \\
   &\geq V^{\pi^*}_{r}(\rho) \tag{Since $\lambda \in \R^m_+$, and entropy is non-negative}
\end{align*}
Rearranging the above inequality and since $\gH(\rho, \pi) \leq \frac{\ln(A)}{1 - \gamma}$,
\begin{align*}
   \implies V^{\pi^*}_{r}(\rho) - V^{\pistarl}_{r}(\rho) &\leq \frac{\ln(A)}{1 - \gamma} \, \tau + 
\sum_{i=1}^m \lambda_i \, [V^{\pistarl}_{c_i}(\rho) - b_i] \numberthis \label{eq:reg_sub_opt_start} \, .
\end{align*}
We decompose the remaining sum as
\begin{align*}
    &\sum_{i=1}^m \lambda_i \, [V^{\pistarl}_{c_i}(\rho) - b_i] \\
    &= 
    \sum_{i=1}^m [\lambda - \lstarmu + \lstarmu]_i \, [V^{\pistarl}_{c_i}(\rho) - b_i] \tag{Add/Subtract $\lstarmu$} \\
    &=\sum_{i=1}^m [\lambda - \lstarmu]_i \, [V^{\pistarl}_{c_i}(\rho) - b_i] + 
 \sum_{i=1}^m [\lstarmu]_i \, [V^{\pistarl}_{c_i}(\rho) - b_i] \\  
    &=\underbrace{\sum_{i=1}^m [\lambda - \lstarmu]_i \, [V^{\pistarl}_{c_i}(\rho) - b_i]}_{\text{Term (i)}} + \underbrace{\sum_{i=1}^m[\lstarmu]_i \, [V^{\pistarl}_{c_i}(\rho) - V^{\pistarlmu}_{c_i}(\rho)]}_{\text{Term (ii)}} + \underbrace{\sum_{i=1}^m[\lstarmu]_i \, [V^{\pistarlmu}_{c_i}(\rho) - b_i]}_{\text{Term (iii)}}
    \tag{Add/Subtract $V^{\pistarlmu}_{c_i}(\rho)$}
\end{align*}

For Term (i), since $c_i \in [0, 1]^{SA}$ and $b_i \geq 0$ we have $\abs*{V^{\pistarl}_{c_i}(\rho) - b_i} \leq \frac{1}{1 - \gamma}$ for each $i \in [m]$, and hence using Cauchy-Schwarz,
\begin{align*}
   \sum_{i=1}^m [\lambda - \lstarmu]_i \, [V^{\pistarl}_{c_i}(\rho) - b_i] \leq \norm{\lambda - \lstarmu} \, \norm{V^{\pistarl}_{c}(\rho) - b} \leq \frac{\sqrt{m}}{1 - \gamma} \, \norm{\lambda - \lstarmu}
\end{align*}
For Term (ii), by~\cref{proposition:dual_gradient}, $[\nabla \cD(\lambda)]_i = V^{\pistarl}_{c_i}(\rho) - b_i$ and using Cauchy-Schwarz we have ,
\begin{align*}
    \sum_{i=1}^m[\lstarmu]_i \, [V^{\pistarl}_{c_i}(\rho) - V^{\pistarlmu}_{c_i}(\rho)] 
    &\leq \norm{\lstarmu} \, \norm{\dg - \nabla \cD(\lstarmu)} \\
    &\leq \norm{\lstarmu} \, G_\tau \, \norm{\lambda - \lstarmu} \tag{By~\cref{lemma:dual_smooth}}\\
    &\leq \sqrt{m} \, \supnorm{\lstarmu} \, G_\tau \, \norm{\lambda - \lstarmu} \tag{Since $\norm{x} \leq \sqrt{m} \, \supnorm{x}$ for $x \in \R^m$}\\
    &\leq \frac{\sqrt{m} (2 + 2 \, \tau \, \ln(A))}{(1 - \gamma) \min_{i \in [m]} \zeta_i} \, G_\tau \, \norm{\lambda - \lstarmu} \tag{By~\cref{lemma:reg_optimal_dual_bounded}}
\end{align*}
Finally, Term (iii) is non-positive by~\cref{eq:reg_complementary} of~\cref{lemma:reg_dual_kkt}. 
Substituting the three bounds into~\cref{eq:reg_sub_opt_start} results in
\begin{align*}
   V^{\pi^*}_{r}(\rho) - V^{\pistarl}_{r}(\rho) &\leq \frac{\ln(A)}{1 - \gamma} \, \tau + \parens*{\frac{\sqrt{m}}{1 - \gamma} + \frac{\sqrt{m} \, (2 + 2 \, \tau \, \ln(A))}{(1 - \gamma) \min_{i \in [m]} \, \zeta_i} \, G_\tau} \,  \norm{\lambda - \lstarmu} \\
   &\leq \gO(G_\tau \, \norm{\lambda - \lstarmu}) + \gO(\tau)
\end{align*}
We next consider the constraint violations for a fixed constraint $i \in [m]$. If the policy $\pistarl = \argmax_{\pi} \lag(\pi, \lambda)$ satisfies
$V^{\pistarl}_{c_i}(\rho) \geq b_i$
then
\begin{align*}
b_i - V^{\pistarl}_{c_i}(\rho) &\leq \max\braces*{b_i - V^{\pistarl}_{c_i}(\rho), 0} = 0 \\ &\leq \max\braces*{b_i - V^{\pistarlmu}_{c_i}(\rho), 0}  \tag{$\pistarlmu \coloneqq \argmax_{\pi} \lag(\pi, \lstarmu)$}\\
&\leq \max\braces*{b_i - V^{\pistarlmu}_{c_i}(\rho), 0} + \norm{\nabla \cD(\lstarmu) - \dg}_1
\end{align*}
Otherwise if $V^{\pistarl}_{c_i}(\rho) < b_i$ , we have
\begin{align*}
    b_i - V^{\pistarl}_{c_i}(\rho)
    &= \bracks*{b_i - V^{\pistarlmu}_{c_i}(\rho)} + \bracks*{V^{\pistarlmu}_{c_i}(\rho) - V^{\pistarl}_{c_i}(\rho)} \tag{Add/subtract $V^{\pistarlmu}_{c_i}(\rho)$} \\
    &\leq \max\braces*{b_i - V^{\pistarlmu}_{c_i}(\rho), 0} + \abs*{V^{\pistarlmu}_{c_i}(\rho) - V^{\pistarl}_{c_i}(\rho)} \\
    &= \max\braces*{b_i - V^{\pistarlmu}_{c_i}(\rho), 0}  + \abs*{[\nabla \cD(\lstarmu)]_i - [\dg]_i} \tag{By~\cref{proposition:dual_gradient}, $[\dg]_i = V^{\pistarl}_{c_i}(\rho) - b_i$} \\
    &\leq \max\braces*{b_i - V^{\pistarlmu}_{c_i}(\rho), 0} + \norm{\nabla \cD(\lstarmu) - \dg}_1
\end{align*}
Hence in either case we have,
\begin{align*}
   \max\braces*{b_i - V^{\pistarl}_{c_i}(\rho), 0}  &\leq \max\braces*{b_i - V^{\pistarlmu}_{c_i}(\rho), 0} + \norm{\nabla \cD(\lstarmu) - \dg}_1  \\
   &\leq \max\braces*{b_i - V^{\pistarlmu}_{c_i}(\rho), 0} + \sqrt{m} \, \norm{\nabla \cD(\lstarmu) - \dg}  \tag{Since for $x \in \R^m$, $\norm{x}_1 \leq \sqrt{m} \, \norm{x}$}\\
    &= \mu \, [\lstarmu]_i + \sqrt{m} \, \norm{\nabla \cD(\lstarmu) - \dg}  \tag{By~\cref{eq:reg_bias_identity} of~\cref{lemma:reg_dual_kkt}} \\
    &\leq \mu \, [\lstarmu]_i + \sqrt{m} \, G_\tau \, \norm{\lambda - \lstarmu} \tag{By~\cref{lemma:dual_smooth} where $G_\tau \coloneqq \frac{m}{\tau (1 - \gamma)^3}$}\\
    &\leq \frac{\mu \, (2 + 2 \, \tau \, \ln(A))}{(1 - \gamma) \, \zeta_i} + \sqrt{m} \, G_\tau \, \norm{\lambda - \lstarmu} \tag{Under~\cref{assumption:slater_condition}, by~\cref{lemma:reg_optimal_dual_bounded}} \\
    &\leq \gO(G_\tau \, \norm{\lambda - \lstarmu}) + \gO(\mu)
\end{align*}
which completes the proof.
\end{proof}

\regduallinearconvergencelemma*
\begin{proof}
For this proof, it will be useful to define the following update using the exact gradient
\begin{align*}
    \ltt' &= \projL[\lt - \beta \, \nabla \F(\lt)] = \projL[\lt - \beta \, (\dgt + \mu \, \lt)]
\end{align*}
Starting for a fixed $t \geq 0$,
\begin{align*}
   \norm{\ltt - \lambda^*_{\tau, \mu}} &=\norm{\ltt - \ltt' + \ltt' - \lambda^*_{\tau, \mu}}
   \tag{Add/Subtract $\ltt'$} \\
   &\leq  \norm{\ltt' - \lambda^*_{\tau, \mu}} + \norm{\ltt - \ltt'}  \tag{Using Triangle inequality}
\end{align*}
To bound the first term, since $\lstarmu = \projL[\lstarmu - \beta \, \nabla \F(\lstarmu)]$ for any $\beta > 0$, we have
\begin{align*}
   &\normsq{\ltt' - \lambda^*_{\tau, \mu}}  \\
   &= \normsq{\projL[\lt - \beta \, \nabla \F(\lt)] - \projL[\lstarmu - \beta \, \nabla \F(\lstarmu)]} \tag{$\lambda^*_{\tau, \mu}$ is a fixed-point} \\
   &\leq \normsq{\lt - \beta \, \nabla \F(\lt) -\parens*{\lstarmu - \beta \, \nabla \F(\lstarmu)}} \tag{$\projL$ is non-expansive} \\
   &= \normsq{\lt - \lstarmu} - 2 \, \beta \, \dpd{\nabla \F(\lt) - \nabla \F(\lstarmu), \lt - \lstarmu} + \beta^2 \, \normsq{\nabla \F(\lt) - \nabla \F(\lstarmu)}
\end{align*}
To lower-bound the inner product, by~\cref{lemma:dual_smooth}, $\cD(\lambda)$ is $\frac{m}{\tau \, (1 -\gamma)^3}$-smooth and hence $\F(\lambda) = \cD(\lambda) + \frac{\mu}{2} \, \normsq{\lambda}$ is $L_{\tau, \mu} \coloneqq \frac{m}{\tau \, (1 -\gamma)^3} + \mu$-smooth. Additionally $\F$ is $\mu$-strongly convex.
Hence, by~\cite[Theorem 2.1.12]{nesterov2013introductory}, 
$\F$ satisfies for any $x, y \in \Lambda$
\begin{align*}
\dpd{\nabla \F(x) - \nabla \F(y), x - y} &\geq \frac{\mu \, L_{\tau, \mu}}{\mu + L_{\tau, \mu}} \normsq{x - y} + \frac{1}{\mu + L_{\tau, \mu}} \normsq{\nabla \F(x) - \nabla \F(y)}  \, .
\end{align*}
Using the above inequality with $x = \lt$ and $y=\lambda^*_{\tau, \mu}$
\begin{align*}
    \implies \normsq{\ltt' - \lambda^*_{\tau, \mu}} \leq \parens*{1 - \frac{2 \, \beta \, \mu \, L_{\tau, \mu}}{\mu + L_{\tau, \mu}}}\normsq{\lt - \lambda^*_{\tau, \mu}}  + \beta \, \parens*{\beta - \frac{2}{\mu + L_{\tau, \mu}}} \, \normsq{\nabla \F(\lt) - \nabla \F(\lstarmu)}
\end{align*}
\begin{align*}
\implies \norm{\ltt' - \lambda^*_{\tau, \mu}}&\leq
\sqrt{\parens*{1 - \frac{2 \, \beta \, \mu \, L_{\tau, \mu}}{\mu + L_{\tau, \mu}}}} \, \norm{\lt - \lambda^*_{\tau, \mu}} \tag{Since $\beta \leq \frac{2}{\mu + L_{\tau, \mu}}$, the second term is non-positive}
\end{align*}
Furthermore, since $L_{\tau, \mu} \geq \mu$
and $4\mu L_{\tau,\mu} \leq (\mu + L_{\tau,\mu})^2$,
we have
\begin{equation*}
    0 < \frac{\beta \, \mu}{2} \leq \frac{\beta \, \mu \, L_{\tau, \mu}}{\mu + L_{\tau, \mu}} \leq \frac{2 \, \mu \, L_{\tau, \mu}}{(\mu + L_{\tau, \mu})^2} \leq \frac{1}{2}
\end{equation*}
Hence, using the fact that $\sqrt{1 - 2x} \leq 1 -x$ for $x \in [0, \nicefrac{1}{2}]$ with $x = \frac{\beta \, \mu \, L_{\tau, \mu}}{\mu + L_{\tau, \mu}}$
\begin{align*}
\implies \norm{\ltt' - \lambda^*_{\tau, \mu}} &\leq
\parens*{1 - \frac{\beta \, \mu \, L_{\tau, \mu}}{\mu + L_{\tau, \mu}}} \, \norm{\lt - \lambda^*_{\tau, \mu}} 
\leq \parens*{1 - \frac{\beta \, \mu}{2}} \, \norm{\lt - \lambda^*_{\tau, \mu}} \numberthis \label{linear_rec_1}
\end{align*}
To bound the second term,
\begin{align*}
   \norm{\ltt - \ltt'}  &= \norm{\projL[\lt - \beta \, \bFgt] - \projL[\lt - \beta \, \nabla \F(\lt)]} \tag{Using the updates of $\ltt'$ and $\ltt$} \\
  &\leq \beta \, \norm{\bFgt - \nabla \F(\lt)}  \tag{Since $\projL$ is non-expansive} \\
  &= \beta \, \norm{\bdgt - \dgt} \tag{The $\mu \, \lt$ terms cancel}
  \intertext{Assuming that $\norm{\bdgt - \dgt} \leq \delta_t$ for all $t \in [T]$}
  \implies \norm{\ltt - \ltt'}  &\leq \beta \, \delta_t \numberthis \label{linear_rec_2}
\end{align*}
Putting~\cref{linear_rec_1,linear_rec_2} together, since $\beta \leq \frac{2}{\mu + L_{\tau, \mu}}$
\begin{equation*}
   \implies \norm{\ltt - \lstarmu} \leq \parens*{1 - \frac{\beta \, \mu}{2}} \, \norm{\lt - \lambda^*_{\tau, \mu}} + \beta \, \delta_t \numberthis \label{eq:reg_recurse}
\end{equation*}
Using~\cref{eq:reg_recurse} and recursing from $t = T-1$ to $0$,
\begin{align*}
    \norm{\lambda_T - \lstarmu} &\leq \parens*{1 - \frac{\beta \, \mu}{2}}^T \, \norm{\lambda_0 - \lstarmu} + \beta \sum_{t=0}^{T-1} \parens*{1 - \frac{\beta \, \mu}{2}}^{T - 1 - t} \, \delta_t \\
    &\leq \exp\parens*{-\frac{\beta \, \mu \, T}{2}} \, \norm{\lstarmu} + \beta \sum_{t=0}^{T-1} \parens*{1 - \frac{\beta \, \mu}{2}}^{T - t - 1} \, \delta_t \, . \tag{Since $1 - x \leq \exp(-x)$ and $\lambda_0 = 0$}
\end{align*}
\end{proof}

\subsection{Using $K = \gO(\ln(\nicefrac{1}{\eps}))$ Inner-Iterations}\label{appendix:k_ln_eps_smooth_reg}
\regsmoothnessltworate*
\begin{proof}
Following the proof of~\cref{theorem:meta_smooth}, we construct a uniform bound of the dual gradient error.
By~\cref{lemma:bound_bias_subgradient}, for all $t \geq 0$, we can relate the approximate gradient to the following optimization term:
\begin{equation*}
   \norm{\dgt - \bdgt} \leq  \frac{\sqrt{ 2 \, m }}{(1 - \gamma)^{\nicefrac{3}{2}} \, \sqrt{\tau}} \sqrt{ \max_{\pi} V^{\pi}_{R(\lambda_t), \tau}(\rho) - V^{\pi_{t}}_{R(\lambda_t), \tau}(\rho)}
\end{equation*}
To characterize the sub-optimality gap $\max_{\pi} V^{\pi}_{R(\lambda_t), \tau}(\rho) - V^{\pi_{t}}_{R(\lambda_t), \tau}(\rho)$, we instantiate the \texttt{PolicyOptimization} method in~\cref{algorithm:dd} using \texttt{NPG} update in~\cref{eq:entropy_npg_update}.
At iteration $t$, by~\cref{lemma:entropy_npg_supnorm_convergence}, with $\diamond = R(\lambda_t)$, $\pi_0 = \pi_{t-1}$, $\eta = \frac{1 - \gamma}{w \, \tau}$ and $w \geq 1$, the policy $\pi_{t} = \pi_{K}$ satisfies
\begin{align*}
    \max_{\pi} V^{\pi}_{R(\lambda_t), \tau}(\rho) - V^{\pi_{t}}_{R(\lambda_t), \tau}(\rho) \leq \gO\parens*{\frac{\tau^{3} \, \mu^2}{T^4}} \, .
\end{align*}
after
\begin{equation}\label{eq:npg_updates_2}
   K \geq \gO\parens*{\frac{w}{1 - \gamma}\ln\parens*{\frac{T^4}{\tau^3\, \mu^2}}} 
\end{equation}
iterations.
Hence, for all $t \in [T]$
\begin{equation*}
   \norm{\dgt - \bdgt} \leq  \gO\parens*{\frac{\tau \, \mu}{T^2}} \coloneqq \delta_t.
\end{equation*}
and let $\hat{\delta} \coloneqq \max_{t \in [T]} \delta_t$.

We now characterize the dual convergence using the above uniform gradient error. 
By~\cref{lemma:dual_smooth}, $\cD(\lambda)$ is $G_\tau = \frac{m}{\tau \, (1 - \gamma)^3}$-smooth. Hence, under~\cref{assumption:slater_condition}, by~\cref{lemma:reg_dual_linear_convergence},~\cref{algorithm:dd} using step-size $\beta = \frac{2}{\mu + L_{\tau, \mu}} = \frac{2 \, \tau \, (1 - \gamma)^3}{m + 2 \, \mu \, \tau \, (1 - \gamma)^3}$, and $\lambda_0 = 0$, results in the following convergence: for $T \geq 1$,    
\begin{align*}
    \norm{\lambda_T - \lstarmu} &\leq \exp\parens*{-\frac{\beta \, \mu \, T}{2}} \, \norm{\lstarmu} + \beta \sum_{t=0}^{T-1} \parens*{1 - \frac{\beta \, \mu}{2}}^{T-t-1}\delta_t \\
    &\leq \exp\parens*{-\frac{\beta \, \mu \, T}{2}} \, \norm{\lstarmu} + \hat{\delta} \, \beta \sum_{t=0}^{T-1} \parens*{1 - \frac{\beta \, \mu}{2}}^{T-t-1} \\
    \intertext{Summing up the geometric series since $\frac{\beta \, \mu}{2} < 1$}
\implies \norm{\lambda_T - \lstarmu}    &\leq \exp\parens*{-\frac{\beta \, \mu \, T}{2}} \, \norm{\lstarmu} +  \frac{2 \, \hat{\delta}}{\mu} \numberthis \label{eq:l2_smooth_bound} \\
    &\leq \exp\parens*{-\frac{\beta \, \mu \, T}{2}} \, \norm{\lstarmu} + \gO\parens*{\frac{\tau}{T^2}}
\end{align*}
Combining the above result with~\cref{lemma:reg_dual_distance_to_sub}, we have
\begin{align*}
    V^{\pi^*}_{r}(\rho) - V^{\pi^*_{\lambda_T, \tau}}_{r}(\rho) &\leq \gO\parens*{\frac{1}{\tau} \exp\parens*{-\frac{\beta \, \mu \, T}{2}}} + \gO\parens*{\frac{1}{T^2}} + \gO\parens*{\tau} \, , \\
    \max\braces*{b_i - V^{\pi^*_{\lambda_T, \tau}}_{c_i}(\rho), 0} &\leq \gO\parens*{\frac{1}{\tau} \exp\parens*{-\frac{\beta \, \mu \, T}{2}}} + \gO\parens*{\frac{1}{T^2}} + \gO(\mu) \quad \forall i \in [m] \, .
\end{align*}
Hence, setting $\tau = \mu = \gO(\eps)$, and $T = \tilde{\gO}(\nicefrac{1}{\eps^2})$  ensures
\begin{align*}
V^{\pi^*}_r(\rho) - V^{\pi^*_{\lambda_T, \tau}}_{r}(\rho) \leq \gO(\eps) \quad \text{ and } \max\braces*{b_i - V^{\pi^*_{\lambda_T, \tau}}_{c_i}(\rho), 0} \leq \gO(\eps) \quad \forall i \in [m] \, .
\end{align*}
Then by~\cref{proposition:transfer_rate}, applying \texttt{NPG} update to the pseudo-reward $R(\lambda_T)$ recovers a policy $\pi_N$ that is also $\gO(\eps)$-optimal after $N=\gO\parens*{\ln\parens*{\nicefrac{1}{\eps}}}$ iterations.

Finally, we account for the total number of gradient evaluations. Each of the $T = \tilde{\gO}(\nicefrac{1}{\eps^2})$ outer iterations requires $K = \gO(\ln(\nicefrac{1}{\eps}))$ iterations, and the final policy optimization requires an additional $N = \gO(\ln(1/\eps))$ iterations. Therefore, the total number of gradient evaluations is
\begin{align*}
    TK + N &= \tilde{\gO}\parens*{T} + \gO\parens*{\ln\parens*{\frac{1}{\eps}}} = \tilde{\gO}\parens*{\frac{1}{\eps^2}}
\end{align*}
\end{proof}

\subsection{Using $K = 1$ Inner-Iterations}\label{appendix:k_one_smooth_reg}
\elltwolastiterateKone*
\begin{proof}
Following the proof of~\cref{theorem:smooth_last_iterate_K_1}, 
by~\cref{lemma:bias_bounded_by_phi}, for all $t \geq 0$, we can relate the approximate dual gradient to an optimization error term:
\begin{equation*}
\norm{\dgt - \bdgt} \leq \frac{4 \sqrt{m} \, A}{\tau \, (1 - \gamma)^2} \, \Phi_t
\end{equation*}
where $\Phi_t \coloneqq \frac{1}{1 - \gamma} \, \supnorm{Q^{\pistart}_{t} - Q^{\pi_{t}}_{t}} + \frac{\tau}{4} \, \supnorm{\ln(\pistart) - \ln(\pi_t)}$. 
Moreover by~\cref{lemma:dual_smooth}, the dual $\cD(\lambda)$ is $G_\tau \coloneqq \frac{m}{\tau \, (1 - \gamma)^3}$-smooth. Hence, under~\cref{assumption:slater_condition}, by~\cref{lemma:reg_dual_linear_convergence} with $\delta_t = \frac{4 \sqrt{m} \, A}{\tau \, (1 - \gamma)^2} \, \Phi_t$,~\cref{algorithm:dd}, $\beta \leq \frac{2}{L_{\tau,\mu} + \mu}$, and $\lambda_0 = 0$ results in the following convergence
\begin{align*}
    \norm{\lambda_T - \lstarmu} &\leq \exp\parens*{-\frac{\beta \, \mu \, T}{2}} \, \norm{\lstarmu} + \frac{4 \, \sqrt{m} \, A \, \beta}{\tau \, (1 - \gamma)^2} \sum_{t=0}^{T-1} \parens*{1 - \frac{\beta \, \mu}{2}}^{T - t - 1} \, \Phi_t \\
    \numberthis \label{eq:smoothness_l2_recursion_phi}
\end{align*}
Furthermore, recall from~\cref{eq:tracking_bound9} of~\cref{theorem:smooth_last_iterate_K_1} we have,
\begin{equation*}
\Phi_{t+1} \leq  \frac{1+\gamma}{2} \, \Phi_t + \frac{m \, (1+\gamma)}{(1 - \gamma)^2} \supnorm{\lt - \ltt}
\end{equation*}
By the update in~\cref{algorithm:dd} (\cref{line:update}) $\ltt = \projL[\lt - \beta \bFgt] = \projL[\lt - \beta(\bdgt + \mu \lt)]$, we have
\begin{align*}
   \supnorm{\lt - \ltt} &\leq \norm{\lt - \ltt} \tag{$\supnorm{\cdot} \leq \norm{\cdot}$} \\
   &\leq \beta \norm{\bFgt} \tag{Since projections are non-expansive} \\
   &=\beta \norm{\bdgt + \mu \lt} \\
   &\leq \beta \, \norm{\bdgt} + \beta \, \mu \norm{\lt} \tag{By triangle-inequality}\\
   &\leq \beta \, \parens*{\frac{\sqrt{m}}{1 - \gamma} + \mu \, \abs{\Lambda}} \tag{Since $\lt \in \Lambda$ and by ~\cref{eq:bias_gradient_bound_norm_square}} \\
   \implies \supnorm{\lt - \ltt} \leq \norm{\lt - \ltt} &\leq \beta \, \parens*{\frac{\sqrt{m}}{1 - \gamma} + \mu \, \abs{\Lambda}} \numberthis \label{eq:lt_ltt_mu_ub} \\
   \implies 
\Phi_{t+1} &\leq  \underbrace{\frac{1+\gamma}{2}}_{< 1} \Phi_t + \underbrace{\frac{m \, (1+\gamma)}{(1 - \gamma)^2} \parens*{\frac{\sqrt{m}}{1 - \gamma} + \mu \, \abs{\Lambda}}}_{\coloneqq C_\mu} \, \beta \,
\end{align*}
Next, we bound the weighted sum $\sum_{t=0}^{T-1} \alpha^{T-1-t} \, \Phi_t$, where $\alpha \coloneqq 1 - \frac{\mu \, \beta}{2}$.
Let $\kappa \coloneqq \frac{1+\gamma}{2}$, so that $1 - \kappa = \frac{1-\gamma}{2}$. Furthermore, since $\mu \beta > 0$, $\alpha < 1$, we show $\alpha \in (\kappa, 1)$.
Since $\beta \leq \frac{2}{\mu + L_{\tau,\mu}}$ and $L_{\tau,\mu} \geq G_\tau = \frac{m}{\tau (1 - \gamma)^3}$,
\begin{align*}
    \mu \, \beta \leq \frac{2 \, \mu}{G_\tau} = \frac{2 \, \mu \, \tau \, (1 - \gamma)^3}{m} \leq \frac{1 - \gamma}{2} = 1 - \kappa \tag{$\mu \tau \leq \frac{1}{4}$, $m \geq 1$} \\
    \implies \alpha - \kappa = (1 - \kappa) - \frac{\mu \, \beta}{2} \geq \frac{1 - \gamma}{4} > 0 \, .
    \numberthis \label{eq:alpha_lb}
\end{align*}
Hence, by~\cref{eq:x-weighted} of~\cref{proposition:recursion-const} with $\nu = C_\mu \, \beta$,
\begin{align*}
    \sum_{t=0}^{T-1} \parens*{1 - \frac{\mu \, \beta}{2}}^{T-1-t} \Phi_t
    &\leq \frac{\alpha^T \, \Phi_0}{\alpha - \kappa} + \frac{C_\mu \, \beta}{(1 - \kappa) \, (1 - \alpha)}
    \leq \frac{4 \, \Phi_0}{1 - \gamma} \exp\parens*{-\frac{\mu \, \beta \, T}{2}} + \frac{4 \, C_\mu}{\mu \, (1 - \gamma)} \\
    &\leq \gO\parens*{\exp\parens*{-\frac{\mu \, \beta \, T}{2}}} + \gO\parens*{\frac{1}{\mu}} \numberthis \label{eq:smoothness_l2_recursion_phi_bound}
\end{align*}
Putting~\cref{eq:smoothness_l2_recursion_phi,eq:smoothness_l2_recursion_phi_bound} together, we have
\begin{align*}
\norm{\lambda_T - \lstarmu} &\leq \gO\parens*{\exp\parens*{\frac{-\mu \, \beta \, T}{2}}} + \gO\parens*{\frac{\beta \, \exp(-\frac{\mu \, \beta \, T}{2})}{\tau}} + \gO\parens*{\frac{\beta}{\tau \, \mu}}
\end{align*}
Using the above dual sub-optimality guarantee with~\cref{lemma:reg_dual_distance_to_sub}
\begin{align*}
    V^{\pi^*}_{r}(\rho) - V^{\pi^*_{\lambda_T, \tau}}_{r}(\rho) &\leq \gO\parens*{\frac{1}{\tau} \, \exp\parens*{\frac{-\mu \, \beta \, T}{2}}} + \gO\parens*{\frac{\beta \, \exp(-\frac{\mu \, \beta \, T}{2})}{\tau^2}} + \gO\parens*{\frac{\beta}{\tau^2 \, \mu}} + \gO(\tau) \\
    \max\braces*{b_i - V^{\pi^*_{\lambda_T, \tau}}_{c_i}(\rho), 0} &\leq \gO\parens*{\frac{1}{\tau} \, \exp\parens*{\frac{-\mu \, \beta \, T}{2}}} + \gO\parens*{\frac{\beta \, \exp(-\frac{\mu \, \beta \, T}{2})}{\tau^2}} + \gO\parens*{\frac{\beta}{\tau^2 \, \mu}} + \gO(\mu)  \quad \forall i \in [m] \, ,
\end{align*}
Hence, setting $\tau = \mu = \gO(\eps)$, $\beta = \gO(\eps^4)$, and $T = \tilde{\gO}(\nicefrac{1}{\eps^5})$ ensures
\begin{align*}
V^{\pi^*}_r(\rho) - V^{\pi^*_{\lambda_T, \tau}}_{r}(\rho) \leq \gO(\eps) \quad \text{ and } \max\braces*{b_i - V^{\pi^*_{\lambda_T, \tau}}_{c_i}(\rho), 0} \leq \gO(\eps) \quad \forall i \in [m] \, .
\end{align*}
Then by~\cref{proposition:transfer_rate}, applying \texttt{NPG} update to the pseudo-reward $R(\lambda_T)$ recovers a policy $\pi_N$ that is also $\gO(\eps)$-optimal after $N=\gO\parens*{\ln\parens*{\nicefrac{1}{\eps}}}$ iterations.

Finally, we account for the total number of gradient evaluations. Each of the $T = \tilde{\gO}(\nicefrac{1}{\eps^5})$ outer-iterations requires $K = 1$ inner-iterations, and the final policy optimization requires an additional $N = \gO(\ln(1/\eps))$ iterations. Therefore, the total number of gradient evaluations is $\tilde{\gO}\parens*{\frac{1}{\eps^5}}$.

\end{proof}

\subsection{Interpolating Between $K=1$ and $K > 1$ Inner-Iterations} \label{appendix:k_geq_one_smooth_reg}
\ellinterpolatelastiteratetheorem*
\begin{proof}
This proof proceeds similar to~\cref{theorem:smooth_last_iterate_K_1}.
By~\cref{lemma:bias_bounded_by_L}, for all $t \geq 0$, we can relate the approximate dual gradient to an optimization error term:
\begin{equation*}
\norm{\dgt - \bdgt} \leq \frac{\sqrt{m} \, A}{(1 - \gamma)^2} \, L_t
\end{equation*}
where $L_t \coloneqq \supnorm{\ln(\pistart) - \ln(\pi_t)}$. 
Moreover, by~\cref{lemma:dual_smooth}, the dual $\cD(\lambda)$ is $\frac{m}{\tau \, (1 - \gamma)^3}$-smooth. Hence, under~\cref{assumption:slater_condition}, by~\cref{lemma:reg_dual_linear_convergence} with $\delta_t = \frac{\sqrt{m} \, A}{(1 - \gamma)^2} \, L_t$,~\cref{algorithm:dd} 
using $\beta \leq \min\braces*{\frac{\tau (1 - \gamma)^3}{m}, \frac{2}{G_\tau + 2\mu}}$
, and $\lambda_0 = 0$
results in the following convergence
\begin{align*}
    \norm{\lambda_T - \lstarmu} &\leq \exp\parens*{-\frac{\beta \, \mu \, T}{2}} \, \norm{\lstarmu} + \frac{\sqrt{m} \, A \, \beta}{(1 - \gamma)^2} \sum_{t=0}^{T-1} \parens*{1 - \frac{\beta \, \mu}{2}}^{T - t - 1} \, L_t  \\
    &\leq \exp\parens*{-\frac{\beta \, \mu \, T}{2}} \, \norm{\lstarmu} + \frac{\sqrt{m} \, A \, \beta}{(1 - \gamma)^2} \max_t L_t \sum_{t=0}^{\infty} \parens*{1 - \frac{\beta \, \mu}{2}}^{t}  \\
    &= \exp\parens*{-\frac{\beta \, \mu \, T}{2}} \, \norm{\lstarmu} + \frac{2 \, \sqrt{m} \, A }{\mu \, (1 - \gamma)^2} \max_t L_t \tag{Since $\frac{\beta \, \mu}{2} < 1$, $\sum_{t=0}^{\infty} \parens*{1 - \frac{\beta \, \mu}{2}}^{t} = \frac{2}{\beta \, \mu}$}\\
   \implies  \norm{\lambda_T - \lstarmu} &\leq \exp\parens*{-\frac{\beta \, \mu \, T}{2}} \, \norm{\lstarmu} + \frac{2 \, \sqrt{m} \, A }{\mu \, (1 - \gamma)^2} \max_t L_t
    \numberthis \label{eq:reg_gd_L}
\end{align*}
We next derive a recursive bound on the optimization error, relating iteration $t+1$ to iteration $t$. 
Let $\pi_{t+1}$ denote the policy returned after running $K$ iterations of the \texttt{NPG} update in~\cref{eq:entropy_npg_update} with $\eta = \frac{1 - \gamma}{\tau}$ and $\diamond = R_{t+1}$.
Then by~\cref{lemma:genericLrecurse} with $E = 0$ and setting $n = 2^K$ we have,
\begin{align*}
    L_{t+1}  &\leq C_1^{n-1} \, \parens*{\supnorm{\ln(\pistartt) - \ln(\pistart)} + L_t }^{n}
\end{align*}
Furthermore, by~\cref{eq:sup_to_sup_lipschitz} of~\cref{proposition:pistarl_lipschitz} and~\cref{lemma:qstar_lipschitz},
\begin{align*}
   \supnorm{\ln(\pistartt) - \ln(\pistart)} &\leq \frac{2 \, \sqrt{m}}{\tau \, (1 - \gamma)} \, \norm{\ltt - \lt} \\
   &= \frac{2 \, \sqrt{m}}{\tau\, (1 - \gamma)} \, \norm{\projL[(1 - \beta \, \mu) \lt - \beta \, \bdgt] - \projL[\lt]}  \tag{Using the update in~\cref{line:update} of~\cref{algorithm:dd} and since $\lt \in \Lambda$} \\
   &\leq \frac{2 \, \sqrt{m} \, \beta}{\tau\, (1 - \gamma)} \, \norm{\bdgt + \mu \, \lt}  \tag{Projections are non-expansive} \\
   &\leq \frac{2 \, \sqrt{m} \, \beta}{\tau\, (1 - \gamma)} \parens*{\, \norm{\bdgt} + \mu \norm{\lt}} \tag{By triangle-inequality}\\
   &\leq \frac{2 \, \sqrt{m} \, \beta}{\tau\, (1 - \gamma)} \parens*{\, \norm{\bdgt} + \mu \abs*{\Lambda}} \tag{Since $\lambda_t \in \Lambda$}\\
   &\leq \frac{\beta}{\tau} \, \underbrace{\frac{2 \, \sqrt{m}}{1 - \gamma} \, \parens*{\frac{\sqrt{m}}{1 - \gamma} + \mu \, \abs{\Lambda}}}_{\coloneqq C_\mu} \tag{By~\cref{eq:bias_gradient_bound_norm_square}}
\end{align*}
\begin{align*}
    \implies L_{t+1} \leq C_1^{n-1} \parens*{\frac{C_\mu \, \beta}{\tau} + L_t}^n
\end{align*}
Furthermore, assuming that $\pi_0$ satisfies,
\begin{equation}
    L_0 = \supnorm{\ln(\pi^*_0) - \ln(\pi_0)} \leq 2^n \, C_1^{n -1} \, \parens*{C_\mu \, \frac{\beta}{\tau}}^n \label{eq:warm-start-pi1}
\end{equation}
and  the step-size satisfies $\beta \leq \frac{\tau}{4 \, C_1 \, C_\mu}$ we have,
\begin{align*}
C_1 \, C_\mu \, \frac{\beta}{\tau} &\leq \frac{1}{4} 
\implies 2^n \, \parens*{C_1 \, C_\mu \, \frac{\beta}{\tau}}^{n-1} \leq 2^n \, 4^{-(n-1)} \leq 2^{2 - n} = 2^{2 - 2^K} \leq 1 \tag{Since $n=2^K \geq 2$}.
\end{align*}
Hence, by~\cref{lemma:quadratic_recursion} with $\kappa = C_1^{n - 1}$, $\nu = \frac{C_\mu \, \beta}{\tau}$ we have, for all $t \geq 0$
\begin{equation*}
    L_t \leq 2^n \, C_1^{n-1} \, \parens*{C_\mu \frac{\beta}{\tau}}^n = \gO\parens*{\parens*{\frac{\beta}{\tau}}^n}
\end{equation*}
Using the above inequality with~\cref{eq:reg_gd_L},
\begin{align*}
   \norm{\lambda_T - \lstarmu} &\leq \exp\parens*{-\frac{\beta \, \mu \, T}{2}} \, \norm{\lstarmu} + \frac{2 \, \sqrt{m} \, A }{\mu \, (1 - \gamma)^2} \max_t L_t \leq \gO\parens*{\exp\parens*{-\frac{\beta \, \mu \, T}{2}}} + \gO\parens*{\frac{1}{\mu} \, \parens*{\frac{\beta}{\tau}}^n}
\end{align*}
Combining the above result with~\cref{lemma:reg_dual_distance_to_sub},
\begin{align*}
V^{\pi^*}_r(\rho) - V^{\pi^*_{\lambda_T, \tau}}_{r}(\rho) \leq \gO\parens*{\frac{\exp\parens*{-\frac{\beta \, \mu \, T}{2}}}{\tau}} + \gO\parens*{\frac{1}{\mu \, \tau} \, \parens*{\frac{\beta}{\tau}}^n} + \gO(\tau)
\\ 
\max\braces*{b_i - V^{\pi^*_{\lambda_T, \tau}}_{c_i}(\rho), 0} \leq \gO\parens*{\frac{\exp\parens*{-\frac{\beta \, \mu \, T}{2}}}{\tau}} + \gO\parens*{\frac{1}{\mu \, \tau} \, \parens*{\frac{\beta}{\tau}}^n} + \gO(\mu)
\end{align*}
To obtain an $\gO(\eps)$-optimal policy, we require each term to be $\gO(\eps)$. Setting $\tau = \mu = \gO(\eps)$ it suffices to set the step-size such that
\begin{align*}
\frac{1}{\eps^2}\parens*{\frac{\beta}{\eps}}^n \leq \gO(\eps)
  &\implies \beta = \gO\parens*{\eps^{\frac{n+3}{n}}}, \\
\intertext{Next, to determine the number of outer-iterations $T$,}
 \gO\parens*{\frac{1}{\eps}\exp\parens*{-\frac{\beta\,\eps\,T}{2}}} \leq \gO(\eps) & \implies  T \geq \tilde{\gO}\parens*{\frac{1}{\eps^{\frac{2n + 3}{n}}}}.
\end{align*}
Since $n = 2^K$, setting $\beta = \gO\parens*{\eps^\frac{3 + 2^K}{2^K}}$, $T = \tilde{\gO}\parens*{\frac{1}{\eps^{2 + \frac{3}{2^K}}}}$, we have
\begin{align*}
V^{\pi^*}_r(\rho) - V^{\pi^*_{\lambda_T, \tau}}_{r}(\rho) \leq \gO(\eps) \quad \text{ and } \max\braces*{b_i - V^{\pi^*_{\lambda_T, \tau}}_{c_i}(\rho), 0} \leq \gO(\eps) \quad \forall i \in [m] \, .
\end{align*}
Then by~\cref{proposition:transfer_rate}, applying \texttt{NPG} update to the pseudo-reward $R(\lambda_T)$ recovers a policy $\pi_N$ that is also $\gO(\eps)$-optimal after $N=\gO\parens*{\ln\parens*{\frac{1}{\eps}}}$ iterations.

Finally, we account for the total number of gradient evaluations.   Each of the $T = \tilde{\gO}\parens*{\frac{1}{\eps^{2 + \frac{3}{2^K}}}}$ outer iterations requires $K \geq 1$ iterations.
For the first iteration only, in order to satisfy~\cref{eq:warm-start-pi1}, we require
\begin{align*}
 L_0 &\leq 2^n \, C_1^{n-1} \, \parens*{\frac{C_\mu \, \beta}{\tau}}^n \leq \gO(\eps^3)
\end{align*}
By~\cref{lemma:entropy_npg_supnorm_convergence} with $\diamond = R(\lambda_0)$ and $\eta = \frac{1 - \gamma}{\tau}$ ensures
$\supnorm{\ln(\pi^*_0) - \ln(\pi_0)} \leq \gO(\eps^3)$ after $K_0 = \gO(\ln(\nicefrac{1}{\eps}))$ iterations. 
Finally, to recover a deployable policy  the final policy optimization requires an additional $N = \gO(\ln(\nicefrac{1}{\eps}))$ iterations. Therefore, the total number of gradient evaluations is
$\tilde{\gO}\parens*{\frac{1}{\eps^{2 + \frac{3}{2^K}}}}$.
\end{proof}

\section{Convergence Results When Using Linear Function Approximation}\label{appendix:linear_fa}
\coresetlemma*
\begin{proof} 
For a fixed iteration $t \geq 0$, solving the following regression problem
\begin{equation*}
    \min_{w \in \R^d} \sum_{(s, a) \in \cC} \varsigma(s, a) \, [Q^{\pi_k}_{R_t, \tau}(s, a) - \dpd{w, \varphi(s, a)}]^2
\end{equation*}
results in the following explicit solution:
\begin{equation*}
  \hat{w}^\star = G_\varsigma^\dagger \left[ \sum_{(s', a') \in \mathcal{C}} \varsigma(s', a') Q_{R_t, \tau}^{\pi_k}(s', a') \varphi(s', a') \right]  
\end{equation*}
where $G_\varsigma := \sum_{(s, a) \in \mathcal{C}} \varsigma(s, a) \varphi(s, a) \varphi(s, a)^\top$.
Let $w^\star$ be the minimizer of the regression loss over the entire state-action space over all admissible policies and pseudo-rewards. Under \cref{assumption:q_realizability,assumption:eps_opt}, we know that for any $k \in [K]$ and $t \in [T]$,
\begin{align*}
    | \langle \varphi(s, a), w^\star \rangle - Q_{R_t, \tau}^{\pi_k}(s, a) | \le \epsilon_{\text{bias}}, \quad | \langle \varphi(s, a), w_k \rangle - \langle \varphi(s, a), \hat{w}^\star \rangle | \le \sqrt{\epsilon_{\text{opt}}}.    
\end{align*}

Then, for any arbitrary $(s, a)$, using the triangle inequality, we have that
\begin{align*}
    \MoveEqLeft
    | Q_{R_t, \tau}^{\pi_k}(s, a) - \langle \varphi(s, a), w_k \rangle 
    \leq | Q_{R_t, \tau}^{\pi_k}(s, a) - \langle \varphi(s, a), w^\star \rangle | + | \langle \varphi(s, a), w_k \rangle - \langle \varphi(s, a), w^\star \rangle | \tag{Add/Subtract $\langle \varphi(s, a), w^\star \rangle$}\\
    &\leq \eps_{\bias} + | \langle \varphi(s, a), w_k \rangle - \langle \varphi(s, a), w^\star \rangle | \tag{Under~\cref{assumption:q_realizability}} \\
    &\leq \epsilon_{\text{bias}} + | \langle \varphi(s, a), w_k \rangle - \langle \varphi(s, a), \hat{w}^\star \rangle | + | \langle \varphi(s, a), \hat{w}^\star \rangle - \langle \varphi(s, a), w^\star \rangle | \tag{Add/Subtract $\langle \varphi(s, a), \hat{w}^\star \rangle$}\\
    &\leq \epsilon_{\text{bias}} + \sqrt{\epsilon_{\text{opt}}} + | \langle \varphi(s, a), \hat{w}^\star - w^\star \rangle | \tag{Under~\cref{assumption:eps_opt}}\,.
\end{align*}
It now remains to bound $| \langle \varphi(s, a), \hat{w}^\star - w^\star \rangle |$. 
First, we define $Y(s', a') := Q_{R_t, \tau}^{\pi_k}(s', a') - \langle \varphi(s', a'), w^\star \rangle$ for any $(s', a') \in \mathcal{C}$. Then, we have that
\begin{align*}
    \hat{w}^\star &= G_\varsigma^\dagger \left[ \sum_{(s', a') \in \mathcal{C}} \varsigma(s', a') \left( Y(s', a') + \langle \varphi(s', a'), w^\star \rangle \right) \varphi(s', a') \right] \\
    &= G_\varsigma^\dagger \left[ \sum_{(s', a') \in \mathcal{C}} \varsigma(s', a') \varphi(s', a') \varphi(s', a')^\top \right] w^\star + G_\varsigma^\dagger \left[ \sum_{(s', a') \in \mathcal{C}} \varsigma(s', a') Y(s', a') \varphi(s', a') \right] \\
    &= w^\star + G_\varsigma^\dagger \left[ \sum_{(s', a') \in \mathcal{C}} \varsigma(s', a') Y(s', a') \varphi(s', a') \right] \,. \\
    \implies \hat{w}^\star - w^\star &= G_\varsigma^\dagger \left[ \sum_{(s', a') \in \mathcal{C}} \varsigma(s', a') Y(s', a') \varphi(s', a') \right] \,.
\end{align*}
Hence, for any arbitrary $(s, a) \in \cS \times \cA$, we have that
\begin{align*}
    | \langle \varphi(s, a), \hat{w}^\star - w^\star \rangle | &= \left| \sum_{(s', a') \in \mathcal{C}} \varsigma(s', a') Y(s', a') \varphi(s, a)^\top G_\varsigma^\dagger \varphi(s', a') \right| \\
    &\leq \sum_{(s', a') \in \mathcal{C}} | Y(s', a') | \varsigma(s', a') \left| \varphi(s, a)^\top G_\varsigma^\dagger \varphi(s', a') \right| \tag{Using triangle inequality}\\
    &\leq \left( \max_{(s', a') \in \mathcal{C}} | Y(s', a') | \right) \sum_{(s', a') \in \mathcal{C}} \varsigma(s', a') \left| \varphi(s, a)^\top G_\varsigma^\dagger \varphi(s', a') \right| \\
    &\leq \epsilon_{\text{bias}} \sum_{(s', a') \in \mathcal{C}} \varsigma(s', a') \left| \varphi(s, a)^\top G_\varsigma^\dagger \varphi(s', a') \right| \tag{Under~\cref{assumption:q_realizability}, $\abs{Y(s', a')} \leq \eps_{\bias}$}\\
    &= \epsilon_{\text{bias}} \sqrt{ \left( \mathbb{E}_{(s', a') \sim \varsigma} \bracks*{\left| \varphi(s, a)^\top G_\varsigma^\dagger \varphi(s', a') \right|} \right)^2 } \\
    &\leq \epsilon_{\text{bias}} \sqrt{ \mathbb{E}_{(s', a') \sim \varsigma} \left[ \left| \varphi(s, a)^\top G_\varsigma^\dagger \varphi(s', a') \right|^2 \right] } \tag{Using Jensen's inequality}\\
    &= \epsilon_{\text{bias}} \sqrt{ \varphi(s, a)^\top G_\varsigma^\dagger \left[ \sum_{(s', a') \in \mathcal{C}} \varsigma(s', a') \varphi(s', a') \varphi(s', a')^\top \right] G_\varsigma^\dagger \varphi(s, a) } \\
    &= \epsilon_{\text{bias}} \sqrt{ \varphi(s, a)^\top G_\varsigma^\dagger G_\varsigma G_\varsigma^\dagger \varphi(s, a) } \tag{Definition of $G_\varsigma$} \\
    &= \epsilon_{\text{bias}} \sqrt{ \varphi(s, a)^\top G_\varsigma^\dagger \varphi(s, a) } = \epsilon_{\text{bias}} \lVert \varphi(s, a) \rVert_{G_\varsigma^\dagger} \,.
\end{align*}
Putting everything together, we have that for any $(s, a) \in \cS \times \cA$,
\begin{equation*}
| Q_{R_t, \tau}^{\pi_k}(s, a) - \langle \varphi(s, a), w_k \rangle | \le \sqrt{\epsilon_{\text{bias}}} \left( 1 + \lVert \varphi(s, a) \rVert_{G_\varsigma^\dagger} \right) + \sqrt{\epsilon_{\text{opt}}}    \, .
\end{equation*}
Taking the maximum over $(s, a) \in \cS \times \cA$ finishes the proof.
\end{proof}

\begin{corollary}\label{corollary:ineact_npg_coreset}
Under~\cref{assumption:q_realizability,assumption:eps_opt}, for a fixed pseudo-reward $R_t$ and $\tau > 0$, the \texttt{Q-E-NPG} update in~\cref{eq:q_npg_update} with $\eta = \frac{1 - \gamma}{\tau}$ initialized at $\pi_0$ satisfies, for all $k \geq 0$,
\begin{align}
\supnorm{\ln(\pi^*_{t}) - \ln(\pi_{k+1})} &\leq \frac{2}{\tau} \parens*{\supnorm{Q^{\pi^*_t}_{t} - Q^{\pi_{k}}_{t}} + \eps_{\cC}} \, , \label{eq:qnpg_one_step}
\end{align}
and for all $K \geq 0$,
\begin{align*}
    \supnorm{Q^{\pi^*_t}_{t} - Q^{\pi_{K+1}}_{t}} &\leq \gamma^{K+1} \supnorm{Q^{\pi^*_t}_{t} - Q^{\pi_0}_{t}} + \frac{2 \, \gamma \, \eps_{\cC}}{(1 - \gamma)^2} \, , \\
    \supnorm{\ln(\pi^*_{t}) - \ln(\pi_{K+1})} &\leq \frac{2 \, \gamma^{K}}{\tau} \supnorm{Q^{\pi^*_t}_{t} - Q^{\pi_0}_{t}} + \frac{4 \, \eps_{\cC}}{\tau \, (1 - \gamma)^2} \, , \\
    V^{\pi^*_t}_{R_t, \tau}(\rho) - V^{\pi_{K+1}}_{R_t, \tau}(\rho) &\leq (2 + \gamma) \bracks*{\gamma^{K} \supnorm{Q^{\pi^*_t}_{t} - Q^{\pi_0}_{t}} + \frac{2 \, \eps_{\cC}}{(1 - \gamma)^2}} \, ,
\end{align*}
where $\pi^*_{t} \coloneqq \argmax_{\pi} V^{\pi}_{R_t, \tau}(\rho)$ and $\eps_\cC \coloneqq \eps_{\bias}(1 + \sqrt{d}) + \sqrt{d \, \eps_{\opt}}$.
\end{corollary}
\begin{proof}
Under~\cref{assumption:slater_condition,assumption:q_realizability,assumption:eps_opt}, using~\cref{lemma:coreset}, we have that for all $k \in [K]$, 
\begin{align}
    \norm*{\dpd{\varphi(s, a), w_k} - Q_{R_t, \tau}^{\pi_k}(s, a)}_\infty \leq \eps_{\bias} \, \parens*{1 +\max_{s,a}\lVert \varphi(s, a) \rVert_{G_\varsigma^\dagger}} + \sqrt{\eps_{\opt}} \,.
\end{align}
Furthermore, by the Kiefer-Wolfowitz theorem~\citep{kiefer1960equivalence}, there exists a coreset $\gC$ such that $\gC \leq \frac{d (d + 1)}{2}$ and distribution $\varsigma \in (\Delta_A)^S$ that ensures $\max_{s, a} \|\varphi(s,a)\|_{G^\dagger} \leq \sqrt{d}$. 
Hence, we have that
\begin{align*}
\norm*{\dpd{\varphi(s, a), w_k} - Q_{R_t, \tau}^{\pi_k}(s, a)}_\infty \leq \eps_{\bias} \, \parens*{1 + \sqrt{d}} + \sqrt{d \, \eps_{\opt}}  \coloneqq \eps_\cC
\end{align*}
Since the \texttt{Q-E-NPG} update in~\cref{eq:q_npg_update} is equivalent to the inexact \texttt{NPG} update in~\cref{eq:inexact_npg_update_appendix} with $\eta = \frac{1 - \gamma}{\tau}$ and $\diamond = R_t$, the last three inequalities follow by~\cref{lemma:inexact_npg} with $\diamond = R_t$ and $E = \eps_{\cC}$.

Finally for~\cref{eq:qnpg_one_step}: for each $s$, $\pi_{k+1}(\cdot|s) = \sigma^*_\tau(\dpd{\varphi(s,\cdot), w_k})$ and $\pi^*_t(\cdot|s) = \sigma^*_\tau(Q^{\pi^*_t}_t(s,\cdot))$, so by~\cref{eq:sup_to_sup_lipschitz} of~\cref{proposition:pistarl_lipschitz} and the triangle inequality, $\supnorm{\ln \pi^*_t - \ln \pi_{k+1}} \leq \frac{2}{\tau}\parens*{\supnorm{Q^{\pi^*_t}_t - Q^{\pi_k}_t} + \eps_\cC}$.
\end{proof}

\subsection{Convergence Results Under Smoothness}
\begin{corollary}\label{corollary:smooth_last_iterate_K_ln_eps_FA}
Under~\cref{assumption:slater_condition,assumption:q_realizability,assumption:eps_opt}, for a given $\eps \in (0, 1)$ assuming $\eps_{\bias} \leq \gO(\eps^{8})$ and $\eps_{\opt} \leq \gO(\eps^{16})$,~\cref{algorithm:dd} with (i) $\tau = \gO(\eps)$, $T = \gO(\nicefrac{1}{\eps^4})$, $\beta = \gO(\eps)$, $\mu = 0$, $\lambda_0 = 0$ and (ii) using the \texttt{Q-E-NPG} update in~\cref{eq:inexact_entropy_npg_update} with $\eta = \frac{1 - \gamma}{\tau}$ and $K = J = \gO(\ln(\nicefrac{1}{\eps}))$  returns an $\gO(\eps)$-optimal policy for~\cref{problem:cmdp} after $\widetilde{\gO}\parens*{\nicefrac{1}{\eps^4}}$ gradient evaluations.
\end{corollary}
\begin{proof}
This proof proceeds similarly to~\cref{theorem:meta_smooth}.  
We first construct a uniform bound of the dual gradient error. By~\cref{lemma:bound_bias_subgradient}, for all $t \geq 0$, we can relate the approximate gradient to the following optimization term:
\begin{equation*}
   \norm{\dgt - \bdgt} \leq  \frac{\sqrt{ 2 \, m }}{(1 - \gamma)^{\nicefrac{3}{2}} \, \sqrt{\tau}} \sqrt{ \max_{\pi} V^{\pi}_{R(\lambda_t), \tau}(\rho) - V^{\pi_{t}}_{R(\lambda_t), \tau}(\rho)}
\end{equation*}
To characterize the sub-optimality gap $\max_{\pi} V^{\pi}_{R(\lambda_t), \tau}(\rho) - V^{\pi_{t}}_{R(\lambda_t), \tau}(\rho)$, we instantiate the \texttt{PolicyOptimization} method in~\cref{algorithm:dd} using \texttt{Q-E-NPG} update in~\cref{eq:q_npg_update}.
At iteration $t$, by~\cref{corollary:ineact_npg_coreset}, with $\diamond = R(\lambda_t)$, $\pi_0 = \pi_{t-1}$, $\eta = \frac{1 - \gamma}{\tau}$, the policy $\pi_{t} = \pi_{K}$ satisfies
\begin{align*}
    \max_{\pi} V^{\pi}_{R(\lambda_t), \tau}(\rho) - V^{\pi_{t}}_{R(\lambda_t), \tau}(\rho) \leq \gO\parens*{\frac{\tau}{T^2} + \frac{\eps_{\cC}}{\tau}} \, .
\end{align*}
after $K \geq \gO\parens*{\ln\parens*{\frac{T}{\tau}}}$ iterations where $\eps_{\cC} = \gO(\eps_{\bias} + \sqrt{\eps_{\opt}})$. 
Hence, for all $t \in [T]$
\begin{equation}
   \norm{\dgt - \bdgt} \leq \gO\parens*{\frac{1}{T} + \frac{\sqrt{\eps_{\cC}}}{\tau}} \coloneqq \delta_t \, .\label{eq:q_npg_gradient_bound}
\end{equation}
and let $\hat{\delta} = \max_{t \in [T]} \delta_t$.

We now characterize the dual convergence using the above uniform gradient error.  By~\cref{lemma:dual_smooth}, $\cD(\lambda)$ is $G_\tau = \frac{m}{\tau \, (1 - \gamma)^3}$-smooth.
Hence, under~\cref{assumption:slater_condition}, by~\cref{lemma:meta_lemma_smooth} with $\delta_t = \hat{\delta}$,~\cref{algorithm:dd} using step-size $\beta = \frac{1}{G_\tau}$, $\mu = 0$ and $\lambda_0 =0$ results in the following convergence:
\begin{align*}
\cD(\bar{\lambda}_T) - \cD(\lstar) \leq \frac{G_\tau \, \normsq{\lstar}}{2 \, T}  + \abs{\Lambda} \, \hat{\delta}  \leq \gO\parens*{\frac{G_\tau}{T}} = \gO\parens*{\frac{1}{\tau \, T} + \frac{\sqrt{\eps_{\cC}}}{\tau}}
\end{align*}
Thus, by~\cref{lemma:dual_optimality_to_sub} with $\Delta = \gO\parens*{\frac{1}{\tau \, T} + \frac{\sqrt{\eps_{\cC}}}{\tau}}$, the average dual variable $\bar{\lambda}_T$ returned by~\cref{algorithm:dd} satisfies
\begin{align*}
V^{\pi^*}_r(\rho) - V^{\pistarblT}_{r}(\rho) &\leq \tilde{\gO}\parens*{\frac{1}{{\sqrt{T} \, \tau}} + \frac{\sqrt[4]{\eps_{\cC}}}{\sqrt{\tau}}} + \gO(\tau)  \\
\max\braces*{b_i - V^{\pistarblT}_{c_i}(\rho), 0} &\leq \tilde{\gO}\parens*{\frac{1}{{\sqrt{T} \, \tau}} + \frac{\sqrt[4]{\eps_{\cC}}}{\sqrt{\tau}}} \quad \forall i \in [m] \, .
\end{align*}
Assuming that $\eps_{\bias} \leq \gO(\eps^{8})$ and $\eps_{\opt} \leq \gO(\eps^{16})$, setting $\tau = \gO(\eps)$ and $T = \gO(\nicefrac{1}{\eps^4})$ ensures
\begin{align*}
V^{\pi^*}_r(\rho) - V^{\pistarblT}_{r}(\rho) \leq \gO(\eps) \quad \text{ and } \max\braces*{b_i - V^{\pistarblT}_{c_i}(\rho), 0} \leq \gO(\eps) \quad \forall i \in [m] \, .
\end{align*}
Then by~\cref{proposition:transfer_rate_fa}, applying \texttt{Q-E-NPG} to the pseudo-reward $R(\bar{\lambda}_T)$ recovers a policy $\pi_N$ that is also an $\gO(\eps)$-optimal policy after $N=\gO\parens*{\ln\parens*{\frac{1}{\eps}}}$ iterations.
Therefore, the total number of gradient evaluations is $\tilde{\gO}\parens*{\frac{1}{\eps^4}}$
\end{proof}

\begin{restatable}{corollary}{smoothddlastitkoneeratecorollary}\label{corollary:smooth_last_iterate_K_1_FA}
Under~\cref{assumption:slater_condition,assumption:q_realizability,assumption:eps_opt}, for a given $\eps \in (0, 1)$ assuming $\eps_{\bias} \leq \gO(\eps^4)$ and $\eps_{\opt} \leq \gO(\eps^8)$,~\cref{algorithm:dd} with (i) $\tau = \gO(\eps)$, $T = \gO(\nicefrac{1}{\eps^7})$, $\beta = \gO(\eps^4)$, $\mu = 0$, $\lambda_0 = 0$ and (ii) using the \texttt{Q-E-NPG} update in~\cref{eq:inexact_entropy_npg_update} with $\eta = \frac{1 - \gamma}{\tau}$, $K = 1$, $J = \gO(\ln(\nicefrac{1}{\eps}))$ returns an $\gO(\eps)$-optimal policy for~\cref{problem:cmdp} after $\widetilde{\gO}\parens*{\nicefrac{1}{\eps^7}}$ gradient evaluations.
\end{restatable}
\begin{proof}
This proof proceeds similarly to~\cref{theorem:smooth_last_iterate_K_1}. Starting from~\cref{eq:smoothness_recursion_phi} we have
\begin{align*}
   \cD(\bar{\lambda}_T) - \cD(\lstar) &\leq  \frac{\normsq{\lstar}}{2 \, \beta \, T}  +  \frac{4 \, \sqrt{m} \, A \, \abs{\Lambda}}{\tau \, (1 - \gamma)^2 T} \sum_{t=0}^{T-1} \Phi_t \numberthis \label{eq:smoothness_recursion_phi_la}
\end{align*}

Next, we derive a recursive bound on the optimization error, relating iteration $t+1$ to iteration $t$ with an additional drift term between $\lambda_t$ and $\lambda_{t+1}$ and error term due to the linear approximation. 
By~\cref{corollary:ineact_npg_coreset} with $\pi_0 = \pi_t$, running $K=1$ iterations of \texttt{Q-E-NPG} update in~\cref{eq:q_npg_update} with $\eta = \frac{1 - \gamma}{\tau}$ returns a policy $\pi_{t+1} = \pi_K$ such that
\begin{align}
\supnorm{Q^{\pistartt}_{t+1} - Q^{\pi_{t+1}}_{t+1}} &\leq \gamma \, \supnorm{Q^{\pistartt}_{t+1} - Q^{\pi_{t}}_{t+1}} + \frac{2 \, \eps_{\cC}}{(1 - \gamma)^2} \label{eq:linear_tracking_bound1} \\
\supnorm{\ln(\pi^*_{t+1}) - \ln(\pi_{t+1})} &\leq \frac{2}{\tau}  \supnorm{Q^{\pistartt}_{t+1} - Q^{\pi_{t}}_{t+1}} + \frac{4 \,  \eps_{\cC}}{\tau \, (1 - \gamma)^2} \label{eq:linear_tracking_bound2}
\end{align}
As a result, by~\cref{lemma:generic_phi_recurse} with $E = \frac{2 \, \eps_\cC}{(1 - \gamma)^2}$ we have
\begin{align*}
    \Phi_{t+1} &\leq \frac{1+\gamma}{2} \, \Phi_t + \frac{m \, (1+\gamma)}{(1 - \gamma)^2} \supnorm{\lt - \ltt} + \frac{(3 - \gamma) \, \eps_\cC}{(1 - \gamma)^3} \, \numberthis \label{eq:phi_recurse_fa}\\
    &\leq \frac{1+\gamma}{2} \, \Phi_t + \frac{m^{3/2} \, (1+\gamma)}{(1 - \gamma)^3} \, \beta + \frac{(3 - \gamma) \, \eps_\cC}{(1 - \gamma)^3} \tag{Using~\cref{eq:lt_ltt_ub}}
\end{align*}
By~\cref{proposition:recursion-const} with $\kappa = \frac{1+\gamma}{2}$ and $\nu = \frac{m^{3/2} (1+\gamma)}{(1 - \gamma)^3} \beta + \frac{(3 - \gamma) \eps_\cC}{(1 - \gamma)^3}$,
\begin{align*}
    \frac{1}{T} \sum_{t=0}^{T-1} \Phi_t &\leq \frac{2 \, \Phi_0}{(1 - \gamma) \, T} + \frac{2 \, m^{3/2} \, (1+\gamma) \, \beta}{(1 - \gamma)^4} + \frac{2 \, (3 - \gamma) \, \eps_\cC}{(1 - \gamma)^4} \, .
\end{align*}
Using the above bound with~\cref{eq:smoothness_recursion_phi}, we have
\begin{align*}
   \cD(\bar{\lambda}_T) - \cD(\lstar) &\leq  \frac{\normsq{\lstar}}{2 \, \beta \, T}  +  \frac{4 \, \sqrt{m} \, A \, \abs{\Lambda}}{\tau \, (1 - \gamma)^2 T} \sum_{t=0}^{T-1} \Phi_t  \\
  &\leq \gO\parens*{\frac{1}{\beta \, T}} + \gO\parens*{\frac{1}{\tau \, T}} + \gO\parens*{\frac{\beta}{\tau}} + \gO\parens*{\frac{\eps_{\cC}}{\tau}}
\end{align*}
Using the above dual sub-optimality guarantee with~\cref{lemma:dual_optimality_to_sub}, we have
\begin{align*}
    V^{\pi^*}_r(\rho) - V^{\pistarblT}_r(\rho) &\leq \gO\parens*{\frac{1}{\sqrt{\tau \, \beta \, T}}} + \gO\parens*{\frac{\sqrt{\beta}}{\tau}} + \gO\parens*{\frac{\sqrt{\eps_{\cC}}}{\tau}} + \gO(\tau) \\
   \max\braces*{b_i - V^{\pistarblT}_{c_i}(\rho), 0} &\leq  \gO\parens*{\frac{1}{\sqrt{\tau \, \beta \, T}}} + \gO\parens*{\frac{\sqrt{\beta}}{\tau}} + \gO\parens*{\frac{\sqrt{\eps_{\cC}}}{\tau}}+ \gO(\tau)
   \quad \forall i \in [m] \, .
\end{align*}
Since $\eps_\cC = \gO(\eps_{\bias} + \sqrt{\eps_{\opt}})$, assuming $\eps_{\bias} \leq \gO(\eps^4)$ and $\eps_{\opt} \leq \gO(\eps^8)$, setting $\tau = \gO(\eps)$, $\beta = \gO(\eps^4)$ and $T = \gO(\nicefrac{1}{\eps^7})$ ensures
\begin{align*}
V^{\pi^*}_r(\rho) - V^{\pistarblT}_{r}(\rho) \leq \gO(\eps) \quad \text{ and } \max\braces*{b_i - V^{\pistarblT}_{c_i}(\rho), 0} \leq \gO(\eps) \quad \forall i \in [m] \, .
\end{align*}
Then by~\cref{proposition:transfer_rate_fa}, applying \texttt{Q-E-NPG} update to the pseudo-reward $R(\bar{\lambda}_T)$ recovers a policy $\pi_N$ that is also $\gO(\eps)$-optimal after $N=\gO\parens*{\ln\parens*{\nicefrac{1}{\eps}}}$ iterations. Furthermore, following the same argument as~\cref{theorem:smooth_last_iterate_K_1}, the total number of gradient evaluations is $\gO(\nicefrac{1}{\eps^7})$.
\end{proof}

\begin{corollary}\label{corollary:smooth_last_iterate_K_geq_1_fa}
Under~\cref{assumption:slater_condition,assumption:q_realizability,assumption:eps_opt}, for a given $\eps \in (0, 1)$ assuming $\eps_{\bias} \leq \gO(\eps^4)$ and $\eps_{\opt} \leq \gO(\eps^8)$,~\cref{algorithm:dd} with (i) $\tau = \gO(\eps)$, $T = \tilde{\gO}(\nicefrac{1}{\eps^{4 + 3/2^K}})$, $\beta = \gO\parens*{\eps^\frac{3 + 2^K}{2^K}}$, $\mu = 0$, $\lambda_0 = 0$ (ii) using the \texttt{Q-E-NPG} update in~\cref{eq:q_npg_update} with $\eta = \frac{1 - \gamma}{\tau}$, $K \geq 1$, $J = \gO(\ln(\nicefrac{1}{\eps}))$ and
(iii) the initial policy satisfies $L_0 \leq B$ for a constant $B > 0$,~\cref{algorithm:dd} return an
$\gO(\eps)$-optimal policy for~\cref{problem:cmdp} 
after $\widetilde{\gO}(\nicefrac{1}{\eps^{4 + 3/2^K}})$ gradient evaluations.
\end{corollary}

\begin{proof}
This proof proceeds similarly to~\cref{theorem:smooth_last_iterate_K_geq_1}. Starting from~\cref{eq:smoothness_recursion_L} we have
\begin{align}
   \cD(\bar{\lambda}_T) - \cD(\lstar) &\leq  \frac{\normsq{\lstar}}{2 \, \beta \, T}  +  \frac{\sqrt{m} \, A \, \abs{\Lambda}}{(1 - \gamma)^2 \, T} \sum_{t=0}^{T-1} L_t  \leq \frac{\normsq{\lstar}}{2 \, \beta \, T} + \frac{\sqrt{m} \, A \, \abs{\Lambda}}{(1 - \gamma)^2} \, \max_{t \in [0, T-1]} L_t
   \numberthis \label{eq:smoothness_recursion_L_fa}
\end{align}
We next show that for all $t \in [0, T-1]$,
\begin{equation}
    L_t \leq B \coloneqq 2^n \, \parens*{\frac{\gamma}{1 - \gamma}}^{n-1} \, \parens*{\frac{2 \, m \, \beta}{\tau \, (1 - \gamma)^2} + \frac{4 \, \eps_{\cC}}{\tau \, (1 - \gamma)^2}}^n + \frac{4 \, \eps_{\cC}}{\tau \, (1 - \gamma)^2}
\end{equation}
where $\eps_{\cC} \coloneqq \sqrt{\bias}(1 + \sqrt{d}) + \sqrt{d \, \eps_{\opt}}$. 
To do so, we assume that the initial policy satisfies
\begin{equation*}
    \supnorm{\ln(\pi^*_0) - \ln(\pi_0)} = L_0 \leq B \, ,
\end{equation*}
where $B$ is within the following basin,
\begin{equation}\label{eq:basin_condition}
    B \leq R \coloneqq \frac{1- \gamma}{\gamma} \bracks*{\frac{1}{4} - \frac{2 \, \gamma \, \eps_{\cC}}{\tau \, (1 - \gamma)^3}} - \frac{2 \, m \, \beta}{\tau \, (1 - \gamma)^2}
\end{equation}
and that
\begin{equation}\label{eq:fa_invariant_assumption_smooth}
    2^n \parens*{\frac{\gamma}{1 - \gamma}}^{n-1}
    \parens*{\frac{2 \, m \, \beta + 4 \, \eps_{\cC}}{\tau \, (1 - \gamma)^2}}^{n-1} \leq 1 \, .
\end{equation}
We will prove that $L_t \leq B$, for all $t \geq 0$ by an inductive argument.
Suppose that there exists an $t$ such that $L_t \leq B$. Then at iteration $t+1$, by~\cref{corollary:ineact_npg_coreset} with $\pi_0 = \pi_{t}$, running $K \geq 1$ iterations of \texttt{Q-E-NPG} update in~\cref{eq:q_npg_update} with $\eta = \frac{1 - \gamma}{\tau}$produces a sequence of policies $\{\hat{\pi}_k\}_{k=0}^{K}$ such that for $k \in [0, K-1]$,
\begin{align}
\supnorm{\ln(\pi^*_{t+1}) - \ln(\hat{\pi}_{k+1})} \leq \frac{2}{\tau} \parens*{\supnorm{Q^{\pi^*_{t+1}}_{t+1} - Q^{\hat{\pi}_k}_{t+1}} + \frac{\eps_{\cC}}{(1 - \gamma)^2}} \, ,
\end{align}
where $\hat{\pi}_0 = \pi_t$ and $\hat{\pi}_{K} = \pi_{t+1}$.
In order to use the above result to show a contraction using~\cref{lemma:genericLrecurse}, we require that the following condition holds:
\begin{equation*}
    \frac{\gamma}{1 - \gamma} \supnorm{\ln(\pi^*_{t+1}) - \ln(\pi_t)}  + \frac{2 \, \gamma \, \eps_{\cC}}{\tau \, (1 - \gamma)^3} \leq \frac{1}{4}
\end{equation*}

Observe that
\begin{align*}
   \supnorm{\ln(\pi^*_{t+1}) - \ln(\pi_t)}  &\leq  \supnorm{\ln(\pi^*_{t+1}) - \ln(\pi^*_t) + \ln(\pi^*_t) - \ln(\pi_t)} \tag{Add/Subtract $\pi^*_t$} \\
&\leq  \supnorm{\ln(\pi^*_{t+1}) - \ln(\pi^*_t)} + \supnorm{\ln(\pi^*_t) - \ln(\pi_t)} \tag{Using triangle inequality} \\ 
&\leq  \supnorm{\ln(\pi^*_{t+1}) - \ln(\pi^*_t)} + B \tag{Since $L_t \leq B$} \\ 
&\leq \frac{2 \, \sqrt{m}}{\tau \, (1 - \gamma)} \, \norm{\ltt - \lt} + B\tag{By~\cref{eq:ln_star_tt_ln_star_t_ub}} \\
&\leq \frac{2 \, m \, \beta}{\tau \, (1 - \gamma)^2} + B\tag{By~\cref{eq:lt_ltt_ub}}
\end{align*}
Therefore, using the above relation,
\begin{align*}
&\frac{\gamma}{1 - \gamma} \, \supnorm{\ln(\pi^*_{t+1}) - \ln(\pi_t)} + \frac{2 \, \gamma \, \eps_{\cC}}{\tau \, (1 - \gamma)^3} \\
&\leq \frac{\gamma}{1 - \gamma} \parens*{\frac{2 \, m \, \beta}{\tau \, (1 - \gamma)^2} + B} + \frac{2 \, \gamma \, \eps_{\cC}}{\tau \, (1 - \gamma)^3}
\\
&\leq \frac{\gamma}{1 - \gamma} \, \frac{2 \, m \, \beta}{\tau \, (1 - \gamma)^2} + \frac{\gamma}{1 - \gamma} \, R + \frac{2 \, \gamma \, \eps_{\cC}}{\tau \, (1 - \gamma)^3} \tag{Since $B \leq R$ and $\frac{\gamma}{1 - \gamma} > 0$} \\
&= \frac{\gamma}{1 - \gamma} \, \frac{2 \, m \, \beta}{\tau \, (1 - \gamma)^2}
    + \bracks*{\frac{1}{4} - \frac{2 \, \gamma \, \eps_{\cC}}{\tau \, (1 - \gamma)^3} - \frac{\gamma}{1 - \gamma} \, \frac{2 \, m \, \beta}{\tau \, (1 - \gamma)^2}} + \frac{2 \, \gamma \, \eps_{\cC}}{\tau \, (1 - \gamma)^3}
    \tag{Definition of $R$} \\
   &= \frac{1}{4} 
\end{align*}
Hence by~\cref{lemma:genericLrecurse} with $E = \frac{\eps_{\cC}}{(1 - \gamma)^2}$,
\begin{align*}
L_{t+1} &\leq  \parens*{\frac{\gamma}{1 - \gamma}}^{n-1} \, \parens*{\frac{2 \, \sqrt{m}}{\tau \, (1- \gamma)} \, \norm{\ltt - \lt} + L_t }^n + \frac{4 \, \eps_{\cC}}{\tau \, (1 - \gamma)^2}  \\
&\leq \parens*{\frac{\gamma}{1 - \gamma}}^{n-1} \, \parens*{\frac{2 \, m \, \beta}{\tau \, (1- \gamma)^2} \,  + L_t }^n + \frac{4 \, \eps_{\cC}}{\tau \, (1 - \gamma)^2} \tag{Using~\cref{eq:lt_ltt_ub}} 
\end{align*}
Using the above contraction, we now show that $L_t \leq B$, for all $t \geq 0$. By the definition of $B$ and~\cref{eq:fa_invariant_assumption_smooth}
\begin{align*}
B + \frac{2 \, m \, \beta}{\tau \, (1 - \gamma)^2}
&= \frac{2 \, m \, \beta + 4 \, \eps_{\cC}}{\tau \, (1 - \gamma)^2}
   \bracks*{2^n \parens*{\frac{\gamma}{1 - \gamma}}^{n-1}
   \parens*{\frac{2 \, m \, \beta + 4 \, \eps_{\cC}}{\tau \, (1 - \gamma)^2}}^{n-1} + 1}  \leq 2 \, \frac{2 \, m \, \beta + 4 \, \eps_{\cC}}{\tau \, (1 - \gamma)^2} \numberthis \label{eq:inv_step_condition}\, .
\end{align*}
Therefore, since $x \mapsto x^n$ is non-decreasing on $x \geq 0$,
\begin{align*}
\parens*{\frac{\gamma}{1 - \gamma}}^{n-1} \parens*{B + \frac{2 \, m \, \beta}{\tau \, (1 - \gamma)^2}}^n
+ \frac{4 \, \eps_{\cC}}{\tau \, (1 - \gamma)^2}
&\leq 2^n \parens*{\frac{\gamma}{1 - \gamma}}^{n-1}
   \parens*{\frac{2 \, m \, \beta + 4 \, \eps_{\cC}}{\tau \, (1 - \gamma)^2}}^{n}
   + \frac{4 \, \eps_{\cC}}{\tau \, (1 - \gamma)^2} = B \, .
\end{align*}
Since the conditions of~\cref{lemma:invariant_recursion} are satisfied  with $\kappa = \parens*{\frac{\gamma}{1 - \gamma}}^{n-1}$, $\nu = \frac{2 \, m \, \beta}{\tau \, (1- \gamma)^2}$ and $e = \frac{4 \, \eps_{\cC}}{\tau \,  (1 - \gamma)^2}$, we have for all $t \geq 0$.
\begin{equation*}
    L_t \leq B
\end{equation*}
Combining the above result with~\cref{eq:smoothness_recursion_L_fa},
\begin{align*}
   \cD(\bar{\lambda}_T) - \cD(\lstar) &\leq  \frac{\normsq{\lstar}}{2 \, \beta \, T}  +  \frac{\sqrt{m} \, A \, \abs{\Lambda}}{(1 - \gamma)^2 \, T} \leq \gO\parens*{\frac{1}{\beta \, T}} + \gO\parens*{\frac{\beta}{\tau} + \frac{\eps_{\cC}}{\tau}}^n + \gO\parens*{\frac{\eps_{\cC}}{\tau}} \,.
\end{align*}
Combining the dual sub-optimality bound with~\cref{lemma:dual_optimality_to_sub} with $\Delta = \gO\parens*{\frac{1}{\beta \, T}} + \gO\parens*{\frac{\beta}{\tau} + \frac{\eps_{\cC}}{\tau}}^n + \gO\parens*{\frac{\eps_{\cC}}{\tau}}$
\begin{align*}
V^{\pi^*}_r(\rho) - V^{\pistarblT}_{r}(\rho)
&\leq \gO\parens*{\frac{1}{\sqrt{\tau}} \parens*{\frac{1}{\sqrt{\beta \, T}} + \parens*{\frac{\beta + \eps_{\cC}}{\tau}}^{n/2} + \sqrt{\frac{\eps_{\cC}}{\tau}}}} + \gO(\tau) \, , \\
\max\braces*{b_i - V^{\pistarblT}_{c_i}(\rho), 0}
&\leq \gO\parens*{\frac{1}{\sqrt{\tau}} \parens*{\frac{1}{\sqrt{\beta \, T}} + \parens*{\frac{\beta + \eps_{\cC}}{\tau}}^{n/2} + \sqrt{\frac{\eps_{\cC}}{\tau}}}} \quad \forall i \in [m] \, .
\end{align*}
To obtain an $\gO(\eps)$-optimal policy, we require each term to be $\gO(\eps)$ and hence, set $\tau = \gO(\eps)$.
For the bound the approximation error $\eps_{\cC}$,
\begin{align*}
    \gO\parens*{\frac{1}{\sqrt{\eps}} \, \sqrt{\frac{\eps_{\cC}}{\eps}}} = \gO\parens*{\frac{\sqrt{\eps_{\cC}}}{\eps}} \leq \gO(\eps)
    \implies \eps_{\cC} \leq \gO\parens*{\eps^4} \, ,
\end{align*}
which since $\eps_{\cC} = \eps_{\bias}(1 + \sqrt{d}) + \sqrt{d \, \eps_{\opt}}$, holds when 
$\eps_{\bias} \leq \gO\parens*{\eps^4}$ and $\eps_{\opt} \leq \gO\parens*{\eps^8}$. \\
For the step-size $\beta$,
\begin{align*}
    \gO\parens*{\frac{1}{\sqrt{\eps}} \, \parens*{\frac{\beta + \eps_{\cC}}{\eps}}^{n/2}} \leq \gO(\eps)
    \implies \beta + \eps_{\cC} \leq \gO\parens*{\eps^{\frac{3 + n}{n}}} \, ,
\end{align*}
and since $\eps_{\cC} \leq \gO(\eps^4) \leq \gO\parens*{\eps^{\frac{3 + n}{n}}}$ for $n \geq 1$, it suffices that
$\beta \leq \gO\parens*{\eps^{\frac{3 + n}{n}}}$.

For the number of outer-iterations $T$
\begin{align*}
    \gO\parens*{\frac{1}{\sqrt{\eps}} \, \frac{1}{\sqrt{\beta \, T}}} \leq \gO(\eps)
    \implies T \geq \gO\parens*{\frac{1}{\beta \, \eps^3}} = \gO\parens*{\frac{1}{\eps^{\frac{4n + 3}{n}}}} \, .
\end{align*}
Since $n = 2^K$, setting $\beta = \gO\parens*{\eps^{\frac{3 + 2^K}{2^K}}}$ and $T = \gO\parens*{\frac{1}{\eps^{4 + \frac{3}{2^K}}}}$, we have
\begin{align*}
V^{\pi^*}_r(\rho) - V^{\pistarblT}_{r}(\rho) \leq \gO(\eps) \quad \text{and} \quad \max\braces*{b_i - V^{\pistarblT}_{c_i}(\rho), 0} \leq \gO(\eps) \quad \forall i \in [m] \, .
\end{align*}

Next, we verify that~\cref{eq:fa_invariant_assumption_smooth} and~\cref{eq:basin_condition} hold. 
With the above choices
\begin{equation*}
\frac{2 \, m \, \beta + 4 \, \eps_{\cC}}{\tau} = \gO\parens*{\eps^{3/n} + \eps^3}   
\end{equation*}
Assuming $\eps$ sufficiently small, such that
\begin{equation*}
    x \coloneqq \frac{\gamma}{1 - \gamma} \, \frac{2 \, m \, \beta + 4 \, \eps_{\cC}}{\tau \, (1 - \gamma)^2} \leq \frac{1}{10} \, .
\end{equation*}
This implies~\cref{eq:fa_invariant_assumption_smooth}, as
\begin{align*}
    2^n \parens*{\frac{\gamma}{1 - \gamma}}^{n-1}
    \parens*{\frac{2 \, m \, \beta + 4 \, \eps_{\cC}}{\tau \, (1 - \gamma)^2}}^{n-1} \leq 2^n \, 10^{-(n-1)} \leq 1 \, \tag{since $n = 2^K \geq 2$}.
\end{align*}
Finally, $B \leq R$ in~\cref{eq:basin_condition} is equivalent to
\begin{equation*}
\frac{\gamma}{1 - \gamma} \parens*{B + \frac{2 \, m \, \beta}{\tau \, (1 - \gamma)^2}} + \frac{2 \, \gamma \, \eps_{\cC}}{\tau \, (1 - \gamma)^3} \leq \frac{1}{4} \, .
\end{equation*}
By the definition of $B$ and~\cref{eq:inv_step_condition},
\begin{align*}
B + \frac{2 \, m \, \beta}{\tau \, (1 - \gamma)^2}
&= \frac{2 \, m \, \beta + 4 \, \eps_{\cC}}{\tau \, (1 - \gamma)^2}
   \bracks*{2^n \parens*{\frac{\gamma}{1 - \gamma}}^{n-1}
   \parens*{\frac{2 \, m \, \beta + 4 \, \eps_{\cC}}{\tau \, (1 - \gamma)^2}}^{n-1} + 1}
\leq 2 \, \frac{2 \, m \, \beta + 4 \, \eps_{\cC}}{\tau \, (1 - \gamma)^2} \, ,
\end{align*}
so multiplying by $\frac{\gamma}{1 - \gamma} > 0$ gives
$\frac{\gamma}{1 - \gamma} \parens*{B + \frac{2 \, m \, \beta}{\tau \, (1 - \gamma)^2}} \leq 2 \, x$
Moreover,
\begin{equation*}
\frac{2 \, \gamma \, \eps_{\cC}}{\tau \, (1 - \gamma)^3}
= \frac{\gamma}{1 - \gamma} \, \frac{2 \, \eps_{\cC}}{\tau \, (1 - \gamma)^2}
\leq \frac{\gamma}{1 - \gamma} \, \frac{m \, \beta + 2 \, \eps_{\cC}}{\tau \, (1 - \gamma)^2}
= \frac{x}{2} \, .
\end{equation*}
Hence the left-hand side is at most $2 \, x + \frac{x}{2} = \frac{5 \, x}{2} \leq \frac{1}{4}$.

Then by~\cref{proposition:transfer_rate}, applying \texttt{Q-E-NPG} update to the pseudo-reward $R(\bar{\lambda}_T)$ recovers a policy $\pi_N$ that is also $\gO(\eps)$-optimal after $N=\gO\parens*{\ln\parens*{\nicefrac{1}{\eps}}}$ iterations. The total number of gradient evaluations follows the same argument as~\cref{theorem:smooth_last_iterate_K_geq_1}.
\end{proof}

\subsection{Faster Convergence Results Using $\ell_2$-regularization}
\begin{corollary}\label{corollary:smooth_l2_last_iterate_K_ln_eps_FA}
Under~\cref{assumption:slater_condition,assumption:q_realizability,assumption:eps_opt}, for a given $\eps \in (0, 1)$ assuming $\eps_{\bias} \leq \gO(\eps^{8})$ and $\eps_{\opt} \leq \gO(\eps^{16})$,~\cref{algorithm:dd} with (i) $\tau = \mu = \gO(\eps)$, $T = \gO(\nicefrac{1}{\eps^2})$, $\beta = \gO(\eps)$, $\lambda_0 = 0$ and (ii) using the \texttt{Q-E-NPG} update in~\cref{eq:q_npg_update} with $\eta = \frac{1 - \gamma}{\tau}$ and $K = J = \gO(\ln(\nicefrac{1}{\eps}))$ returns $\gO(\eps)$-optimal policy for~\cref{problem:cmdp} after $\widetilde{\gO}\parens*{\nicefrac{1}{\eps^2}}$ gradient evaluations.
\end{corollary}
\begin{proof}
Following the initial proof of~\cref{corollary:smooth_last_iterate_K_ln_eps_FA}, using \texttt{Q-E-NPG} update in~\cref{eq:q_npg_update} with $\eta = \frac{1-\gamma}{\tau}$ returns a policy such that by~\cref{eq:q_npg_gradient_bound}, for all $t \in [T]$
\begin{equation*}
   \norm{\dgt - \bdgt} \leq \gO\parens*{\frac{\tau \, \mu}{T^2} + \frac{\sqrt{\eps_{\cC}}}{\tau}} \coloneqq \delta_t 
\end{equation*}
after $K = \gO(\ln(\nicefrac{T^4}{\tau^3\, \mu^2}))$ iterations.
Furthermore, by~\cref{eq:l2_smooth_bound} of~\cref{theorem:reg_smoothness_l2_rate},~\cref{algorithm:dd} using step-size $\beta = \frac{2 \, \tau \, (1 - \gamma)^3}{m + 2 \, \mu \, \tau \, (1 - \gamma)^3}$ and $\lambda_0$ results in the following convergence: for $T \geq 1$
\begin{align*}
\norm{\lambda_T - \lstarmu}  &\leq \exp\parens*{-\frac{\beta \, \mu \, T}{2}} \, \norm{\lstarmu} +  \frac{2 \, \hat{\delta}}{\mu} \\ 
&\leq \exp\parens*{-\frac{\beta \, \mu \, T}{2}} \, \norm{\lstarmu} + \gO\parens*{\frac{\tau}{T^2} + \frac{\sqrt{\eps_{\cC}}}{\tau \, \mu}}
\end{align*}
Combining the above result with~\cref{lemma:reg_dual_distance_to_sub}, we have
\begin{align*}
    V^{\pi^*}_{r}(\rho) - V^{\pi^*_{\lambda_T, \tau}}_{r}(\rho) &\leq \gO\parens*{\frac{1}{\tau} \exp\parens*{-\frac{\beta \, \mu \, T}{2}}} + \gO\parens*{\frac{1}{T^2}} + \gO\parens*{\frac{\sqrt{\eps_{\cC}}}{\tau^2 \, \mu}} + \gO\parens*{\tau} \, , \\
    \max\braces*{b_i - V^{\pi^*_{\lambda_T, \tau}}_{c_i}(\rho), 0} &\leq \gO\parens*{\frac{1}{\tau} \exp\parens*{-\frac{\beta \, \mu \, T}{2}}} + \gO\parens*{\frac{1}{T^2}} + \gO\parens*{\frac{\sqrt{\eps_{\cC}}}{\tau^2 \, \mu}} + \gO(\mu) \quad \forall i \in [m] \, .
\end{align*}
Assuming that $\eps_{\bias} \leq \gO(\eps^{8})$ and $\eps_{\opt} \leq \gO(\eps^{16})$, setting $\tau = \mu = \gO(\eps)$, and $T = \tilde{\gO}(\nicefrac{1}{\eps^2})$  ensures
\begin{align*}
V^{\pi^*}_r(\rho) - V^{\pi^*_{\lambda_T, \tau}}_{r}(\rho) \leq \gO(\eps) \quad \text{ and } \max\braces*{b_i - V^{\pi^*_{\lambda_T, \tau}}_{c_i}(\rho), 0} \leq \gO(\eps) \quad \forall i \in [m] \, .
\end{align*}
Then by~\cref{proposition:transfer_rate_fa}, applying \texttt{Q-E-NPG} update to the pseudo-reward $R(\lambda_T)$ recovers a policy $\pi_N$ that is also $\gO(\eps)$-optimal after $N=\gO\parens*{\ln\parens*{\nicefrac{1}{\eps}}}$ iterations.
\end{proof}

\begin{restatable}{corollary}{smoothscddlastitkoneeratecorollary}\label{corollary:smooth_l2_last_iterate_K_1_FA}
Under~\cref{assumption:slater_condition,assumption:q_realizability,assumption:eps_opt}, for a given $\eps \in (0, 1)$ assuming $\eps_{\bias} \leq \gO(\eps^4)$ and $\eps_{\opt} \leq \gO(\eps^8)$,~\cref{algorithm:dd} with (i) $\tau = \mu = \gO(\eps)$, $T = \gO(\nicefrac{1}{\eps^5})$, $\beta = \gO(\eps^4)$, $\lambda_0 = 0$ and (ii) using the \texttt{Q-E-NPG} update in~\cref{eq:q_npg_update} with $\eta = \frac{1 - \gamma}{\tau}$ and $K = 1$, $J = \gO(\ln(\nicefrac{1}{\eps}))$ returns an $\gO(\eps)$-optimal policy for~\cref{problem:cmdp} after $\widetilde{\gO}\parens*{\nicefrac{1}{\eps^5}}$ gradient evaluations.
\end{restatable}
\begin{proof}
This proof proceeds similarly to~\cref{theorem:l2_last_iterate_K_1}. Starting from~\cref{eq:smoothness_l2_recursion_phi} we have
\begin{align*}
    \norm{\lambda_T - \lstarmu} &\leq \exp\parens*{-\frac{\beta \, \mu \, T}{2}} \, \norm{\lstarmu} + \frac{4 \, \sqrt{m} \, A \, \beta}{\tau \, (1 - \gamma)^2} \sum_{t=0}^{T-1} \parens*{1 - \frac{\beta \, \mu}{2}}^{T - t - 1} \, \Phi_t \\
    \numberthis \label{eq:smoothness_l2_recursion_phi_fa}
\end{align*}
Furthermore by~\cref{eq:phi_recurse_fa} and~\cref{eq:lt_ltt_mu_ub}
\begin{align*}
    \Phi_{t+1} &\leq \frac{1+\gamma}{2} \, \Phi_t + \underbrace{\frac{m \, (1 + \gamma)}{(1 - \gamma)^2}  \parens*{\frac{\sqrt{m}}{1 - \gamma} + \mu \, \abs*{\Lambda}}}_{\coloneqq C_\mu}\, \beta + \frac{(3 - \gamma) \, \eps_\cC}{(1 - \gamma)^3} 
\end{align*}
Let $\kappa \coloneqq \frac{1+\gamma}{2}$ and $\alpha \coloneqq 1 - \frac{\mu \, \beta}{2}$. Since $\beta = \gO(\eps^4)$ and $\mu = \gO(\eps)$, for $\eps$ sufficiently small $\mu \, \beta \leq \frac{1 - \gamma}{2}$, so that
\begin{equation*}
    \alpha - \kappa = \frac{1 - \gamma}{2} - \frac{\mu \, \beta}{2} \geq \frac{1 - \gamma}{4} \, , \qquad 1 - \kappa = \frac{1 - \gamma}{2} \, , \qquad 1 - \alpha = \frac{\mu \, \beta}{2} \, .
\end{equation*}
Hence, by~\cref{eq:x-weighted} of~\cref{proposition:recursion-const} with $\nu = C_\mu \, \beta + \frac{(3 - \gamma) \, \eps_\cC}{(1 - \gamma)^3}$ and since $\alpha^T \leq \exp\parens*{-\frac{\mu \, \beta \, T}{2}}$,
\begin{align*}
    \sum_{t=0}^{T-1} \alpha^{T-1-t} \, \Phi_t
    &\leq \frac{\alpha^T \, \Phi_0}{\alpha - \kappa} + \frac{\nu}{(1 - \kappa) \, (1 - \alpha)} \\
    &\leq \frac{4 \, \Phi_0}{1 - \gamma} \exp\parens*{-\frac{\mu \, \beta \, T}{2}} + \frac{4 \, C_\mu}{\mu \, (1 - \gamma)} + \frac{4 \, (3 - \gamma) \, \eps_\cC}{\mu \, \beta \, (1 - \gamma)^4} \\
    &= \gO\parens*{\exp\parens*{-\frac{\mu \, \beta \, T}{2}}} + \gO\parens*{\frac{1}{\mu}} + \gO\parens*{\frac{\eps_\cC}{\mu \, \beta}} \, . \numberthis \label{eq:smoothness_l2_recursion_phi_bound_fa}
\end{align*}
Putting~\cref{eq:smoothness_l2_recursion_phi_fa,eq:smoothness_l2_recursion_phi_bound_fa} together,
\begin{align*}
\norm{\lambda_T - \lstarmu} &\leq \gO\parens*{\exp\parens*{-\frac{\mu \, \beta \, T}{2}}} + \gO\parens*{\frac{\beta}{\tau} \exp\parens*{-\frac{\mu \, \beta \, T}{2}}} + \gO\parens*{\frac{\beta}{\tau \, \mu}} + \gO\parens*{\frac{\eps_\cC}{\tau \, \mu}} \, .
\end{align*}
Using the above bound with~\cref{lemma:reg_dual_distance_to_sub}, where $G_\tau = \gO(\nicefrac{1}{\tau})$,
\begin{align*}
    V^{\pi^*}_{r}(\rho) - V^{\pi^*_{\lambda_T, \tau}}_{r}(\rho) &\leq \gO\parens*{\frac{1}{\tau} \, \exp\parens*{-\frac{\mu \, \beta \, T}{2}}} + \gO\parens*{\frac{\beta}{\tau^2} \exp\parens*{-\frac{\mu \, \beta \, T}{2}}} + \gO\parens*{\frac{\beta}{\tau^2 \, \mu}} + \gO\parens*{\frac{\eps_\cC}{\tau^2 \, \mu}} + \gO(\tau) \, , \\
    \max\braces*{b_i - V^{\pi^*_{\lambda_T, \tau}}_{c_i}(\rho), 0} &\leq \gO\parens*{\frac{1}{\tau} \, \exp\parens*{-\frac{\mu \, \beta \, T}{2}}} + \gO\parens*{\frac{\beta}{\tau^2} \exp\parens*{-\frac{\mu \, \beta \, T}{2}}} + \gO\parens*{\frac{\beta}{\tau^2 \, \mu}} + \gO\parens*{\frac{\eps_\cC}{\tau^2 \, \mu}} + \gO(\mu)  \quad \forall i \in [m] \, .
\end{align*}
Since $\eps_\cC = \gO(\eps_{\bias} + \sqrt{\eps_{\opt}})$, assuming $\eps_{\bias} \leq \gO(\eps^4)$ and $\eps_{\opt} \leq \gO(\eps^8)$, setting $\tau = \mu = \gO(\eps)$, $\beta = \gO(\eps^4)$, and $T = \tilde{\gO}(\nicefrac{1}{\eps^5})$ ensures
\begin{align*}
V^{\pi^*}_r(\rho) - V^{\pi^*_{\lambda_T, \tau}}_{r}(\rho) \leq \gO(\eps) \quad \text{ and } \max\braces*{b_i - V^{\pi^*_{\lambda_T, \tau}}_{c_i}(\rho), 0} \leq \gO(\eps) \quad \forall i \in [m] \, .
\end{align*}
Then by~\cref{proposition:transfer_rate_fa}, applying \texttt{Q-E-NPG} update to the pseudo-reward $R(\lambda_T)$ recovers a policy $\pi_N$ that is also $\gO(\eps)$-optimal after $N=\gO\parens*{\ln\parens*{\nicefrac{1}{\eps}}}$ iterations.
\end{proof}

\smoothscddlastitkgeqoneratecorollary*
\begin{proof}
This proof proceeds similarly to~\cref{theorem:l2_last_iterate_K_geq_1}. Starting from~\cref{eq:reg_gd_L} we have
\begin{align}
   \norm{\lambda_T - \lstarmu} &\leq \exp\parens*{-\frac{\beta \, \mu \, T}{2}} \, \norm{\lstarmu} + \frac{2 \, \sqrt{m} \, A }{\mu \, (1 - \gamma)^2} \max_t L_t
    \numberthis \label{eq:reg_gd_L_fa}
\end{align}
We next show that for all $t \in [0, T-1]$,
\begin{equation}
    L_t \leq B \coloneqq 2^n \, \parens*{\frac{\gamma}{1 - \gamma}}^{n-1} \, \parens*{\frac{2 \, \sqrt{m} \, \beta}{\tau \, (1 - \gamma)} \parens*{\frac{\sqrt{m}}{1 - \gamma} + \mu \, \abs{\Lambda}} + \frac{4 \, \eps_{\cC}}{\tau \, (1 - \gamma)^2}}^n + \frac{4 \, \eps_{\cC}}{\tau \, (1 - \gamma)^2}
\end{equation}
where $\eps_{\cC} \coloneqq \eps_{\bias}(1 + \sqrt{d}) + \sqrt{d \, \eps_{\opt}}$. 
To do so, we assume that the initial policy satisfies
\begin{equation*}
    \supnorm{\ln(\pi^*_0) - \ln(\pi_0)} = L_0 \leq B \, ,
\end{equation*}
where $B$ is within the following basin,
\begin{equation}\label{eq:basin_condition_l2}
    B \leq R \coloneqq \frac{1- \gamma}{\gamma} \bracks*{\frac{1}{4} - \frac{2 \, \gamma \, \eps_{\cC}}{\tau \, (1 - \gamma)^3}} - \frac{2 \, \sqrt{m} \, \beta}{\tau \, (1 - \gamma)} \parens*{\frac{\sqrt{m}}{1 - \gamma} + \mu \, \abs{\Lambda}}
\end{equation}
and that
\begin{equation}\label{eq:fa_invariant_assumption_l2}
    2^n \parens*{\frac{\gamma}{1 - \gamma}}^{n-1}
    \parens*{\frac{2 \, \sqrt{m} \, \beta}{\tau \, (1 - \gamma)} \parens*{\frac{\sqrt{m}}{1 - \gamma} + \mu \, \abs{\Lambda}} + \frac{4 \, \eps_{\cC}}{\tau \, (1 - \gamma)^2}}^{n-1} \leq 1 \, .
\end{equation}
Suppose that there exists an $t$ such that $L_t \leq B$. Then at iteration $t+1$, by~\cref{corollary:ineact_npg_coreset} with $\pi_0 = \pi_{t}$, running $K \geq 1$ iterations of \texttt{Q-E-NPG} update in~\cref{eq:q_npg_update} with $\eta = \frac{1 - \gamma}{\tau}$ produces a sequence of policies $\{\hat{\pi}_k\}_{k=0}^{K}$ such that for $k \in [0, K-1]$,
\begin{align}
\supnorm{\ln(\pi^*_{t+1}) - \ln(\hat{\pi}_{k+1})} \leq \frac{2}{\tau} \parens*{\supnorm{Q^{\pi^*_{t+1}}_{t+1} - Q^{\hat{\pi}_k}_{t+1}} + \frac{\eps_{\cC}}{(1 - \gamma)^2}} \, ,
\end{align}
where $\hat{\pi}_0 = \pi_t$ and $\hat{\pi}_{K} = \pi_{t+1}$.
In order to use the above result to show a contraction using~\cref{lemma:genericLrecurse}, we require that the following condition holds:
\begin{equation*}
    \frac{\gamma}{1 - \gamma} \supnorm{\ln(\pi^*_{t+1}) - \ln(\pi_t)}  + \frac{2 \, \gamma \, \eps_{\cC}}{\tau \, (1 - \gamma)^3} \leq \frac{1}{4}
\end{equation*}
Compared to the case when $\mu = 0$, the dual drift now includes the regularization term. 
By~\cref{eq:ln_star_tt_ln_star_t_ub,eq:lt_ltt_mu_ub}, we hve
\begin{equation}
    \supnorm{\ln(\pi^*_{t+1}) - \ln(\pi^*_t)} \leq \frac{2 \, \sqrt{m}}{\tau \, (1 - \gamma)} \, \norm{\ltt - \lt}
    \leq \frac{\beta}{\tau} \, \frac{2 \, \sqrt{m}}{1 - \gamma} \parens*{\frac{\sqrt{m}}{1 - \gamma} + \mu \, \abs{\Lambda}}
    \label{eq:ln_star_drift_mu}
\end{equation}
Therefore, using the above relation,
\begin{align*}
&\frac{\gamma}{1 - \gamma} \, \supnorm{\ln(\pi^*_{t+1}) - \ln(\pi_t)} + \frac{2 \, \gamma \, \eps_{\cC}}{\tau \, (1 - \gamma)^3} \\
&\leq \frac{\gamma}{1 - \gamma} \parens*{\frac{2 \, \sqrt{m} \, \beta}{\tau \, (1 - \gamma)} \parens*{\frac{\sqrt{m}}{1 - \gamma} + \mu \, \abs{\Lambda}} + B} + \frac{2 \, \gamma \, \eps_{\cC}}{\tau \, (1 - \gamma)^3}
\\
&\leq \frac{\gamma}{1 - \gamma} \, \frac{2 \, \sqrt{m} \, \beta}{\tau \, (1 - \gamma)} \parens*{\frac{\sqrt{m}}{1 - \gamma} + \mu \, \abs{\Lambda}} + \frac{\gamma}{1 - \gamma} \, R + \frac{2 \, \gamma \, \eps_{\cC}}{\tau \, (1 - \gamma)^3} \tag{Since $B \leq R$ and $\frac{\gamma}{1 - \gamma} > 0$} \\
&= \frac{\gamma}{1 - \gamma} \, \frac{2 \, \sqrt{m} \, \beta}{\tau \, (1 - \gamma)} \parens*{\frac{\sqrt{m}}{1 - \gamma} + \mu \, \abs{\Lambda}}
    + \bracks*{\frac{1}{4} - \frac{2 \, \gamma \, \eps_{\cC}}{\tau \, (1 - \gamma)^3} - \frac{\gamma}{1 - \gamma} \, \frac{2 \, \sqrt{m} \, \beta}{\tau \, (1 - \gamma)} \parens*{\frac{\sqrt{m}}{1 - \gamma} + \mu \, \abs{\Lambda}}} + \frac{2 \, \gamma \, \eps_{\cC}}{\tau \, (1 - \gamma)^3}
    \tag{Definition of $R$} \\
   &= \frac{1}{4} 
\end{align*}
Hence by~\cref{lemma:genericLrecurse} with $E = \frac{\eps_{\cC}}{(1 - \gamma)^2}$,
\begin{align*}
L_{t+1} &\leq  \parens*{\frac{\gamma}{1 - \gamma}}^{n-1} \, \parens*{\frac{2 \, \sqrt{m}}{\tau \, (1- \gamma)} \, \norm{\ltt - \lt} + L_t }^n + \frac{4 \, \eps_{\cC}}{\tau \, (1 - \gamma)^2}  \\
&\leq \parens*{\frac{\gamma}{1 - \gamma}}^{n-1} \, \parens*{\frac{2 \, \sqrt{m} \, \beta}{\tau \, (1 - \gamma)} \parens*{\frac{\sqrt{m}}{1 - \gamma} + \mu \, \abs{\Lambda}} + L_t }^n + \frac{4 \, \eps_{\cC}}{\tau \, (1 - \gamma)^2} \tag{Using~\cref{eq:lt_ltt_mu_ub}} 
\end{align*}
Using the above contraction, we now show that $L_t \leq B$, for all $t \geq 0$. By the definition of $B$ and~\cref{eq:fa_invariant_assumption_l2}
\begin{align*}
&B + \frac{2 \, \sqrt{m} \, \beta}{\tau \, (1 - \gamma)} \parens*{\frac{\sqrt{m}}{1 - \gamma} + \mu \, \abs{\Lambda}} \\
&= \parens*{\frac{2 \, \sqrt{m} \, \beta}{\tau \, (1 - \gamma)} \parens*{\frac{\sqrt{m}}{1 - \gamma} + \mu \, \abs{\Lambda}} + \frac{4 \, \eps_{\cC}}{\tau \, (1 - \gamma)^2}}
   \bracks*{2^n \parens*{\frac{\gamma}{1 - \gamma}}^{n-1}
   \parens*{\frac{2 \, \sqrt{m} \, \beta}{\tau \, (1 - \gamma)} \parens*{\frac{\sqrt{m}}{1 - \gamma} + \mu \, \abs{\Lambda}} + \frac{4 \, \eps_{\cC}}{\tau \, (1 - \gamma)^2}}^{n-1} + 1} \\
&\leq 2 \, \parens*{\frac{2 \, \sqrt{m} \, \beta}{\tau \, (1 - \gamma)} \parens*{\frac{\sqrt{m}}{1 - \gamma} + \mu \, \abs{\Lambda}} + \frac{4 \, \eps_{\cC}}{\tau \, (1 - \gamma)^2}} \numberthis \label{eq:inv_step_condition_l2}\, .
\end{align*}
Therefore, since $x \mapsto x^n$ is non-decreasing on $x \geq 0$,
\begin{align*}
&\parens*{\frac{\gamma}{1 - \gamma}}^{n-1} \parens*{B + \frac{2 \, \sqrt{m} \, \beta}{\tau \, (1 - \gamma)} \parens*{\frac{\sqrt{m}}{1 - \gamma} + \mu \, \abs{\Lambda}}}^n
+ \frac{4 \, \eps_{\cC}}{\tau \, (1 - \gamma)^2} \\
&\leq 2^n \parens*{\frac{\gamma}{1 - \gamma}}^{n-1}
   \parens*{\frac{2 \, \sqrt{m} \, \beta}{\tau \, (1 - \gamma)} \parens*{\frac{\sqrt{m}}{1 - \gamma} + \mu \, \abs{\Lambda}} + \frac{4 \, \eps_{\cC}}{\tau \, (1 - \gamma)^2}}^{n}
   + \frac{4 \, \eps_{\cC}}{\tau \, (1 - \gamma)^2} = B \, .
\end{align*}
Since the conditions of~\cref{lemma:invariant_recursion} are satisfied with $\kappa = \parens*{\frac{\gamma}{1 - \gamma}}^{n-1}$, $\nu = \frac{2 \, \sqrt{m} \, \beta}{\tau \, (1 - \gamma)} \parens*{\frac{\sqrt{m}}{1 - \gamma} + \mu \, \abs{\Lambda}}$ and $e = \frac{4 \, \eps_{\cC}}{\tau \,  (1 - \gamma)^2}$, we have for all $t \geq 0$,
\begin{equation*}
    L_t \leq B
\end{equation*}
Combining the above result with~\cref{eq:reg_gd_L_fa}, and since $\mu \, \abs{\Lambda} = \gO(1)$,
\begin{align*}
   \norm{\lambda_T - \lstarmu} &\leq \exp\parens*{-\frac{\beta \, \mu \, T}{2}} \, \norm{\lstarmu} + \frac{2 \, \sqrt{m} \, A }{\mu \, (1 - \gamma)^2} \, B \\
   &\leq \gO\parens*{\exp\parens*{-\frac{\beta \, \mu \, T}{2}}} + \gO\parens*{\frac{1}{\mu}\parens*{\frac{\beta + \eps_{\cC}}{\tau}}^n} + \gO\parens*{\frac{\eps_{\cC}}{\tau \, \mu}} \,.
\end{align*}
Combining the above bound with~\cref{lemma:reg_dual_distance_to_sub}, where $G_\tau = \gO(\nicefrac{1}{\tau})$,
\begin{align*}
V^{\pi^*}_r(\rho) - V^{\pi^*_{\lambda_T, \tau}}_{r}(\rho)
&\leq \gO\parens*{\frac{1}{\tau} \exp\parens*{-\frac{\beta \, \mu \, T}{2}}} + \gO\parens*{\frac{1}{\tau \, \mu}\parens*{\frac{\beta + \eps_{\cC}}{\tau}}^n} + \gO\parens*{\frac{\eps_{\cC}}{\tau^2 \, \mu}} + \gO(\tau) \, , \\
\max\braces*{b_i - V^{\pi^*_{\lambda_T, \tau}}_{c_i}(\rho), 0}
&\leq \gO\parens*{\frac{1}{\tau} \exp\parens*{-\frac{\beta \, \mu \, T}{2}}} + \gO\parens*{\frac{1}{\tau \, \mu}\parens*{\frac{\beta + \eps_{\cC}}{\tau}}^n} + \gO\parens*{\frac{\eps_{\cC}}{\tau^2 \, \mu}} + \gO(\mu) \quad \forall i \in [m] \, .
\end{align*}
To obtain an $\gO(\eps)$-optimal policy, we require each term to be $\gO(\eps)$ and hence, set $\tau = \mu = \gO(\eps)$.
For the bound the approximation error $\eps_{\cC}$,
\begin{align*}
    \gO\parens*{\frac{\eps_{\cC}}{\eps^3}} \leq \gO(\eps)
    \implies \eps_{\cC} \leq \gO\parens*{\eps^4} \, ,
\end{align*}
which since $\eps_{\cC} = \eps_{\bias}(1 + \sqrt{d}) + \sqrt{d \, \eps_{\opt}}$, holds when 
$\eps_{\bias} \leq \gO\parens*{\eps^4}$ and $\eps_{\opt} \leq \gO\parens*{\eps^8}$. \\
For the step-size $\beta$,
\begin{align*}
    \gO\parens*{\frac{1}{\eps^2} \, \parens*{\frac{\beta + \eps_{\cC}}{\eps}}^{n}} \leq \gO(\eps)
    \implies \beta + \eps_{\cC} \leq \gO\parens*{\eps^{\frac{3 + n}{n}}} \, ,
\end{align*}
and since $\eps_{\cC} \leq \gO(\eps^4) \leq \gO\parens*{\eps^{\frac{3 + n}{n}}}$ for $n \geq 1$, it suffices that
$\beta \leq \gO\parens*{\eps^{\frac{3 + n}{n}}}$.

For the number of outer-iterations $T$
\begin{align*}
    \gO\parens*{\frac{1}{\eps} \exp\parens*{-\frac{\beta \, \mu \, T}{2}}} \leq \gO(\eps)
    \impliedby T \geq \frac{2 \ln(\nicefrac{1}{\eps^2})}{\beta \, \mu} = \tilde{\gO}\parens*{\frac{1}{\eps^{\frac{2n + 3}{n}}}} \, .
\end{align*}
Since $n = 2^K$, setting $\beta = \gO\parens*{\eps^{\frac{3 + 2^K}{2^K}}}$ and $T = \tilde{\gO}\parens*{\frac{1}{\eps^{2 + \frac{3}{2^K}}}}$, we have
\begin{align*}
V^{\pi^*}_r(\rho) - V^{\pi^*_{\lambda_T, \tau}}_{r}(\rho) \leq \gO(\eps) \quad \text{and} \quad \max\braces*{b_i - V^{\pi^*_{\lambda_T, \tau}}_{c_i}(\rho), 0} \leq \gO(\eps) \quad \forall i \in [m] \, .
\end{align*}
Moreover, $\beta \leq \frac{2}{G_\tau + 2 \mu}$ for $\eps$ sufficiently small, as required by~\cref{lemma:reg_dual_linear_convergence}.

Next, we verify that~\cref{eq:fa_invariant_assumption_l2} and~\cref{eq:basin_condition_l2} hold. 
With the above choices, 
\begin{equation*}
\frac{2 \, \sqrt{m} \, \beta}{\tau \, (1 - \gamma)} \parens*{\frac{\sqrt{m}}{1 - \gamma} + \mu \, \abs{\Lambda}} + \frac{4 \, \eps_{\cC}}{\tau \, (1 - \gamma)^2} = \gO\parens*{\eps^{3/n} + \eps^3}   
\end{equation*}
Assuming $\eps$ sufficiently small, such that
\begin{equation*}
    x \coloneqq \frac{\gamma}{1 - \gamma} \parens*{\frac{2 \, \sqrt{m} \, \beta}{\tau \, (1 - \gamma)} \parens*{\frac{\sqrt{m}}{1 - \gamma} + \mu \, \abs{\Lambda}} + \frac{4 \, \eps_{\cC}}{\tau \, (1 - \gamma)^2}} \leq \frac{1}{10} \, .
\end{equation*}
This implies~\cref{eq:fa_invariant_assumption_l2}, as
\begin{align*}
    2^n \parens*{\frac{\gamma}{1 - \gamma}}^{n-1}
    \parens*{\frac{2 \, \sqrt{m} \, \beta}{\tau \, (1 - \gamma)} \parens*{\frac{\sqrt{m}}{1 - \gamma} + \mu \, \abs{\Lambda}} + \frac{4 \, \eps_{\cC}}{\tau \, (1 - \gamma)^2}}^{n-1} = 2^n \, x^{n-1} \leq 2^n \, 10^{-(n-1)} \leq 1 \, \tag{since $n = 2^K \geq 2$}.
\end{align*}
Finally, $B \leq R$ in~\cref{eq:basin_condition_l2} is equivalent to
\begin{equation*}
\frac{\gamma}{1 - \gamma} \parens*{B + \frac{2 \, \sqrt{m} \, \beta}{\tau \, (1 - \gamma)} \parens*{\frac{\sqrt{m}}{1 - \gamma} + \mu \, \abs{\Lambda}}} + \frac{2 \, \gamma \, \eps_{\cC}}{\tau \, (1 - \gamma)^3} \leq \frac{1}{4} \, .
\end{equation*}
By~\cref{eq:inv_step_condition_l2}, multiplying by $\frac{\gamma}{1 - \gamma} > 0$ gives
$\frac{\gamma}{1 - \gamma} \parens*{B + \frac{2 \, \sqrt{m} \, \beta}{\tau \, (1 - \gamma)} \parens*{\frac{\sqrt{m}}{1 - \gamma} + \mu \, \abs{\Lambda}}} \leq 2 \, x$. 
Moreover,
\begin{equation*}
\frac{2 \, \gamma \, \eps_{\cC}}{\tau \, (1 - \gamma)^3}
= \frac{\gamma}{1 - \gamma} \, \frac{2 \, \eps_{\cC}}{\tau \, (1 - \gamma)^2}
\leq \frac{1}{2} \, \frac{\gamma}{1 - \gamma} \parens*{\frac{2 \, \sqrt{m} \, \beta}{\tau \, (1 - \gamma)} \parens*{\frac{\sqrt{m}}{1 - \gamma} + \mu \, \abs{\Lambda}} + \frac{4 \, \eps_{\cC}}{\tau \, (1 - \gamma)^2}}
= \frac{x}{2} \, .
\end{equation*}
Hence the left-hand side is at most $2 \, x + \frac{x}{2} = \frac{5 \, x}{2} \leq \frac{1}{4}$.

Then by~\cref{proposition:transfer_rate}, applying \texttt{Q-E-NPG} update to the pseudo-reward $R(\bar{\lambda}_T)$ recovers a policy $\pi_N$ that is also $\gO(\eps)$-optimal after $N=\gO\parens*{\ln\parens*{\nicefrac{1}{\eps}}}$ iterations. The total number of gradient evaluations follows the same argument as~\cref{theorem:l2_last_iterate_K_geq_1}.
\end{proof}

\section{Proof of Additional Technical Tools}\label{appendix:tools}
\begin{proposition}\label{proposition:transfer_rate}
For a given $\eps \in (0, 1)$ and $\tau = \gO(\eps)$, suppose there exists an dual variable $\lambda \in \Lambda$ such that $\pi^*_{\lambda, \tau} = \argmax_{\pi} V^{\pi}_{R(\lambda), \tau}$ is an $\gO(\eps)$-optimal policy to~\cref{problem:cmdp}. 
Then the policy recovered using the $\texttt{NPG}$ update in~\cref{eq:entropy_npg_update} with $\diamond = R(\lambda)$ and $\eta = \frac{1 - \gamma}{\tau}$ is an $\gO(\eps)$-optimal policy after $N=\gO(\ln(\nicefrac{1}{\eps})))$ iterations.
\end{proposition}
\begin{proof}
Let $\pi_N$ denote the policy returned by \texttt{NPG}$(\pi_{-1}, R(\lambda), \tau, N)$ after $N$ iterations where $\pi_{-1}$ is the initial policy.
We first consider the sub-optimality gap:
\begin{align*}
    &V^{\pistar}_{r}(\rho) - V^{\pi_N}_r(\rho) \\
    &=  
    V^{\pistar}_{r}(\rho) - V^{\pistarl}_r(\rho) +
    V^{\pistarl}_{r}(\rho) - V^{\pi_N}_{r}(\rho)
    \tag{Add/Subtract $V^{\pistarl}_{r}(\rho)$} \\
    &\leq \gO(\eps) + V^{\pistarl}_{r}(\rho) - V^{\pi_N}_{r}(\rho) \tag{Since $\pistarl$ is an $\gO(\eps)$-optimal policy to~\cref{problem:cmdp}} \\
    &\leq \gO(\eps) + \gO\parens*{\frac{\sqrt{\lag(\pistarl, \lambda) - \lag(\pi_N, \lambda)}}{\sqrt{\tau}}}\\
    \tag{Using~\cref{eq:vdl_upper_bound_lag} from~\cref{lemma:generic_vdl_upper_bound} with $\diamond = r$, $\pi' = \pi^*_{\lambda, \tau}$, and $\pi = \pi_N$} \\
    &= \gO(\eps) + \gO\parens*{\frac{\sqrt{V^{\pistarl}_{R(\lambda), \tau}(\rho) - V^{\pi_N}_{R(\lambda), \tau}(\rho)}}{\sqrt{\tau}}}\\
    \tag{Since $\lag(\pi, \lambda) = V^{\pi}_r(\rho) + \sum_{i=1}^m \lambda_i \, [V^{\pi}_{c_i}(\rho) - b_i]$} \\
    &\leq \gO(\eps) + \gO\parens*{\frac{\sqrt{\gamma^{N-1}}}{\sqrt{\tau}}}    \tag{By~\cref{lemma:entropy_npg_supnorm_convergence} with $\diamond = R(\lambda)$ and $\eta = \frac{1-\gamma}{\tau}$}
\end{align*}
Hence setting $N = \gO(\ln(\nicefrac{1}{\eps}))$ implies
\begin{equation*}
   V^{\pistar}_{r}(\rho) - V^{\pi_N}_r(\rho)  \leq \gO(\eps).
\end{equation*}
Following a similar argument for the constraint violation, we have
\begin{equation*}
  \max\braces*{b_i - V^{\pi_N}_{c_i}(\rho), 0} \leq \gO(\eps)
\end{equation*} 
\end{proof}
\begin{proposition}\label{proposition:transfer_rate_fa}
Under~\cref{assumption:q_realizability,assumption:eps_opt}, for a given $\eps \in (0, 1)$, $\tau = \gO(\eps)$,  suppose that (i) $\eps_{\bias} \leq \gO(\eps^3)$, $\eps_{\opt} \leq \gO(\eps^6)$ (ii) there exists an dual variable $\lambda \in \R^m_+$ such that $\pi^*_{\lambda, \tau} = \argmax_{\pi} V^{\pi}_{R(\lambda), \tau}$ is an $\gO(\eps)$-optimal policy to~\cref{problem:cmdp}. 
Then the policy recovered using the $\texttt{Q-E-NPG}$ update in~\cref{eq:q_npg_update} with $\diamond = R(\lambda)$ and $\eta = \frac{1 - \gamma}{\tau}$ is an $\gO(\eps)$-optimal policy after $N=\gO(\ln(\nicefrac{1}{\eps})))$ iterations.
\end{proposition}
\begin{proof}
Let $\pi_N$ denote the policy returned by \texttt{Q-E-NPG}$(\pi_{-1}, R(\lambda), \tau, N)$ after $N$ iterations where $\pi_{-1}$ is the initial policy.
We first consider the sub-optimality gap. Following the initial proof of~\cref{proposition:transfer_rate}
\begin{align*}
    V^{\pi^*}_r(\rho) - V^{\pi_N}_r(\rho) &\leq \gO(\eps) + \gO\parens*{\frac{\sqrt{V^{\pistarl}_{R(\lambda), \tau}(\rho) - V^{\pi_N}_{R(\lambda), \tau}(\rho)}}{\sqrt{\tau}}}  \\
   &= \gO(\eps) + \gO\parens*{\sqrt{\frac{\gamma^{N-1}}{\tau}}} + \gO\parens*{\sqrt{\frac{\eps_\cC}{\tau}}}  \tag{By~\cref{corollary:ineact_npg_coreset} with $\diamond = R(\lambda)$ and $\eta = \frac{1-\gamma}{\tau}$}
\end{align*}
where $\eps_{\cC} = \gO(\eps_{\bias} + \sqrt{\eps_{\opt}})$.
Assuming that $\eps_{\bias} \leq \gO(\eps^6)$ and $\eps_{\opt} \leq \gO(\eps^6)$, setting $\tau = \gO(\eps)$ and $N = \gO(\ln(\nicefrac{1}{\eps}))$ implies
\begin{equation*}
   V^{\pistar}_{r}(\rho) - V^{\pi_N}_r(\rho)  \leq \gO(\eps).
\end{equation*}
Following a similar argument for the constraint violation, we have
\begin{equation*}
  \max\braces*{b_i - V^{\pi_N}_{c_i}(\rho), 0} \leq \gO(\eps)
\end{equation*} 
\end{proof}

\begin{lemma}\label{lemma:soft_bellman_lipschitz}
For a fixed policy $\pi$ and pseudo-reward $\diamond, \diamond' \in [0, U]$, $U > 0$    
\begin{equation}\label{eq:soft_bellman_lipschitz}
    \supnorm{Q^{\pi}_{\diamond, \tau} - Q^{\pi}_{\diamond', \tau}} \leq \frac{1}{1- \gamma} \supnorm{\diamond - \diamond'} \, .
\end{equation}
\end{lemma}
\begin{proof}
Recall the definition of the soft bellman operator
\begin{align*}
    Q^{\pi}_{\diamond, \tau} - Q^\pi_{\diamond', \tau} &= \gT_{\diamond, \tau}^\pi Q^{\pi}_{\diamond, \tau} - \gT_{\diamond', \tau}^\pi Q^\pi_{\diamond', \tau} \\
    &= (\gT_{\diamond, \tau}^\pi Q^{\pi}_{\diamond, \tau} - \gT_{\diamond, \tau}^\pi Q^{\pi}_{\diamond',\tau}) + (\gT_{\diamond}^\pi Q^{\pi}_{\diamond',\tau} -  \gT_{\diamond',\tau}^\pi Q^\pi_{\diamond',\tau}) \tag{Add/Subtract $\gT_{\diamond, \tau}^\pi Q^{\pi}_{\diamond',\tau}$} \\
    \implies \supnorm{Q^{\pi}_{\diamond, \tau} - Q^\pi_{\diamond',\tau}} &\leq \supnorm{\gT_{\diamond, \tau}^\pi Q^{\pi}_{\diamond, \tau} - \gT_{\diamond, \tau}^\pi Q^{\pi}_{\diamond',\tau}} + \supnorm{\gT_{\diamond, \tau}^\pi Q^{\pi}_{\diamond',\tau} -  \gT_{\diamond',\tau}^\pi Q^\pi_{\diamond',\tau}} \\
    &\leq \gamma \, \supnorm{Q^{\pi}_{\diamond, \tau} - Q^\pi_{\diamond',\tau}} + \supnorm{\gT_{\diamond, \tau}^\pi Q^{\pi}_{\diamond',\tau} -  \gT_{\diamond',\tau}^\pi Q^\pi_{\diamond',\tau}} \tag{$\gamma$-contraction} \\
    &= \gamma \, \supnorm{Q^{\pi}_{\diamond, \tau} - Q^\pi_{\diamond',\tau}} + \supnorm{\diamond - \diamond'} \tag{Definition of soft Bellman operator} \\
    \implies \supnorm{Q^{\pi}_{\diamond, \tau} - Q^\pi_{\diamond',\tau}} &\leq \frac{1}{1 - \gamma} \supnorm{\diamond - \diamond'} \, .
\end{align*}
\end{proof}

\section{Proof of Technical Tools when $K=1$ Inner-Iterations}\label{appendix:tools_k1}
\begin{lemma}\label{lemma:generic_phi_recurse}
Assuming that for all $t \geq 0$ the sequence of policies $\pi_t$, $\pi_{t+1}$ satisfies
\begin{align*}
\supnorm{Q^{\pistartt}_{t+1} - Q^{\pi_{t+1}}_{t+1}} &\leq \gamma^K \, \supnorm{Q^{\pistartt}_{t+1} - Q^{\pi_{t}}_{t+1}} + E \, , \\
\supnorm{\ln(\pi^*_{t+1}) - \ln(\pi_{t+1})} &\leq \frac{2 \, \gamma^{K-1}}{\tau} \supnorm{Q^{\pistartt}_{t+1} - Q^{\pi_{t}}_{t+1}} + \frac{2 \, E}{\tau}  \, ,
\end{align*}
for an error $E \geq 0$, then
\begin{equation*}
\Phi_{t+1} \leq \frac{(1+\gamma)\,\gamma^{K-1}}{2} \, \Phi_t
+ \frac{m \, (1+\gamma) \, \gamma^{K-1}}{(1 - \gamma)^2} \supnorm{\lt - \ltt}
+ \frac{(3 - \gamma) \, E}{2 \, (1 - \gamma)}
\end{equation*}
where $\Phi_t \coloneqq \frac{1}{1 - \gamma} \supnorm{Q^{\pistart}_{t} - Q^{\pi_{t}}_{t}} + \frac{\tau}{4} \supnorm{\ln(\pistart) - \ln(\pi_t)}$ as in~\cref{lemma:bias_bounded_by_phi}.
\end{lemma}
\begin{proof}
For a fixed iteration $t \geq 0$ and entropy-regularization parameter $\tau > 0$, assume that the sequence of policies $\pi_t$, $\pi_{t+1}$ satisfies
\begin{align}
\supnorm{Q^{\pistartt}_{t+1} - Q^{\pi_{t+1}}_{t+1}} &\leq \gamma^K \, \supnorm{Q^{\pistartt}_{t+1} - Q^{\pi_{t}}_{t+1}} + E  \label{eq:phi_bound0} \\
\supnorm{\ln(\pi^*_{t+1}) - \ln(\pi_{t+1})} &\leq \frac{2 \, \gamma^{K-1}}{\tau}  \supnorm{Q^{\pistartt}_{t+1} - Q^{\pi_{t}}_{t+1}} + \frac{2 \, E}{\tau} \label{eq:phi_bound1}
\end{align}
The goal is to derive a recursive bound on the optimization error, relating iteration $t+1$ to iteration $t$ with additional drift and error terms.
By~\cref{lemma:qstar_lipschitz} with $\lambda = \ltt$ and $\lambda' = \lt$,
\begin{align*}
\supnorm{Q^{\pistartt}_{t+1} - Q^{\pistart}_{t}} \leq \frac{1}{1 - \gamma} \norm{\lt - \ltt}_1  
\leq \frac{m}{1- \gamma} \supnorm{\lt - \ltt} \numberthis \label{eq:tracking_bound5} .
\end{align*}
Furthermore, by~\cref{lemma:soft_bellman_lipschitz} with $\diamond = R_{t+1}$, $\diamond' = R_{t}$ and $\pi = \pi_t$,
\begin{align*}
\supnorm{Q^{\pi_t}_{t+1} - Q^{\pi_t}_{t}} \leq \frac{1}{1- \gamma} \supnorm{\gammat - \gammatt} \, .
\end{align*}
For a fixed $(s, a) \in \cS \times \cA$ pair,
\begin{align*}
\abs{[\gammat - \gammatt](s, a)} &= \abs*{\sum_{i=1}^m c_i(s, a) \, (\lt[i] - \ltt[i])} \leq\sum_{i=1}^m c_i(s, a) \, \abs*{\lt[i] - \ltt[i]} \tag{Using triangle inequality} \\
&\leq \sum_{i=1}^m \abs*{\lt[i] - \ltt[i]}  \tag{Since $c_i \in [0, 1]^{SA}$} \\
&= \norm{\lt - \ltt}_1 \leq m \, \supnorm{\lt - \ltt} \tag{Since for $x \in \R^m$, $\norm{x}_1 \leq m \, \supnorm{x}$}
\end{align*}
\begin{align*}
\implies  \supnorm{Q^{\pi_t}_{t+1} - Q^{\pi_t}_{t}} &\leq \frac{m}{1- \gamma} \supnorm{\lt - \ltt} \numberthis \label{eq:tracking_bound6}
\end{align*}
Hence, adding and subtracting $Q^{\pistart}_t$ and $Q^{\pi_t}_t$, and using the triangle inequality with~\cref{eq:tracking_bound5,eq:tracking_bound6},
\begin{align*}
\supnorm{Q^{\pistartt}_{t+1} - Q^{\pi_{t}}_{t+1}}
&\leq \supnorm{Q^{\pistartt}_{t+1} - Q^{\pistart}_t} + \supnorm{Q^{\pistart}_{t} - Q^{\pi_t}_t} + \supnorm{Q^{\pi_t}_t - Q^{\pi_{t}}_{t+1}} \\
&\leq \supnorm{Q^{\pistart}_{t} - Q^{\pi_t}_t} + \frac{2 \, m}{1 - \gamma} \supnorm{\lt - \ltt} \, . \numberthis \label{eq:tracking_bound_mid}
\end{align*}
Using~\cref{eq:tracking_bound5,eq:tracking_bound6} to bound the second last term of~\cref{eq:phi_bound0},
\begin{align*}
\implies \underbrace{\supnorm{Q^{\pistartt}_{t+1} - Q^{\pi_{t+1}}_{t+1}}}_{\substack{\text{optimization error at} \\ \text{iteration } t + 1}} &\leq \gamma^K \, \underbrace{\supnorm{Q^{\pistart}_{t} - Q^{\pi_{t}}_{t}}}_{\substack{\text{optimization error at} \\ \text{iteration } t}}  + \underbrace{\frac{2 \, m \, \gamma^K}{1 - \gamma} \, \supnorm{\lt - \ltt}}_{\substack{\text{drift from} \\ \text{iteration } t \to t + 1}} + E
\numberthis \label{eq:tracking_bound7}
\intertext{Moreover, combining~\cref{eq:tracking_bound7} with~\cref{eq:phi_bound1},}
\implies \supnorm{\ln(\pistartt) - \ln(\pi_{t+1})} &\leq \frac{2 \, \gamma^{K-1}}{\tau} \, \supnorm{Q^{\pistart}_{t} - Q^{\pi_{t}}_{t}}  + \frac{4 \, m \, \gamma^K}{(1 - \gamma) \, \tau} \, \supnorm{\lt - \ltt} + \frac{2\, E}{\tau} \numberthis \label{eq:tracking_bound8}
\end{align*}
We now combine the two inequalities. Multiplying~\cref{eq:tracking_bound7} by $\frac{1}{1-\gamma}$ and~\cref{eq:tracking_bound8} by $\frac{\tau}{4}$ and summing,
\begin{align*}
\Phi_{t+1} &\coloneqq \frac{1}{1 - \gamma} \, \supnorm{Q^{\pistartt}_{t+1} - Q^{\pi_{t+1}}_{t+1}} +  \frac{\tau}{4} \, \supnorm{\ln(\pistartt) - \ln(\pi_{t+1})} \\ 
&\leq \gamma^{K-1} \parens*{\frac{\gamma}{1 - \gamma} + \frac{1}{2}} \supnorm{Q^{\pistart}_{t} - Q^{\pi_{t}}_{t}}  + \frac{m \, \gamma^{K-1}}{1 - \gamma} \parens*{\frac{2 \, \gamma}{1-\gamma} + 1} \supnorm{\lt - \ltt}  +  \parens*{\frac{1}{1 - \gamma} + \frac{1}{2}} E \\
&= \frac{(1 + \gamma) \, \gamma^{K-1}}{2 \, (1 - \gamma)} \, \supnorm{Q^{\pistart}_{t} - Q^{\pi_{t}}_{t}}  + \frac{m \, (1 + \gamma) \, \gamma^{K-1}}{(1 - \gamma)^2} \supnorm{\lt - \ltt}  +  \frac{(3 - \gamma) \, E}{2 \, (1 - \gamma)} \\
&= \frac{(1 + \gamma) \, \gamma^{K-1}}{2} \parens*{\frac{1}{1 - \gamma} \, \supnorm{Q^{\pistart}_{t} - Q^{\pi_{t}}_{t}}} + \frac{m \, (1 + \gamma) \, \gamma^{K-1}}{(1 - \gamma)^2} \supnorm{\lt - \ltt}  +  \frac{(3 - \gamma) \, E}{2 \, (1 - \gamma)} \\
&\leq \frac{(1 + \gamma) \, \gamma^{K-1}}{2} \, \parens*{\frac{1}{1 - \gamma} \, \supnorm{Q^{\pistart}_{t} - Q^{\pi_{t}}_{t}} + \frac{\tau}{4} \, \supnorm{\ln(\pistart) - \ln(\pi_t)}} \tag{Since$\frac{\tau}{4} \, \supnorm{\ln(\pistart) - \ln(\pi_t)} \geq 0$} \\ 
&\quad + \frac{m \, (1 + \gamma) \, \gamma^{K-1}}{(1 - \gamma)^2} \supnorm{\lt - \ltt}  +  \frac{(3 - \gamma) \, E}{2 \, (1 - \gamma)}
\end{align*}
\begin{align}
\implies \Phi_{t+1} \leq  \frac{(1 + \gamma) \, \gamma^{K-1}}{2} \, \Phi_t + \frac{m \, (1 + \gamma) \, \gamma^{K-1}}{(1 - \gamma)^2} \supnorm{\lt - \ltt}  +  \frac{(3 - \gamma) \, E}{2 \, (1 - \gamma)}
\end{align}
which completes the proof.
\end{proof}

\begin{proposition}\label{proposition:recursion-const}
Let $\{\Phi_t\}_{t \geq 0}$ be a sequence satisfying, for all $t \geq 0$,
\begin{align}
    \Phi_{t+1} \leq \kappa \, \Phi_t + \nu \,,
    \label{eq:x-recursion-const}
\end{align}
where $\kappa \in [0,1)$, $\nu \geq 0$. Then, 
for all $T \geq 1$,
\begin{align}
    \frac{1}{T} \sum_{t=0}^{T-1} \Phi_t
    \leq \frac{\Phi_0}{(1 - \kappa) \, T} + \frac{\nu}{1 - \kappa} \,.
    \label{eq:x-average}
\end{align}
and 
\begin{align}
    \sum_{t=0}^{T-1} \alpha^{T-1-t} \, \Phi_t
    \leq \frac{\alpha^T}{\alpha - \kappa} \, \Phi_0 + \frac{\nu}{(1 - \kappa)(1 - \alpha)} \,.
    \label{eq:x-weighted}
\end{align}
where $\alpha \in (\kappa, 1)$.
\end{proposition}
\begin{proof}
Recursing~\cref{eq:x-recursion-const} from $t-1$ to $0$, we have
\begin{align*}
    \Phi_t
    &\leq \kappa \, \Phi_{t-1} + \nu 
    \leq \kappa \left( \kappa \, \Phi_{t-2} + \nu \right) + \nu 
    = \kappa^{2} \, \Phi_{t-2} + \nu ( 1 + \kappa ) \\
    &\leq \kappa^{t} \, \Phi_0 +  \nu\sum_{j=0}^{t-1} \kappa^{j} \leq \kappa^{t} \, \Phi_0 +  \nu\sum_{j=0}^{\infty} \kappa^{j} \tag{$\kappa \in [0, 1)$}  \\
    &= \kappa^{t} \, \Phi_0 +  \frac{\nu}{1 - \kappa} \numberthis \label{eq:x-unrolled}
\end{align*}
Multiplying~\cref{eq:x-unrolled} by weights $w_t \geq 0$ and recursing from $t = T-1$ to $0$,
\begin{align}
    \sum_{t=0}^{T-1} w_t \, \Phi_t
    \leq \Phi_0 \sum_{t=0}^{T-1} w_t \, \kappa^t + \frac{\nu}{1 - \kappa} \sum_{t=0}^{T-1} w_t \,.
    \label{eq:x-weighted-generic}
\end{align}
Setting $w_t = 1$ then dividing by $T$ proves~\cref{eq:x-average}.

To prove the second equation, let $w_t = \alpha^{T-1-t}$. Since $\alpha \in (\kappa, 1)$,
\begin{align*}
    \sum_{t=0}^{T-1} \alpha^{T-1-t} \, \kappa^t
    = \frac{\alpha^T - \kappa^T}{\alpha - \kappa}
    \leq \frac{\alpha^T}{\alpha - \kappa} \,,
    \qquad
    \sum_{t=0}^{T-1} \alpha^{T-1-t}
    \leq \sum_{t=0}^{\infty} \alpha^t
    = \frac{1}{1 - \alpha} \,,
\end{align*}
Substituting the above relations into~\cref{eq:x-weighted-generic} gives~\cref{eq:x-weighted}.
\end{proof}

\section{Proof of Technical Tools when Interpolating Between $K=1$ and $K > 1$ Inner-Iterations}\label{appendix:tools_kgeq1}
\begin{proposition}\label{proposition:kl_sq_bound}
Let $p, q \in \Delta_A$ such that $p(a) > 0$ and $q(a) > 0$ for all $a \in \cA$. Assuming $\max_{a} \abs*{\ln\parens*{\frac{p(a)}{q(a)}}} \leq L$, then
\begin{equation*}
    \KL(p \| q) \, \leq \, \frac{L^2}{2} \, .
\end{equation*}
\end{proposition}
\begin{proof}
Let $p, q \in \Delta_A$ such that $p(a) > 0$, $q(a) > 0$ and define $X(a) \coloneqq \ln\parens*{\frac{p(a)}{q(a)}}$ for all $a \in \cA$. Observe that
\begin{equation*}
    \E_{a \sim p}\bracks*{\exp(-X(a))} = \sum_a p(a) \frac{q(a)}{p(a)}  =  \sum_a q(a)  =  1 \, \implies \ln\parens*{\E_{a \sim p}\bracks*{\exp(-X(a))}} = 0. 
\end{equation*}
As a result,  assuming that $\max_{a} \abs*{X(a)} \leq L$, then by Hoeffding's Lemma (\cref{lemma:hoeffding_lemma}) with $y = -1$, $a=-L$, $b=L$, we have,
\begin{align*}
    \E_{a \sim p}[\exp(-X(a) + \E_{a' \sim p}[X(a')])]  &\leq \exp\parens*{\frac{(2L)^2}{8}} = \exp\parens*{\frac{L^2}{2}} \\
    \implies \underbrace{\ln\parens*{\E_{a \sim p}[\exp(-X(a))]}}_{=0}+ \E_{a' \sim p}[X(a')] &\leq \frac{L^2}{2}\tag{Taking the logarithms on both sides} \\
    \implies \KL(p \| q) = \E_{a' \sim p}[X(a')] &\leq \frac{L^2}{2} \, .
\end{align*}
\end{proof}

\begin{lemma}\label{lemma:genericLrecurse}
Fix $t \geq 0$ and $\tau > 0$. Let $\hat{\pi}_0 = \pi_t$ and for $k \in [0,K-1]$, suppose that the policy $\hat{\pi}_{k+1}$ satisfies    
\begin{align*}
\supnorm{\ln(\pi^*_{t+1}) - \ln(\hat{\pi}_{k+1})} &\leq \frac{2}{\tau} \parens*{\supnorm{Q^{\pi^*_{t+1}}_{t+1} - Q^{\hat{\pi}_k}_{t+1}} + E}
\end{align*}
for an error $E \geq 0$. Then, for $\pi_{t+1} = \hat{\pi}_K$ and $n = 2^K$, \\
\textit{(i)}  If $ E = 0$
\begin{equation*}
   \supnorm{\ln(\pistartt) - \ln(\pi_{t+1})} \leq \parens*{\frac{\gamma}{1 - \gamma}}^{n-1} \, \supnorm{\ln(\pistartt) - \ln(\pi_{t})}^{n} 
\end{equation*}
\textit{(ii)}  If $ E > 0$ and assuming $\frac{\gamma}{1 - \gamma} \, \supnorm{\ln(\pistartt) - \ln(\pi_t)} + \frac{2 \,  E \, \gamma}{\tau \, (1 - \gamma)} \leq \frac{1}{4}$
\begin{align*}
    \supnorm{\ln(\pistartt) - \ln(\pi_{t+1})} \leq \parens*{\frac{\gamma}{1 - \gamma}}^{n-1} \, \supnorm{\ln(\pistartt) - \ln(\pi_{t})}^{n} + \frac{4 \, E}{\tau} 
\end{align*}
Consequently, in either case,
\begin{equation}
L_{t+1} \leq  \parens*{\frac{\gamma}{1 - \gamma}}^{n-1} \, \parens*{\frac{2 \, \sqrt{m}}{\tau \, (1- \gamma)} \, \norm{\ltt - \lt} + L_t }^n + \frac{4 \, E}{\tau}
\end{equation}
\end{lemma}
\begin{proof}
Fix $t \geq 0$, $\tau > 0$ and let $\hat{\pi}_0 = \pi_t$. Assuming that for each $k \in [K]$, the policy $\hat{\pi}_{k+1}$ satisfies with error $E \geq 0$,
\begin{align*}
\supnorm{\ln(\pi^*_{t+1}) - \ln(\hat{\pi}_{k+1})} &\leq \frac{2}{\tau} \parens*{\supnorm{Q^{\pi^*_{t+1}}_{t+1} - Q^{\hat{\pi}_k}_{t+1}} + E}
= \frac{2}{\tau} \supnorm{Q^{\pi^*_{t+1}}_{t+1} - Q^{\hat{\pi}_k}_{t+1}} + \frac{2 \, E}{\tau} \\
&\leq \frac{2 \, \gamma}{\tau} \supnorm{V^{\pi^*_{t+1}}_{R_{t+1}, \tau} - V^{\hat{\pi}_k}_{R_{t+1}, \tau}} + \frac{2 \, E}{\tau} \tag{The pseudo-reward are identical} \\
&\leq \frac{2 \, \gamma}{1 - \gamma} \, \max_s \KL(\hat{\pi}_k(\cdot | s) \| \pi^*_{t+1}(\cdot | s)) \underbrace{\sum_{s' \in \cS} d^{\hat{\pi}_k}(s')}_{=1} + \frac{2 \, E}{\tau} \tag{Using~\cref{lemma:soft_suboptimality} with $\pi' = \pi^*_{t+1}$, $\pi = \hat{\pi}_k$ and $\diamond = R_{t+1}$} \\
\intertext{Since $\max_{a} \abs*{\ln\parens*{\hat{\pi}_k(a | s) / \pi^*_{t+1}(a | s)}} \leq \supnorm{\ln(\pi^*_{t+1}) - \ln(\hat{\pi}_{k})}$ for every $s \in \cS$, by~\cref{proposition:kl_sq_bound} with $p = \hat{\pi}_k(\cdot | s)$, $q = \pi^*_{t+1}(\cdot | s)$, and $L = \supnorm{\ln(\pi^*_{t+1}) - \ln(\hat{\pi}_{k})}$,}
\implies \supnorm{\ln(\pi^*_{t+1}) - \ln(\hat{\pi}_{k+1})} &\leq \frac{\gamma}{1 - \gamma} \, \supnorm{\ln(\pi^*_{t+1}) - \ln(\hat{\pi}_{k})}^2 + \frac{2 \, E}{\tau} \\
\implies \frac{\gamma}{1 - \gamma} \, \supnorm{\ln(\pi^*_{t+1}) - \ln(\hat{\pi}_{k+1})} &\leq \parens*{\frac{\gamma}{1 - \gamma}}^{2} \, \supnorm{\ln(\pi^*_{t+1}) - \ln(\hat{\pi}_{k})}^2 + \frac{2 \, E \, \gamma}{\tau \, (1- \gamma)}\, \numberthis \label{eq:quad_recurse} 
\end{align*}
If $E=0$, using~\cref{eq:quad_recurse} and recursing from $k=K-1$ to $0$, we have
\begin{equation*}
   \frac{\gamma}{1 - \gamma} \, \supnorm{\ln(\pi^*_{t+1}) - \ln(\hat{\pi}_{K})} \leq \parens*{\frac{\gamma}{1 - \gamma}}^{2^K} \, \supnorm{\ln(\pi^*_{t+1}) - \ln(\hat{\pi}_{0})}^{2^K}  \,. 
\end{equation*}
Dividing by $\frac{\gamma}{1 - \gamma}$ and noting that $\hat{\pi}_{K} = \pi_{t+1}$ and $\hat{\pi}_0 = \pi_t$ gives 
\begin{equation*}
\supnorm{\ln(\pi^*_{t+1}) - \ln(\pi_{t+1})} \leq \parens*{\frac{\gamma}{1 - \gamma}}^{2^K - 1} \, \supnorm{\ln(\pi^*_{t+1}) - \ln(\pi_t)}^{2^K}  \,. 
\end{equation*}
Otherwise, when $E > 0$ we use induction. For this, let 
\begin{equation*}
 a_k \coloneqq \frac{\gamma}{1 - \gamma} \, \supnorm{\ln(\pi^*_{t+1}) - \ln(\hat{\pi}_{k})}  \quad \text{ ; }  \quad \tilde{e} \coloneqq \frac{2 \,  E \,  \gamma}{\tau \, (1 - \gamma)}
\end{equation*}
We show by induction that $a_k \leq a_0^{2^k} + 2 \, \tilde{e}$ for all $k \in [K]$. \\
\textbf{Base case:} For $k = 0$, $a_0 \leq a_0 + 2 \, \tilde{e}$. \\
\textbf{Induction step:} Suppose that $a_k \leq a_0^{2^k} + 2 \, \tilde{e}$ for some $k \geq 0$. Then, by~\cref{eq:quad_recurse}
\begin{align*}
    a_{k+1} &\leq a_k^2 + \tilde{e} \\
       &\leq \parens*{a_0^{2^k} + 2 \, \tilde{e}}^2 + \tilde{e} \tag{Induction hypothesis} \\
    &= a_0^{2^{k+1}} + 4 \, \tilde{e} \, a_0^{2^k} + 4 \, \tilde{e}^2 + \tilde{e} \tag{Expanding the square} \\
    &= a_0^{2^{k+1}} + \tilde{e} \, \parens*{4 \, \parens*{a_0^{2^k} + \tilde{e}} + 1} \\
    &\leq a_0^{2^{k+1}} + 2 \, \tilde{e} \, . \tag{Since $a_0 < 1$ implies $a_0^{2^k} \leq a_0$, and assuming $a_0 + \tilde{e} \leq \nicefrac{1}{4}$}
\end{align*}
This completes the induction. Setting $k = K$ and dividing by $\frac{\gamma}{1 - \gamma}$, we have 
\begin{equation*}
\supnorm{\ln(\pi^*_{t+1}) - \ln(\pi_{t+1})} \leq \parens*{\frac{\gamma}{1 - \gamma}}^{2^K - 1} \, \supnorm{\ln(\pi^*_{t+1}) - \ln(\pi_t)}^{2^K} + \frac{4 \, E}{\tau}
\end{equation*}
Combining both cases and defining $L_t \coloneqq \supnorm{\ln(\pi^*_{t}) - \ln(\pi_{t})}$ and $n = 2^K$,
\begin{align*}
L_{t+1} &\leq \parens*{\frac{\gamma}{1 - \gamma}}^{n - 1} \, \supnorm{\ln(\pi^*_{t+1}) - \ln(\pi_t)}^{2^K} + \frac{4 \, E}{\tau} \\
&= \parens*{\frac{\gamma}{1 - \gamma}}^{n-1} \, \supnorm{\ln(\pistartt) - \ln(\pistart) + \ln(\pistart) - \ln(\pi_t)}^{n} + \frac{4 \, E}{\tau}\tag{Add/Subtract $\ln(\pistart)$} \\
&\leq \parens*{\frac{\gamma}{1 - \gamma}}^{n-1} \, \parens*{\supnorm{\ln(\pistartt) - \ln(\pistart)} + \underbrace{\supnorm{\ln(\pistart) - \ln(\pi_t)}}_{\coloneqq L_t} }^{n} + \frac{4 \, E}{\tau} \tag{Using triangle inequality}
\end{align*}
\begin{align*}
    \implies L_{t+1} \leq  \parens*{\frac{\gamma}{1 - \gamma}}^{n-1} \, \parens*{\supnorm{\ln(\pistartt) - \ln(\pistart)} + L_t }^n + \frac{4 \, E}{\tau}
\end{align*}
We now bound $\supnorm{\ln(\pistartt) - \ln(\pistart)}$ using the Lipschitz property of $Q^*_{R(\lambda),\tau}$ in $\lambda$. For a fixed $s \in \cS$, by~\cref{eq:sup_to_sup_lipschitz} of~\cref{proposition:pistarl_lipschitz} with $q= Q^{*}_{R_{t+1}, \tau}(s, \cdot)$ and $q' = Q^{*}_{R_{t}, \tau}(s, \cdot)$, we have
\begin{align*}
   \supnorm{\ln(\pistartt(\cdot | s)) - \ln(\pistart(\cdot | s))} &\leq \frac{2}{\tau } \, \supnorm{Q^{\pistartt}_{R_{t+1}, \tau}(s, \cdot) - Q^{\pistart}_{R_{t}, \tau}(s, \cdot)} \\ 
   &\leq \frac{2}{\tau\, (1 - \gamma)} \, \norm{\ltt - \lt}_1 \tag{Using~\cref{lemma:qstar_lipschitz} with $\lambda = \lt$ and $\lambda' = \ltt$} \\
   &\leq \frac{2 \, \sqrt{m}}{\tau \, (1 - \gamma)} \, \norm{\ltt - \lt}  \tag{Since for any $x \in \R^m$, $\norm{x}_1 \leq \sqrt{m} \, \norm{x}$}
\end{align*}
Taking the max over all $s \in \cS$ 
\begin{align}
\implies   \supnorm{\ln(\pistartt) - \ln(\pistart)}  &\leq \frac{2 \, \sqrt{m}}{\tau \, (1 - \gamma)} \, \norm{\ltt - \lt} \label{eq:ln_star_tt_ln_star_t_ub}
\end{align}
\begin{align*}
\implies L_{t+1} \leq  \parens*{\frac{\gamma}{1 - \gamma}}^{n-1} \, \parens*{\frac{2 \, \sqrt{m}}{\tau \, (1- \gamma)} \, \norm{\ltt - \lt} + L_t }^n + \frac{4 \, E}{\tau}
\end{align*}
\end{proof}
\begin{lemma}\label{lemma:quadratic_recursion}
Let $K \geq 1$ and $n = 2^K$. If the non-negative sequence $\{L_t\}_{t\geq 0}$ satisfies
\begin{equation*}
L_0 \leq 2^{n} \, \kappa \, (\nu + e)^{n} + e \quad \text{;} \quad L_{t+1} \leq \kappa \, (L_t+\nu)^{n} + e \quad \forall  t\geq 0.
\end{equation*}
where $\kappa \geq 0$, $\nu>0$ $e \geq 0$ and $2^n \, \kappa \, (\nu + e)^{n-1} \leq 1$. 
Then for all $t \geq 0$, 
\begin{align*}
    L_t &\leq 2^n \, \kappa \, (\nu + e)^{n} + e \, .
\end{align*}
\end{lemma}
\begin{proof}
We use induction to prove this fact for a fixed $K \geq 1$. \\   
\textbf{Base case:} For $t =0$, by assumption $L_0 \leq 2^n \, \kappa \, (\nu + e)^{n} + e$. \\
\textbf{Induction Step:} Suppose $L_t \leq 2^n \, \kappa \, (\nu + e)^{n} + e$ for some $t \geq 0$. 
Then, since $\x \to \kappa x^n$ is non-decreasing for $x \geq 0$,
\begin{align*}
L_{t+1} &\leq \kappa \, (L_t + \nu)^{n} + e \\
&\leq \kappa \, \parens*{2^n \, \kappa \, (\nu + e)^{n} + e + \nu}^{n} + e
    \tag{Induction hypothesis} \\
&= \kappa \, \parens*{(\nu + e) \parens*{2^n \, \kappa \, (\nu + e)^{n-1} + 1}}^{n} + e \\
&= \kappa \, (\nu + e)^{n} \, \parens*{2^n \, \kappa \, (\nu + e)^{n-1} + 1}^{n} + e \\
&\leq 2^n \, \kappa \, (\nu + e)^{n} + e \, ,
    \tag{Since $2^n \, \kappa \, (\nu + e)^{n-1} \leq 1$}
\end{align*}
which completes the proof.
\end{proof}

\begin{lemma}\label{lemma:invariant_recursion}
Let $\{L_t\}_{t\geq 0}$ be a non-negative sequence.
Suppose there exists a region $B > 0$ such that $L_t \leq B$ implies
\begin{equation}\label{eq:conditional_recursion}
    L_{t+1}
    \leq
    \kappa (L_t+\nu)^n + e,
\end{equation}
for constants $\kappa\geq 0$, $\nu\geq 0$, and $e\geq 0$. Then, if
\begin{align}
    L_0 &\leq B, \label{eq:invariant_init}\\
    \kappa(B+\nu)^n+e &\leq B, \label{eq:invariant_condition}
\end{align}
we have, for all $t \geq 0$, $L_t \leq B$.
\end{lemma} 
\begin{proof}
We prove the result by induction.  \\
\textbf{Base case:} For $t=0$, by the initial condition in~\cref{eq:conditional_recursion}, $L_0 \leq B$. \\
\textbf{Induction Step:} Suppose that $L_t\leq B$ for some $t\geq 0$.
Then by~\cref{eq:conditional_recursion} and since $x \mapsto \kappa \, (x + \nu)^n$ is non-decreasing on $x \geq 0$,
\begin{equation*}
L_{t+1} \leq \kappa \, (L_t + \nu)^n + e \leq \kappa \, (B + \nu)^n + e \leq B \tag{by~\cref{eq:invariant_condition}}
\end{equation*}
\end{proof}

\section{Tabular Experiments}\label{appendix:experiments}
All experiments were run on a local desktop on an AMD Ryzen 9 3900X CPU.

\subsection{Constrained Bandit}\label{appendix:experiments_bandit}
In~\cref{fig:bandit}, we use the following constrained armed bandit problem with deterministic rewards and constrained rewards: $r = [1.0, 0.0]$, $c = [0.0, 1.0]$ and $b = 0.7$. The optimal policy is thus $\pi^* = [0.3, 0.7]$. \\
For \texttt{NPG-PD}~\cite{ding2020natural}, we use the following parameters: $\eta_{\text{primal}} = 2 \, \ln(A)$, $\eta_{\text{dual}} = 2 / \sqrt{T}$ and $T = 1000$. \\
For our proposed primal-dual method ($\texttt{DD (K=1)}$), we use the following parameters: $\beta = 0.01$, $\tau = \mu = 0.01$, $K=1$, $T = 950$ and $J = 50$ where we use \texttt{E-NPG} as the inner sub-problem solver with $\eta  = (1 - \gamma) / \tau$.

\subsection{Constrained MDPs}\label{appendix:experiments_mdp}

Following the experimental set-up in~\cite{lu2026augmented}, we consider the following tabular environments. In each of the following environments, we set the initial state distribution to be uniform, i.e. for all $s \in \gS$, $\rho(s) = \nicefrac{1}{S}$. 

\textbf{Constrained Cliff World:} 
This environment modifies the original Cliff World~\citep[Example 6.6]{sutton2018reinforcement} and adds additional constrained rewards.
The environment consists of $21$ states and $4$ actions. The objective is for an agent to each the goal state while avoiding a chasm. 
If the agent falls into the chasm (indicated by the red region), the agent receives a reward of $0$ and a constrained reward of $0$. 
If the agent reaches the goal (the green region), the agent receives a reward of $+1$ and a constrained reward of $0$.  All other rewards are $0$ and constrained rewards are $+1$. 
In this environment $\gamma = 0.9$ and $b = 9.83$.

\begin{figure}[!h]
    \centering
    \includegraphics[scale=0.8]{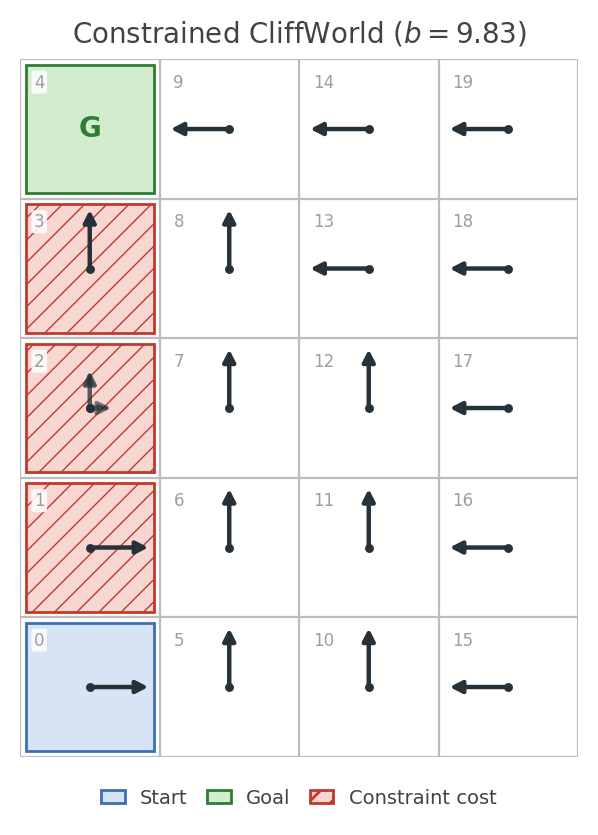}
    \caption{Visualization of the optimal policy for the cliff world environment. A longer arrows indicates a higher the probability of taking a particular action.}
\end{figure}

\textbf{Constrained Deep Sea Treasure} 
This environment modifies the original Deep sea treasure
~\citep{osband2019behaviour} and adds additional constrained rewards.
The environment consists $25$ states and $2$ actions. The agent descends one row per each time it takes an action. The goal of the agent is to stay right in order to reach the treasure (indicated by the green region). If the agent transitions to the left, it receives a reward of $0$. Otherwise if the agent reaches the treasure, it receives a reward of $+1$.  Additionally, in the environment there are two landmines. 
If the agent reaches a landmine (indicated by the red region), the agent receives a constrained reward of $0$. All other constrained rewards are $+1$. In this environment $\gamma = 0.9$ and $b = 9.99$.

\begin{figure}[!h]
    \centering
    \includegraphics[scale=0.8]{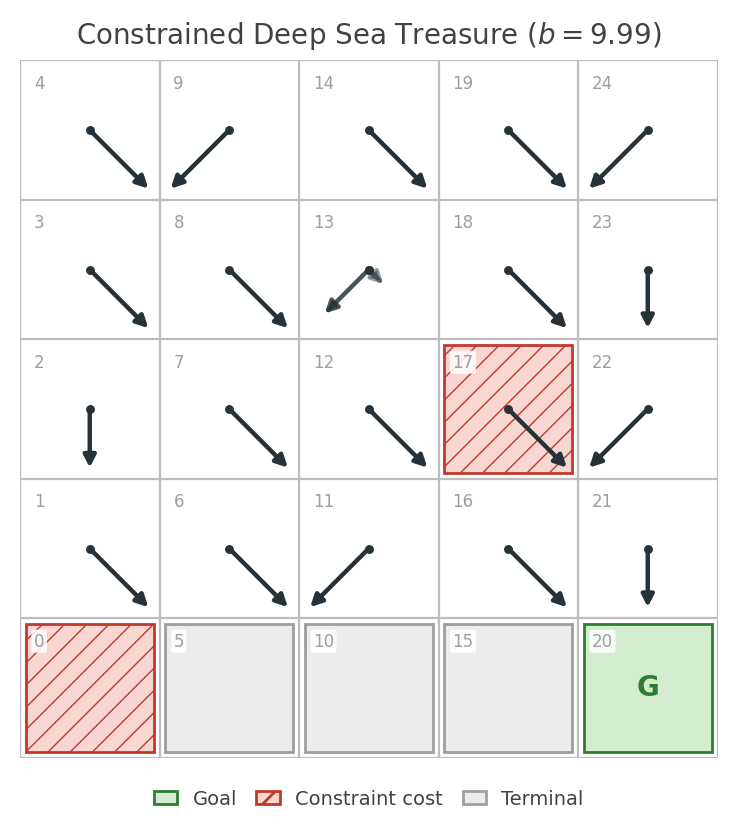}
    \caption{Visualization of the optimal policy for the deep sea treasure environment. A longer arrows indicates a higher the probability of taking a particular action. }
\end{figure}

\begin{figure}[!h]
    \centering
    \includegraphics[scale=0.8]{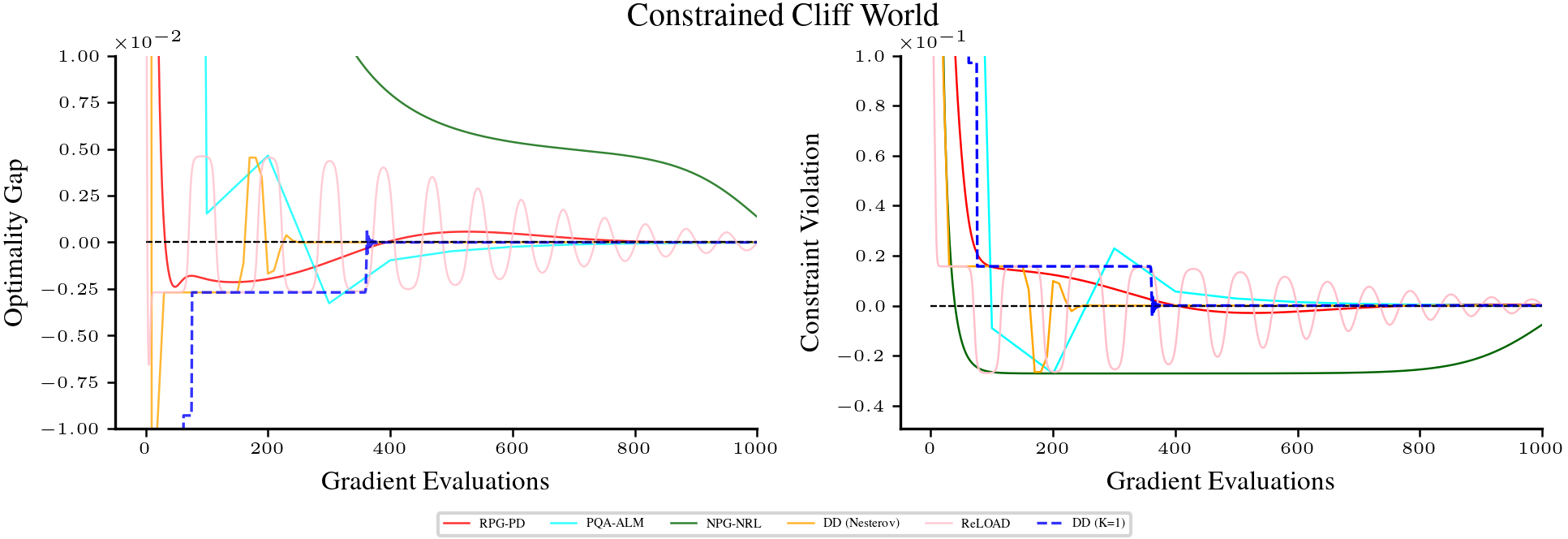}
    \includegraphics[scale=0.8]{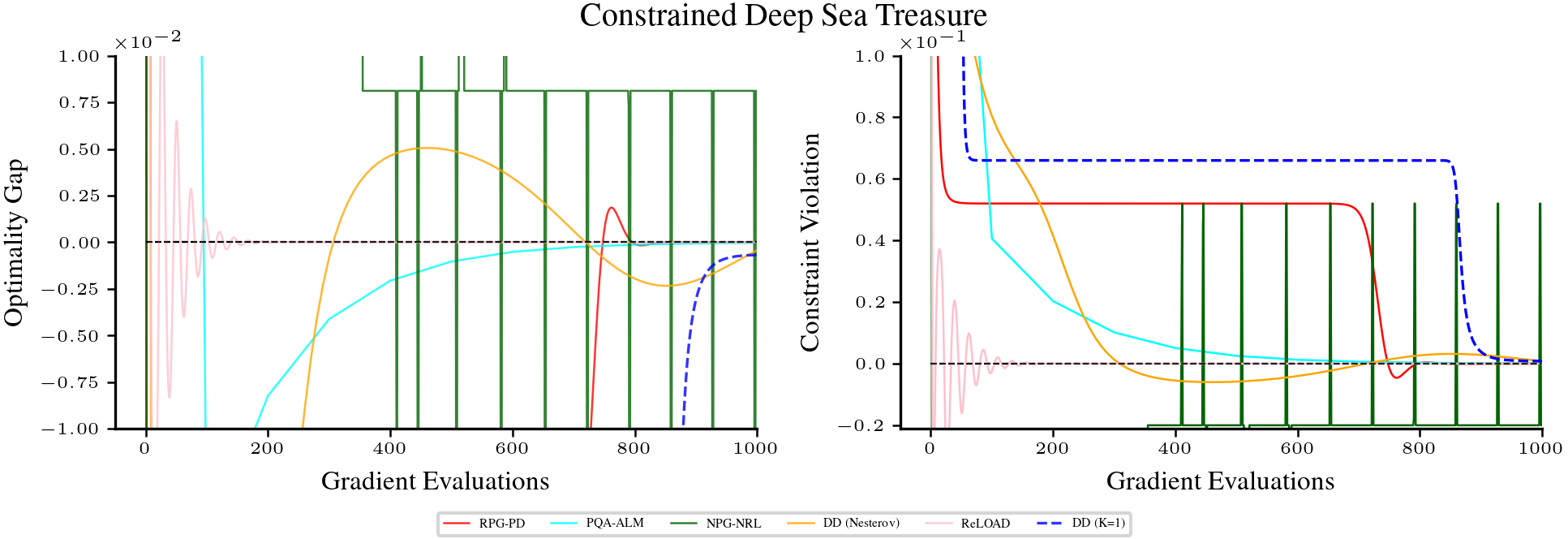}
    \caption{Comparing our regularized inexact \texttt{DD} method $\texttt{DD (K=1)}$ when $K=1$, $\mu > 0$ and warm-starting the initial policy with $50$ iterations. 
    We compare against the following practical last-iterate methods from prior work: \texttt{RPG-PD}~\cite{ding2023last}, \texttt{PQA-ALM}~\cite{lu2026augmented}, \texttt{NPG-NRL}~\cite{muller2024truly}, \texttt{ReLoad}~\cite{moskovitz2023reload}, and the related DD method of Ying et al.~\cite{ying2022dual} (\texttt{DD (Nesterov)}). Our method converges steadily to zero optimality gap and constraint violations without any oscillations }
\end{figure}

For all methods, hyperparameters were selected via grid search, and the full configurations are listed in \cref{table:tabular_grid_search}.
We report results using the configuration that achieved the smallest optimality gap subject to the final policy satisfying the constraint violation $b - V^{\pi_{T}}_{c}(\rho) \leq 0.001$.

\begin{table}[h!]
\centering
\caption{Grid search ranges for the tabular experiments
}
\resizebox{0.9\textwidth}{!}{
\begin{tabular}{l l l}
\toprule
\textbf{Method} & \textbf{Fixed Params.} & \textbf{Grid Search Space} \\
\midrule
\texttt{DD (K=1)} 
& $T=950$, $K=1$, $J=50$, $\tau = \mu = 0.0001$ & Dual step-size ($\beta)$: $\{0.0001, 0.001, 0.01, 0.1\}$  \\
& Warm-up iterations=50 & \\
\midrule
\texttt{PQA-ALM} 
& $T=10$, $K=100$, $\beta=10$ 
& Primal step-size: $\{0.1,1,10\}$ \\
\midrule
\texttt{RPG-PD}~\citep{ding2023last} 
& $T=1000$  & Entropy Coefficient $\{10^{-4},10^{-2},10^{-1}\}$ \\
& & Primal step-size: $\{0.01,0.1,1,10,100\}$  \\
& & Dual step-size: $\{0.01,0.1,1,10,100\}$ \\
\midrule
\texttt{NPG-NRL}~\citep{muller2024truly} 
& $T=1000$ $\mu = 0.0001$ & Entropy Coefficient $\{10^{-4},10^{-2},10^{-1}\}$ \\
& & Primal step-size: $\{0.01,0.1,1,10,100\}$  \\
& & Dual step-size: $\{0.01,0.1,1,10,100\}$ \\
\midrule
\texttt{DD (Nesterov)}~\citep{ying2022dual} 
& $T=100$, $K=10$ 
& Entropy Coefficient: $\{0.001,0.01,0.1,1\}$  \\
& & Dual step-size: $\{0.001,0.01,0.1,1\}$ \\
\midrule
\texttt{ReLOAD}~\citep{moskovitz2023reload} 
& $T=1000$  & Primal step-size: $\{0.01,0.1,1,10,100\}$ \\
& & Dual step-size $\{0.01,0.1,1,10,100\}$ \\
\bottomrule
\end{tabular}
}
\label{table:tabular_grid_search}
\end{table}

\section{Technical Tools for Natural Policy Gradient with Entropy Regularization}\label{appendix:npg}
The following results correspond to~\cite[Theorem 1]{cen2022fast} and~\cite[Theorem 2]{cen2022fast} and is added for completeness.
\subsection{Exact Setting}
At iteration $k \geq 0$, the entropy-regularized \texttt{NPG} update~\cite{cen2022fast} is given by:
\begin{equation}\label{eq:entropy_npg_update}
    \pi_{k+1}(a | s) \propto \pi_k(a | s)^{1 - \frac{\eta \, \tau}{1 - \gamma}} \, \exp\parens*{\frac{\eta \, Q^{\pi_k}_{\diamond, \tau}(s, a)}{1 - \gamma}} \quad \forall (s, a) \in \cS \times \cA \, , 
\end{equation}
where $\diamond \in \bracks*{0, U}^{SA}$, $U > 0$ is a bounded pseudo-reward and $Q^{\pi_k}_{\diamond, \tau} \in \R^{SA}$ is the soft Q-function of policy $\pi_k$.
\begin{restatable}{lemma}{entropynpgsupnormconvergencelemma}\label{lemma:entropy_npg_supnorm_convergence}
For any pseudo-reward $\diamond \in [0, U]^{SA}$, $U > 0$ the \texttt{NPG} update in~\cref{eq:entropy_npg_update} with $\tau > 0$, $\eta \leq \frac{1 - \gamma}{\tau}$ converges as: for $K \geq 0$:
\begin{align*}
\supnorm{Q^{\pi^*_\tau}_{\diamond, \tau} - Q^{\pi_{K+1}}_{\diamond, \tau}} &\leq \gamma \, C \, (1 - \eta \, \tau)^{K}  \\
\supnorm{\ln(\pi^*_{\diamond, \tau}) - \ln(\pi_{K+1})} &\leq \frac{2}{\tau} \, C \, (1 - \eta \, \tau)^{K} \\
\supnorm{V^{\pi^*_\tau}_{\diamond, \tau} - V^{\pi_{K+1}}_{\diamond, \tau}} &\leq (2 + \gamma) \, C \, (1 - \eta \, \tau)^{K}
\end{align*}
where $\pi^*_{\diamond, \tau} \coloneqq \argmax_\pi V^{\pi}_{\diamond, \tau}(\rho)$, $\pi_0$ is the initial policy, and $C=\|Q^{\pi^*_{\diamond, \tau}}_{\diamond, \tau} - Q^{\pi_0}_{\diamond, \tau}\|_\infty + 2 \tau (1-\tfrac{\eta \tau}{1 - \gamma}) \supnorm{\ln(\pi^*_{\diamond, \tau}) - \ln (\pi_0)}$.
\end{restatable}
\subsection{Inexact Setting}\label{appendix:inexact_npg}
At iteration $k \geq 0$, the \emph{inexact} entropy-regularized \texttt{NPG} update~\cite{cen2022fast} is given by:
\begin{equation}\label{eq:inexact_npg_update_appendix}
    \pi_{k+1}(\cdot | s) \propto (\pi_k(\cdot | s))^{1 - \frac{\eta \, \tau}{1 - \gamma}} \exp\parens*{\frac{\eta \, \hat{Q}^{\pi_k}_{\diamond, \tau}(s, \cdot)}{1 - \gamma}}
\end{equation}
where $\hat{Q}^{\pi_k}_{\diamond, \tau} \in \R^{SA}$ is the estimated soft Q-function of policy $\pi_k$. 

\inexactentropynpgsupnormconvergencelemma*
\section{External Results}
For completeness, we provide external results below.
\propositiondualgradient*
\begin{proof}
A proof of this fact may be found in~\cite[Proposition 3.6]{ying2022dual}.
\end{proof}

\begin{lemma}\label{lemma:entropy_saom_strictly_concave}
\cite[Lemma 3.1]{ho2016generative}
The function $H(\saom) = - \sum_{s, a} \saom(s,a)\ln\parens*{\frac{\saom(s, a)}{\sum_{a'} \saom(s, a')}}$ is strictly concave. Moreover for all $\pi$ and its corresponding unnormalized state-action occupancy measure $\saom$, we have $\gH(\pi, \rho) = H(\saom)$.
\end{lemma}

\begin{lemma}\label{lemma:strong_duality}
Under~\cref{assumption:slater_condition}, strong duality holds for~\cref{problem:soft_max_min_lag}, i.e.,
\begin{equation*}
\max_{\pi} \min_{\lambda \in \R^m_+} \lag(\pi, \lambda) = \min_{\lambda \in \R^m_+}\max_{\pi}  \lag(\pi, \lambda)
\end{equation*}
\end{lemma}
\begin{proof}
A proof is provided in~\cite[Theorem 1]{paternain2022safe}. We provide an alternative proof using state-action occupancy measures.

For this, we define $\Pr^\pi[s_h = s, a_h = a \mid s_0]$ as the probability of visiting state $s$ and taking action $a$ at step $h$ when starting at the initial state $s_0$ and the unnormalized \textit{state-action occupancy measure} $\mu^\pi(s, a) \coloneq \sum_{s_0} \rho(s_0) \, \sum_{h =0}^\infty \gamma^h \, \Pr^\pi[s_h = s, a_h = a \mid s_0]$. Note for any $\saom$, there exists a unique one-to-one mapping between $\saom$ and $\pi$:
\begin{equation}\label{eq:one_one_saom_pi_main}
   \pi(a \mid s) = \frac{\mu^{\pi}(s, a)}{\sum_{a' \in \cA}\mu^{\pi}(s, a')} \quad \forall (s, a) \in \cS\times\cA.
\end{equation}
Consequently, the objective can be expressed directly in terms of the occupancy measure $\saom$, i.e.,  $V^{\pi}_r(\rho) = \dpd{\saom, r}$, and $V^{\pi}_{c_i}(\rho) = \dpd{\saom, c_i}$ for each constraint $i \in [m]$.
Additionally, by~\cref{lemma:entropy_saom_strictly_concave}, the entropy regularizer admits the equivalent form:
\begin{equation*}
   \gH(\pi, \rho) = H(\saom) \coloneqq - \sum_{s, a} \saom(s, a) \log \frac{\saom(s, a)}{\sum_{a'} \saom(s, a')} 
\end{equation*}
and that $H(\saom)$ is concave w.r.t $\saom$. Hence,~\cref{problem:soft_max_min_lag} is equivalent to the following optimization problem over occupancy measures~\cite{altman2021constrained}:
\begin{equation}\label{eq:soft_cmdp_problem_lp}
  \max_{\saom \in \gK} \dpd{\saom, r} + \tau \, H(\saom) \quad \mathrm{ s.t. } \quad \dpd{\saom , c_i} \geq b_i\, \quad \forall i \in [m],
\end{equation}
\noindent where
\begin{equation*}
    \gK \coloneq \braces*{\saom \in \R^{SA}_+ \mid \forall s' \in \cS, \sum_{a' \in \cA} \saom(s', a') = \rho(s') + \gamma \sum_{s \in \cS} \sum_{a \in \cA} \cP(s' | s, a) \, \saom(s, a) }
\end{equation*}
represents the set of all feasible state-action occupancy measures $\saom$ satisfying the Bellman flow constraints. 
As a result~\cref{eq:soft_cmdp_problem_lp} is a concave maximization problem over a convex set with linear constraints.

By~\cref{assumption:slater_condition}, there exists a policy $\bar{\pi}$ such that
\[
V^{\bar{\pi}}_{c_i}(\rho) > b_i, \quad \forall i \in [m],
\]
Using the one-to-one mapping between $\pi$ and $\saom$ in~\cref{eq:one_one_saom_pi_main}, the state-action occupancy measure $\mu^{\bar{\pi}}$ satisfies
\[
\dpd{\mu^{\bar{\pi}}, c_i} > b_i , \quad \forall i \in [m].
\]
which implies that Slater's condition holds for~\cref{eq:soft_cmdp_problem_lp} Hence, strong duality holds~\cite[Section 5.2.3]{boyd2004convex}, i.e., \begin{equation*}
\max_{\saom} \min_{\lambda \in \R^m_+} F_\tau(\saom, \lambda) = \min_{\lambda \in \R^m_+}\max_{\saom}  F_\tau(\saom, \lambda)
\end{equation*}
where $F_\tau(\saom, \lambda) \coloneqq \dpd{\saom, r} + \tau \, H(\saom) + \sum_{i=1}^m \lambda_i [\dpd{\saom, c_i} - b_i]$ is the Lagrangian of~\cref{eq:soft_cmdp_problem_lp}.
Finally, using the one-to-one correspondence between $\saom$ and $\pi$ in~\cref{eq:one_one_saom_pi_main}, we have for any $\saom$ and $\lambda$ that
\[
F_\tau(\mu^\pi, \lambda) = \lag(\pi, \lambda),
\]
yielding
\begin{equation*}
\max_{\pi} \min_{\lambda \in \R^m_+} \lag(\pi, \lambda)
=
\min_{\lambda \in \R^m_+} \max_{\pi} \lag(\pi, \lambda),
\end{equation*}
which completes the proof.
\end{proof}

\begin{lemma}\label{lemma:soft_suboptimality}
\cite[Lemma 26]{mei2020global}
For any bounded pseudo-reward function $\diamond \in [0, U]^{SA}$, $U > 0$ and $\tau > 0$, let $\pi' \coloneqq \argmax_{\pi} V^{\pi}_{\diamond, \tau}(\rho)$.
Then for any policy $\pi$,
\begin{equation*}
   V^{\pi'}_{\diamond, \tau}(\rho) - V^{\pi}_{\diamond, \tau}(\rho)  = \frac{\tau}{1 - \gamma} \sum_{s \in \cS} d^{\pi}(s) \,  \KL(\pi(\cdot | s) \| \pi'(\cdot | s)) \, .
\end{equation*}
\end{lemma}

\begin{lemma}[Hoeffding's Lemma]\label{lemma:hoeffding_lemma}
Let $X$ be any real-valued random variable, almost surely, such that $a \leq X \leq b$, then for any $y \in \R$,
\begin{equation*}
    \E[\exp(y (X - \E[X]))] \leq \exp\parens*{\frac{y^2 \, (b - a)^2}{8}} \, .
\end{equation*}
\end{lemma}

\end{document}